\documentclass[]{article}

\PassOptionsToPackage{numbers,sort,compress}{natbib}
\PassOptionsToPackage{dvipsnames}{xcolor}
\usepackage[final, main]{neurips_2026}
\usepackage[utf8]{inputenc}
\usepackage[T1]{fontenc}
\usepackage[colorlinks,citecolor=Blue,linkcolor=BrickRed,urlcolor=Blue,backref=page]{hyperref}
\usepackage{url}
\usepackage{algorithm}
\usepackage{algpseudocode}
\usepackage{booktabs}
\usepackage{amsthm,thmtools,thm-restate}
\RequirePackage{amsmath}
\usepackage{cleveref}
\RequirePackage{amssymb}
\RequirePackage{mathtools}
\ifx\proof\undefined 
\RequirePackage{amsthm}
\fi
\newtheorem{theorem}{Theorem}[section]
\newtheorem{corollary}[theorem]{Corollary}
\newtheorem{lemma}[theorem]{Lemma}
\newtheorem{proposition}[theorem]{Proposition}
\newtheorem{assumption}[theorem]{Assumption}
\theoremstyle{definition}
\newtheorem{definition}[theorem]{Definition}
\newtheorem{remark}[theorem]{Remark}


\crefname{figure}{Fig.}{Figs.}
\crefname{definition}{Defn.}{Defns.}
\crefname{corollary}{Corollary}{Corollaries}
\crefname{proposition}{Prop.}{Props.}
\crefname{theorem}{Thm.}{Thms.}
\crefname{remark}{Remark}{Remarks}
\crefname{lemma}{Lemma}{Lemmata}
\crefname{claim}{Claim}{Claims}
\crefname{table}{Table}{Tables}
\crefname{section}{\S}{\S\S}
\crefname{subsection}{\S}{\S\S}
\crefname{subsubsection}{\S}{\S\S}
\crefname{assumption}{Assm.}{Assms.}
\crefname{appendix}{Appx.}{Appxs.}
\crefname{equation}{Eq.}{Eqs.}
\crefname{example}{Ex.}{Exs.}
\usepackage{amsfonts}
\usepackage{amssymb}
\usepackage{microtype}
\usepackage{xcolor}
\usepackage{graphicx}
\usepackage{notation}
\usepackage{fontawesome}
\usepackage{multirow}
\usepackage{enumitem}
\usepackage{makecell}
\usepackage{titletoc}
\usepackage{mdframed}
\title{
Anchor PCA
}
\author{
 Benedikt Seiter\\
  ETH Zürich\\
  \texttt{bseiter@ethz.ch}\\
  \And
  Anya Fries\\
  Seminar for Statistics\\
  ETH Zürich\\
  \And
  Julius von Kügelgen \\
  Seminar for Statistics\\
  ETH Zürich\\
  \And
  Jonas Peters\\
  Seminar for Statistics\\
  ETH Zürich\\
}

\DeclareMathOperator{\Tr}{Tr}
\DeclareMathOperator*{\argmax}{arg\,max}
\DeclareMathOperator*{\argmin}{arg\,min}

\DeclareMathOperator{\Imop}{Im}

\makeatletter
\newcommand{\anchor}{\mathrel{\mathpalette\anchoraux{}}}
\newcommand{\anchoraux}[2]{%
  \ooalign{%
    \hfil\raise0.01ex\hbox{\vrule width 0.06em height 1.03ex depth 0pt}\hfil\cr
    \hfil\raise-0.24ex\hbox{$\m@th#1\smile$}\hfil\cr
  }%
}
\makeatother

\newcommand{\R}{\mathbb{R}}

\newcommand{\Spp}{\mathbb{S}_+^p}
\newcommand{\cE}{\mathcal{E}}
\newcommand{\Ok}{\mathcal{O}_{p\times k}}

\DeclareRobustCommand{\AnchorPCA}[1]{%
  \ensuremath{\text{\normalfont\ttfamily AnchorPCA}_{%
    \if\relax\detokenize{#1}\relax
      \lambda
    \else
      #1
    \fi
  }}%
}

\DeclareRobustCommand{\AnchorPCAinfty}{\AnchorPCA{\infty}}

\DeclareRobustCommand{\AnchorPCAwout}{%
  \ensuremath{\text{\normalfont\ttfamily AnchorPCA}
  }%
}

\newcommand{\poolPCA}{\textup{\texttt{poolPCA}}}
\newcommand{\fnorm}[1]{\left\lVert #1 \right\rVert_{\mathrm{F}}}
\newcommand{\opnorm}[1]{\left\lVert #1 \right\rVert_{\mathrm{op}}}
\newcommand{\barSigma}{\overline{\Sigma}}
\newcommand{\barPi}{\overline{\Pi}}
\newcommand{\hatPi}{\widehat{\Pi}}
\newcommand{\hatSigma}{\widehat{\Sigma}}
\newcommand{\barR}{\overline{R}}

\newcommand{\Sstar}{\mathcal{S}_{\star}}
\newcommand{\SstarPerp}{\mathcal{S}_{\star}^{\perp}}
\newcommand{\Sanchor}[1]{\mathcal{S}^{\anchor,\lambda}_{#1}}
\newcommand{\Pianchor}[1]{\Pi^{\anchor,\lambda}_{#1}}
\newcommand{\Wanchor}[1]{{W^{\anchor,\lambda}_{#1}}}

\newcommand{\AnchorPCABoldFont}{%
  \normalfont\fontfamily{pcr}\fontseries{b}\selectfont
}

\DeclareRobustCommand{\AnchorPCAbold}[1]{%
  \ensuremath{\text{\AnchorPCABoldFont AnchorPCA}_{%
    \boldsymbol{%
      \if\relax\detokenize{#1}\relax
        \lambda
      \else
        #1
      \fi
    }%
  }}%
}

\begin{document}

\maketitle

\begin{abstract}
\looseness-1 
Principal component analysis (PCA) is one of the most widely used unsupervised dimension reduction techniques. 
We study PCA for data from multiple related domains.
Since principal components generally differ across domains, one way to obtain a shared low-rank embedding is to perform PCA on the pooled data.
However, this approach can focus on spurious directions that exhibit high variation in only a few domains.
To find a robust embedding that still explains most variance in unseen but similar domains, we propose instead
to focus on shared directions of variation.
To this end, we introduce Anchor PCA which trades off overall explained variance with 
agreement between the shared and domain-specific low-rank embeddings.
Anchor PCA amounts to PCA on a modified target matrix and thus can be solved efficiently.
Moreover, we show that Anchor PCA recovers a maximal invariant subspace and admits a minimax reconstruction interpretation under bounded domain-specific covariance inflations.
On simulated and real-world 
data with temporal 
distribution 
shift,
we
demonstrate, respectively,
that Anchor PCA recovers the maximal invariant subspace and yields embeddings that explain more variance on unseen domains than the pooling baseline and a worst-case alternative. 
Taken together, these findings establish Anchor PCA as a promising approach to robust unsupervised dimension reduction from multi-domain data.
\end{abstract}

\section{Introduction}
\label{sec:intro}

Principal component analysis (PCA) \citep{Jolliffe2002} is one of the most widely used methods for unsupervised dimensionality reduction. Given a dataset, PCA finds a low-dimensional subspace that captures as much variance as possible. It is used extensively in exploratory data analysis, preprocessing, and representation learning across the natural sciences, engineering, and machine learning \citep[e.g.,][]{novembre2008genes,migliavacca2021three,wright2004worldwide,turk1991eigenfaces}. 
Unsupervised learning can serve different goals \citep{hyvarinen2026goal}.
Here, we focus on two uses of PCA:
providing interpretable directions (that may carry physical meaning)
and explaining variance. 

We assume that we have data collected across multiple related but heterogeneous settings---different geographic locations, time points or experimental conditions.
We refer to these settings as \emph{source domains}. 
If the 
principal subspaces differ substantially across the source domains, 
we may achieve neither of the two goals mentioned above. 
For example, PCA applied to the average covariance (\poolPCA{}) could be dominated by a direction that concentrates variance in only one or two domains. This direction may explain very little variance elsewhere and thus may 
not yield physically relevant directions; similarly, it may not explain a lot of variance in a target domain.

\looseness-1 In this paper, we propose to 
focus not only on
explained variance but also on finding invariant directions. 
This builds on work in multi-domain dimension reduction and invariant learning. In dimension reduction, recent approaches optimize worst-case performance across domains \citep{fries2026worstcaselowrankapproximations,wang2026stablepcadistributionallyrobustlearning} or enforce invariance of distributions \citep{norman2026unsupervised}. In supervised learning, invariance of predictive relations or representations across domains is widely studied \citep{arjovsky2019invariant,pmlr-v162-wang22x}. In contrast, we introduce a geometric notion of invariance based on agreement of domain-specific principal subspaces, and explicitly trade this against explained variance. 
Similar
in spirit is anchor regression \citep{rothenhausler2020anchor}, which introduces a soft invariance penalty in the supervised setting.
Concretely, we propose
\AnchorPCA{} (with $\lambda \in \mathbb{R}_{\ge0} \cup \{\infty\}$), a method that
finds a $k$-dimensional subspace 
where
the  parameter $\lambda$ 
trades off between
explaining variance and 
finding invariant directions 
(e.g., for $\lambda = 0$, \AnchorPCA{} reduces to \poolPCA{}).\footnote{For a short visual explanation of Anchor PCA,
see \href{https://benesei.github.io/anchor-pca/}{this video}.}
We further detail related work and how their objectives differ from ours in \cref{subsec:related_work}.

We prove that \AnchorPCAwout{} satisfies two attractive theoretical properties. First, if there exists a nontrivial $m$-dimensional invariant subspace $\Sstar$, that is, a subspace that lies in the top-$k$ principal subspace of every domain, then \AnchorPCAinfty{} is guaranteed to contain it, and the leading directions of \AnchorPCA{} converge to $\Sstar$ as $\lambda \to \infty$ (\cref{thm:invariant_subspace_limit}). Second, \AnchorPCA{} admits a minimax reconstruction interpretation: it is the rank-$k$ projector that minimizes worst-case average reconstruction error over test domains that are subject to domain-specific covariance inflations (\cref{thm:minimax_reconstruction}).

We now introduce the setting and 
build intuition through the motivating example in \cref{subsec:motivating_example_4d}.
\Cref{sec:definitions}
introduces Anchor PCA,
\cref{sec:recovery_guarantees} and \cref{sec:robust_reconstruction} contain the main theoretical results, 
\cref{sec:algorithm} shows learning methods 
for finite data,
and
\cref{sec:experiments} contains 
simulated and real-world 
experiments.
\Cref{subsec:related_work},~\cref{app:further_details},~\cref{app:proofs}, and~\cref{app:details_experiments} contain  
further related work,
further details on Anchor PCA, 
proofs,
and experiment details,
respectively.

\subsection{Setting and notation} \label{sec:setting_and_notation}
We consider a set of $E$ source domains indexed by %
$e\in\cE:=\{1,\dots,E\}$. For each domain $e$, %
let $x_e\in\R^p$ be a zero-mean random (column)
vector with covariance $\Sigma_e\in\Spp$ 
where $\Spp := \{\Sigma \in \mathbb{R}^{p\times p} : \Sigma = \Sigma^\top,\ \Sigma \succeq 0\}$ is the set of symmetric positive semidefinite matrices and we further assume $\Tr(\Sigma_e)>0$.
Recall that 
PCA selects the
eigenvectors of a covariance matrix $\Sigma$ corresponding to its $k$ largest
eigenvalues, also called a top-$k$ eigenbasis.
We denote a top-$k$ eigenbasis of $\Sigma_e$ by $W_e\in\Ok$, where 
${\Ok:=\{W\in\R^{p\times k}:W^\top W=I_k\}}$ is the Stiefel manifold of $p\times k$ matrices with orthonormal columns. 
The rank-$k$ principal projection matrices 
are given by $\Pi_k^{(e)}:=W_eW_e^\top$, which can be viewed as a linear autoencoder with encoder $W_e^\top$ and decoder $W_e$.

Given the 
domain-specific
covariance
matrices $(\Sigma_e)_{e \in \cE}$ 
(or i.i.d.\ samples from the corresponding distributions), we seek to solve the following problem, which we formalize in 
\cref{sec:recovery_guarantees} and \cref{sec:robust_reconstruction}.
\begin{mdframed}[linewidth=0.8pt]
\centering
Find a shared $k$-dimensional subspace that 
is robust with
respect to changing domains.
\end{mdframed}
\vspace{-1.6mm}
\looseness-1 %
Denote the average covariance by ${\barSigma:=\frac{1}{E}\sum_{e=1}^E \Sigma_e}$ and the average projector by ${\barPi:=\frac{1}{E}\sum_{e=1}^E \Pi_k^{(e)}}$.
For general $W\in\Ok$, we write $\Pi_W:=WW^\top$ for the projector onto the space spanned by $W$.
The Frobenius %
and operator norms of a matrix $X\in\R^{p\times k}$ are defined as ${\fnorm{X}:= \sqrt{\operatorname{Tr}(X^\top X)}}$ and $\opnorm{X} := \sup_{\|v\|_2=1}\|Xv\|_2$, respectively. 
For a linear subspace $\mathcal S\le \mathbb R^p$, let $\Pi_{\mathcal S}$ denote the orthogonal projector onto $\mathcal S$; 
if $U$ has orthonormal columns spanning~$\mathcal S$, then $\Pi_{\mathcal S}=UU^\top$.

\subsection{Motivating example}
\label{subsec:motivating_example_4d}
\looseness-1 We consider a simple example in 
$\mathbb R^4$ (so $p\!=\!4$) with
canonical basis 
$c_1, c_2, c_3, c_4$, writing $a:= c_1$ and $b := c_2$,
where we seek a $k\!=\!3$-dimensional representation.
We observe $E\!=\!3$ domains. 
Table~\ref{tab:motivating_example_4d} summarizes the 
eigenvectors and
eigenvalues of the covariance matrices
(we defer the details to
\cref{app:motivating_example_4d_details}).
\begin{table}[tbp]
\caption{\looseness-1 \textbf{Eigendecomposition and reconstruction errors for the motivating example.}
If $k=3$,
the invariant directions are
$a=c_1$ and $b=c_2$, whereas $u$, $v$, and $w$ are domain-specific %
directions
in the $c_3$--$c_4$ plane, and $u^\perp$, $v^\perp$, $w^\perp$ are their
orthogonal complements in that plane. The domain-specific columns list eigenvectors and
eigenvalues, equivalently the explained variance of each direction. The invariant directions do
not need fixed ranks: $a$ is top-ranked in domain~2, while $b$ outranks $a$ in
domain~3. The perturbed column evaluates the same reconstruction-error curves shown later
in Fig.~\ref{fig:motivating_example_4d_perturbation_path} at the dotted vertical line.}
\label{tab:motivating_example_4d}
\centering
\small
\begingroup
\setlength{\tabcolsep}{3.5pt}
\begin{tabular}{@{}c@{\hspace{1.8em}}c@{}}
\begin{tabular}{@{}c cc cc cc@{}}
\toprule
\multirow{2}{*}{\textbf{Rank}} & \multicolumn{2}{c}{\textbf{Domain 1}} & \multicolumn{2}{c}{\textbf{Domain 2}} & \multicolumn{2}{c}{\textbf{Domain 3}} \\
\cmidrule(lr){2-3}\cmidrule(lr){4-5}\cmidrule(lr){6-7}
 & direction & value & direction & value & direction & value \\
\midrule
1 & $u$ & $220$ & $a$ & $120$ & $w$ & $320$ \\
2 & $a$ & $140$ & $v$ & $90$ & $b$ & $120$ \\
3 & $b$ & $90$ & $b$ & $70$ & $a$ & $80$ \\
4 & $u^\perp$ & $25$ & $v^\perp$ & $10$ & $w^\perp$ & $10$ \\
\bottomrule
\end{tabular}
&
\begin{tabular}{@{}lcc@{}}
\toprule
\multirow{2}{*}{\textbf{Method}} &
\multicolumn{2}{c}{\textbf{Avg. reconstr. error} ($\downarrow$)} \\
\cmidrule(l){2-3}
& \makecell{Original\\covariances}
& \makecell{Perturbed\\covariances} \\
\midrule
\poolPCA{} & $93.3$ & $243.3$ \\
\AnchorPCA{\lambda=25} & $98.5$ & $163.4$ \\
\AnchorPCAinfty{} & $113.8$ & $172.5$ \\
\bottomrule
\end{tabular}
\end{tabular}
\endgroup
\end{table}
\begin{figure}[tbp]
\centering
\includegraphics[width=\textwidth, trim={0 4mm 0 11mm}, clip]{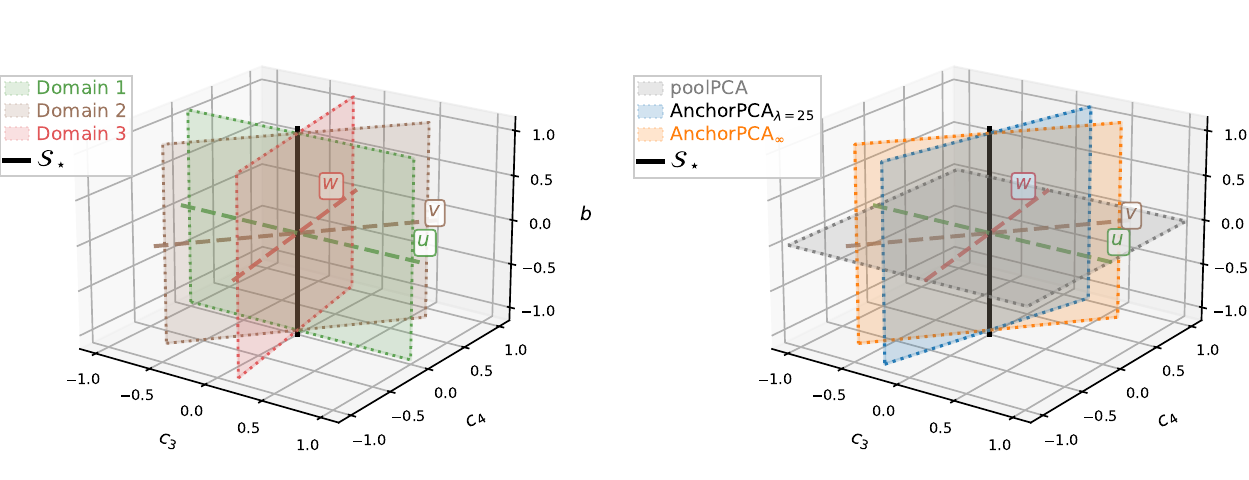}
\vspace{-1.5em}
\caption{\looseness-1 
\textbf{Geometric view of the motivating example.} Here, we ignore the direction $a$, which is chosen by all methods, yielding a reduced representation in the coordinates
$(c_3,c_4,b)$. Left: the three local
top-$3$ eigenspaces, shown as planes through the invariant line. Right: the
recovered rank-$3$ subspaces in the same coordinates.
Unlike \poolPCA{}, 
both \AnchorPCA{\lambda = 25}
and \AnchorPCAinfty{} 
contain the true invariant direction $b$ and differ
only in the third direction ($\lambda = 25$ puts more emphasis on explained variance, so \AnchorPCA{\lambda = 25} is pulled toward the higher variance direction $w$%
).
}
\label{fig:motivating_example_4d_geometry}
\end{figure}

Informally, 
we consider a direction invariant if it
helps explain variance
in every domain (the domain acts as an `anchor' specifying what we would like to be invariant against \citep[see also][]{rothenhausler2020anchor}).
In this example, 
$\operatorname{span}(a, b)$ is an
invariant space:
in all domains, $a$ and $b$ are among the top-$k$ eigenvectors.

\cref{fig:motivating_example_4d_geometry}
visualizes the subspaces that are chosen by the different methods.
The figure shows the qualitative behavior that will recur
throughout the paper. 
\poolPCA{} follows pooled variance and contains the space spanned by the 
nuisance directions: it keeps $a$ but misses the invariant direction~$b$. By
contrast, 
both $\AnchorPCA{\lambda = 25}$ and \AnchorPCAinfty{} 
recover the invariant span
exactly. The difference between 
these two lies in the third direction:
$\AnchorPCA{\lambda=25}$ 
puts a stronger focus on the explained variance and is thus pulled more strongly toward $w$\looseness-1, which is the direction with largest eigenvalue (see~\cref{tab:motivating_example_4d}).

In this paper, we focus on the subspace and not the basis vectors; that is, if in domain 3, the top-$k$ eigenspace contained $\tilde{a}, \tilde{b}$, instead of $a,b$, such that $\operatorname{span}(a,b)=\operatorname{span}(\tilde{a},\tilde{b})$, we would still regard $\operatorname{span}(a,b)$ an invariant subspace.
The other directions $u$, $v$, and $w$ can be considered nuisance directions (in the light of invariance).

\section{Maximal invariant subspace and Anchor PCA} \label{sec:definitions}
We now formally define the notion of a maximal invariant subspace and the Anchor PCA problem.
\subsection{Maximal invariant subspace} 
Given source domains $\cE$ with covariances
$(\Sigma_e)_{e \in \cE}$,
a natural choice 
is to ignore the domain structure and 
solve PCA on the pooled covariance $\barSigma$. We refer to this procedure as \poolPCA.
Formally, a
matrix $W^\star$
solves 
rank-$k$ \poolPCA{} if
$W^\star\in\argmax_{W\in\Ok}\Tr(W^\top\barSigma W)$.
Indeed, for any $W=[w_1,\dots,w_k]\in\Ok$, we have
$\Tr(W^\top\barSigma W)=\sum_{j=1}^k w_j^\top \barSigma w_j
=\Tr(\barSigma\Pi_W)$, which is the total variance explained by projecting
onto the candidate subspace $\operatorname{span}(W)$; hence, maximizing the trace recovers the
$k$ leading eigenvectors of $\barSigma$.
By definition, \poolPCA{}
focuses on the average explained variance, which may result in large differences in explained variance 
across source domains. 

\looseness-1 
In this paper, we propose to focus on 
directions of variation that are shared among all domains. 
In particular, we can consider the 
directions appearing within the top-$k$ principal directions in all domains.    
\begin{definition}[Maximal invariant subspace]
Let $k < p$ 
be fixed. The \emph{maximal invariant subspace} and its dimension %
are defined as
\begin{equation*}
\Sstar:=\bigcap_{e=1}^E \Imop\left(\Pi_k^{(e)}\right),
\qquad
m:=\dim(\Sstar).
\end{equation*}
\end{definition}
(If there are eigenvalue ties, $\Sstar$ and $\barPi$ and the \texttt{AnchorPCA} solutions---defined below---depend on a fixed choice of a rank-$k$ principal projector in each domain.)
The definition of \(\Sstar\) depends on~\(k\):
in general, the dimension $m$ of $\Sstar$ is at most $k$, but varies according to the problem at hand.
In Example~\ref{subsec:motivating_example_4d}, \(m=2\) and \(\Sstar=\operatorname{span}(a,b)\).
Equivalently, we can define $\Sstar$ as the eigenspace of $\barPi$ associated with eigenvalue 1;
indeed, for every unit vector $v$, we have 
$v^\top\barPi v = \frac1E\sum_{e=1}^E v^\top\Pi_k^{(e)}v \le 1$,
and equality holds if and only if $\Pi_k^{(e)}v=v$ for every \(e\), that is, if and only if \(v\in\Sstar\).
We will use this characterization when estimating $\Sstar$ from data (\Cref{subsec:identifying_inv_subspace}).

\subsection{Anchor PCA}
In practice, we 
may not expect to find any 
(nonzero) 
direction that is exactly invariant, in which case $\Sstar = \{0\}$ is trivial.
Furthermore, we may be interested in finding a subspace of fixed dimension~$k$, rather than one whose dimension 
can %
vary depending on the %
concrete set of
covariance matrices.   
We therefore propose selecting a rank-$k$ subspace that minimizes disagreement with the domain-specific top-$k$ eigenspaces. 
Such a space may be represented using different 
bases; we propose to choose one that maximizes pooled explained variance.
We refer to this approach as \AnchorPCAinfty{}. 
\begin{definition}[\AnchorPCAinfty{}]
\label{def:anchorpca_infty}
A matrix 
$W^{\anchor,\infty}_{k} \in\Ok$ 
solves \emph{rank-$k$ \AnchorPCAinfty{}} 
if
\begin{equation} \label{eq:anchorpca_infty_def}
W^{\anchor,\infty}_{k}
\in\argmax_{W\in\Ok
}\Tr\left(W^\top\barSigma W\right)
\quad \text{ subject to: } \quad    W%
\in \argmin_{\widetilde W\in\Ok}\sum_{e=1}^E\fnorm{\Pi_{\widetilde W}-\Pi_k^{(e)}}^2.
\end{equation}
\end{definition}
\looseness-1 We prove in \cref{sec:robust_reconstruction} that \AnchorPCAinfty{} satisfies an asymptotic minimax-optimality guarantee for reconstruction error under additional variance inflations of arbitrary strengths in certain directions.
If we only expect mild inflations, however, \AnchorPCAinfty{} is 
overly
conservative.
We therefore also consider the following version in which the hard invariance constraint is relaxed to a soft penalty.
\begin{definition}[\AnchorPCA{}]
\label{def:anchorpca}
For $\lambda\ge0$, a matrix 
$\Wanchor{k} \in\Ok$ 
solves \emph{rank-$k$ \AnchorPCA{}} 
if
\begin{equation} \label{eq:anchorPCA}
    \Wanchor{k} \in \argmax_{W\in\Ok}
\left\{
\Tr\left(W^\top\barSigma W\right)
-\lambda\sum_{e=1}^E \fnorm{\Pi_W-\Pi_k^{(e)}}^2
\right\}.
\end{equation}
For $\lambda\in [0,\infty]$, the corresponding 
\emph{rank-$k$ \AnchorPCA{} subspace} and \emph{projection matrix} 
are given by 
$\mathcal S^{\anchor,\lambda}_k := \Imop(W^{\anchor,\lambda}_{k})$ and
$\Pi^{\anchor,\lambda}_{k}:=W^{\anchor,\lambda}_{k}(W^{\anchor,\lambda}_{k})^\top$, respectively.
\end{definition}
While \AnchorPCAinfty{}
only considers maximally invariant solutions and explained variance serves as a secondary objective, \AnchorPCA{} balances the two, with the trade-off controlled by the invariance level $\lambda$.
We prove in \cref{sec:relation_methods}
that, as $\lambda$ grows, the two solutions converge to each other: under suitable assumptions, we have  
\[
\opnorm{\Pianchor{k}-\Pi^{\anchor,\infty}_{k}} \to 0 \qquad \text{as } \lambda\to\infty.
\]
\looseness-1 \Cref{thm:invariant_subspace_limit_extended}~(ii) 
studies the convergence at 
the level of directions rather than subspaces and projections.

Both Anchor PCA methods can be extended to a soft agreement score: even when \(\Sstar\) is too small or trivial, directions with larger \(\barPi\)-eigenvalues are still preferred, rather than requiring 
exact equality of eigenvalues to one.

\subsection{Projection-based formulation} \label{sec:projection_based_formulation}
The optimization problems~\eqref{eq:anchorpca_infty_def}
and~\eqref{eq:anchorPCA}
allow for an equivalent, projection-based view: 
expanding the penalty in~\eqref{eq:anchorpca_infty_def} gives 
$\sum_{e=1}^E\|{\Pi_W-\Pi_k^{(e)}}\|_\mathrm{F}^2 = \sum_{e=1}^E 2k - 2\Tr(\Pi_W\Pi_k^{(e)}) = 2Ek-2E\Tr(W^\top\barPi W)$.
Hence, the constraint is equivalent 
to maximizing $\Tr(W^\top\barPi W)$ over $\Ok$.
Similarly, for~\eqref{eq:anchorPCA}, 
solutions 
are obtained by solving a standard rank-$k$ PCA on a modified target matrix $M_\lambda$
\begin{equation}
\label{eq:anchorpca_spectral}
\max_{W\in\Ok}\Tr(W^\top M_\lambda W),
\qquad \text{where} \qquad
M_\lambda:=\barSigma+2E\lambda\,\barPi.
\end{equation}

This view 
has two advantages:
(i)~It suggests computationally feasible ways to solve the 
optimization problems of 
  \AnchorPCAinfty{} 
  and
  \AnchorPCA{} 
(see \cref{sec:algorithm}).
(ii)~It allows for a geometric interpretation: 
$\Tr(\Pi_W\Pi_k^{(e)})$ is the sum of squared cosines of the principal angles
between $\operatorname{span}(W)$ and the 
principal subspace of domain~$e$.
Thus, \AnchorPCAinfty{} first maximizes average overlap with the
domain-specific principal subspaces and then maximizes pooled explained variance
among 
those solutions.
\subsection{Choosing an ordered basis} \label{subsec:orderedbasis}
Both objectives~\eqref{eq:anchorpca_infty_def}
and~\eqref{eq:anchorPCA} depend on $W$ only through
$\Pi_W=WW^\top$, 
so any 
orthonormal basis of a solution space defined by $W$ forms another solution;
statements about individual columns require 
choosing an ordered basis. 
We propose to use the following convention for choosing an ordered basis 
(see \cref{app:ordered_representatives}
for details). 
For \AnchorPCAinfty{}, we first order directions by
decreasing agreement,
that is, by decreasing $\barPi$-eigenvalue; this
creates blocks of equal eigenvalues (which we refer to in the algorithm and proof sections); 
we then
choose a basis that diagonalizes $\barSigma$ restricted to the $\barPi$-eigenspace corresponding to that block and order the vectors 
by
non-increasing $\barSigma$-eigenvalue, equivalently, explained variance.
For \AnchorPCA{}, we use an ordered top-$k$ eigenbasis of
$M_\lambda$, see \cref{eq:anchorpca_spectral}, giving the usual sequential PCA interpretation for the modified
covariance $M_\lambda$.
Both of these conventions are implemented in our code, see~\cref{sec:algorithm}.

We will see in \cref{sec:recovery_guarantees}
and \cref{sec:robust_reconstruction}
that 
both versions of 
\texttt{AnchorPCA} 
can find invariant directions
and 
come with improved recovery and robustness properties compared to \poolPCA{}.

\section{Recovery guarantees} \label{sec:recovery_guarantees}
We now study when 
\AnchorPCAwout{}
recovers invariant structure.
As motivated in Example~\ref{subsec:motivating_example_4d}, the subspace
target is $\Sstar=\cap_{e=1}^E \Imop(\Pi_k^{(e)})$,
the part of the local top-$k$ principal subspaces shared by all domains. Since
\AnchorPCA{} and \AnchorPCAinfty{} return rank-$k$ subspaces, there are two natural recovery
questions: whether these subspaces contain $\Sstar$, and whether their leading
$m = \dim(\Sstar)$ directions span $\Sstar$. 
We first prove that the \AnchorPCAinfty{} subspace $\mathcal{S}_k^{\anchor, \infty}$ contains $\Sstar$, and that 
\AnchorPCA{} identifies $\Sstar$ asymptotically as $\lambda\to\infty$. 

\begin{theorem}[Invariant subspace identification and asymptotic convergence]
\label{thm:invariant_subspace_limit}
Let $m=\dim(\Sstar) \le k$. Let $\Pi_k^{\anchor,\infty}$ be any rank-$k$
\AnchorPCAinfty{} projection matrix. Then
\[
\Pi_k^{\anchor,\infty}\Pi_{\Sstar}=\Pi_{\Sstar},
\qquad\text{equivalently}\qquad
\Sstar\subseteq \Imop(\Pi_k^{\anchor,\infty}).
\]

Moreover, for $\lambda>0$, choose the \AnchorPCA{} solution
$\Wanchor{k}=(w_1^{\anchor,\lambda},\dots,w_k^{\anchor,\lambda})$
to be an ordered orthonormal eigenbasis of $M_\lambda$ associated with its $k$
largest eigenvalues (see~\cref{subsec:orderedbasis}). Let $\Pianchor{m}$ denote the projection matrix onto
$\operatorname{span}(w_1^{\anchor,\lambda},\dots,w_m^{\anchor,\lambda})$.
Then, there exists $\lambda_0$ such that for all $\lambda > \lambda_0$,
the span of these first $m$
eigenvectors is unique and
\[
\operatorname{span}(w_1^{\anchor,\lambda},\dots,w_m^{\anchor,\lambda})
\to\Sstar,
\quad\text{i.e.,}\quad
\opnorm{\Pianchor{m}-\Pi_{\Sstar}}\to0
\quad
\text{as }\lambda\to\infty.
\]
\end{theorem}
\looseness-1 Under the ordered basis convention in~\cref{app:anchor_pca_infty_basis_convention}, also the first $m$
 \AnchorPCAinfty{} 
 directions span $\Sstar$ (see \cref{thm:invariant_subspace_limit_extended}).
Moreover, \cref{app:exact_finite_recovery} shows, under an additional block-diagonality assumption on~$\barSigma$, that \AnchorPCA{} recovers $\Sstar$ already for finite values of $\lambda$ (this assumption holds in the example of \cref{subsec:motivating_example_4d}).

\section{Optimal reconstruction under perturbations}
\label{sec:robust_reconstruction}

We now study robustness of the solutions under structured covariance
inflations. 
It is straightforward to see
that $\Sstar$ 
satisfies the following robustness guarantee under eigengap conditions on the 
covariances $\Sigma_e$ (see \cref{prop:findsstar_exact_robustness} for details):
$\Sstar$ is invariant to arbitrarily large admissible domain-specific
top-$k$ variance inflations; more precisely, 
it remains the maximally invariant subspace 
under perturbations 
of the form 
$\Sigma'_e=\Sigma_e+\Delta_e$ with $0\preceq\Delta_e$ and
$\Imop(\Delta_e)\subseteq \Imop(\Pi_k^{(e)})$ for all $e\in\cE$.

We now
prove that both versions of
\AnchorPCAwout{}
satisfy a similar 
guarantee
under bounded versions of the same perturbations. More precisely,
for a covariance matrix $\Sigma$, the expected reconstruction error of a rank-$k$
projector $\Pi_W$ is given by
$
\mathbb E[\|x-\Pi_Wx\|_2^2]=\Tr(\Sigma(I_p-\Pi_W))$.
We further define, for any $\lambda\ge0$, 
the following set of perturbed covariances
\begin{equation}
\label{eq:C_lambda}
\mathcal C_\lambda
:=
\Bigl\{
(\Sigma'_1,\dots,\Sigma'_E)\,\Big|\, \Sigma'_e=\Sigma_e+\Delta_e\in\Spp,\ \
0\preceq \Delta_e \preceq 2E\lambda\,\Pi_k^{(e)} \text{ for all } e\in\cE
\Bigr\}
\end{equation}
and obtain the following result.
\begin{theorem}[Robust reconstruction]
\label{thm:minimax_reconstruction}
For $W\in\Ok$, we define the average reconstruction error 
over $\{\Sigma_e\}_{e\in\cE}$ as
$
\mathcal R({W};\Sigma_1, \ldots, \Sigma_E)
:=
\frac1E\sum_{e=1}^E \Tr(\Sigma_e(I_p-\Pi_W))$.
Then, (i),
\[
\Wanchor{k}\in\argmin_{W\in\Ok} \max_{(\Sigma'_1,\dots,\Sigma'_E)\in\mathcal C_\lambda} \mathcal R({W};\Sigma'_1,\dots,\Sigma'_E) \; \iff\; \Wanchor{k} \text{ solves rank-}k ~\AnchorPCA{}.
\]
and,
(ii), with 
$\mathcal{B}:=
\argmin_{\widetilde W\in\Ok}\sum_{e=1}^E \|{\Pi_{\widetilde W}-\Pi_k^{(e)}}\|_\mathrm{F}^2
$
\[
W_k^{\anchor,\infty}
\in \argmin_{W\in\mathcal{B}}
\mathcal R({W};\Sigma_1, \ldots, \Sigma_E)
\quad \iff\quad 
W_k^{\anchor,\infty} \text{ solves rank-}k~ \AnchorPCAinfty{}.
\]

\end{theorem}

Thus, if we 
consider target domains that are subject to 
covariance inflations described by the set $\mathcal{C}_{\lambda}$, then 
\AnchorPCA{} minimizes the worst-case reconstruction error over such domains.
We now aim to provide intuition for part~(ii). The proof of \cref{thm:minimax_reconstruction}
shows that 
\begin{equation} \label{eq:worst_case_reconstruction_error_decomposition}
    \max_{(\Sigma_1', \ldots, \Sigma_E') \in \mathcal{C}_\lambda}
\mathcal R({W};\Sigma'_1,\dots,\Sigma'_E)
 =
 \frac1E\sum_{e=1}^E \Tr\bigl(\Sigma_e(I_p-\Pi_W)\bigr)
 +
 \lambda\sum_{e=1}^E \|{\Pi_W-\Pi_k^{(e)}}\|_\mathrm{F}^2.
\end{equation}
Equation~\eqref{eq:worst_case_reconstruction_error_decomposition} shows that the worst-case reconstruction error is affine in $\lambda$, with slope equal to the projector-disagreement penalty. \AnchorPCAinfty{} first minimizes this slope over rank-$k$ projectors and then, among projectors with minimal slope, minimizes the nominal reconstruction error.%
\footnote{
A single minimax problem over the
unbounded union $\bigcup_{\lambda>0}\mathcal C_\lambda$ is typically
degenerate, as can be seen from Equation~\eqref{eq:worst_case_reconstruction_error_decomposition}. So the constrained formulation is the natural non-asymptotic statement.
}

\looseness-1 In finite samples, covariances from new environments will not exactly lie in the set $\mathcal{C}_
\lambda$~\eqref{eq:C_lambda}. For example,
covariance estimation errors generally do not only lie in the local top-$k$ subspaces, are not necessarily positive semidefinite, and they may be full rank.
We provide a relaxation of \cref{thm:minimax_reconstruction} in \cref{cor:full_rank_robustness}, in which we allow the deviations $\Delta_e$ to be full-rank with positive
and negative eigenvalues, provided the resulting covariances
$\Sigma_e+\Delta_e$ remain positive semidefinite and satisfy
$\Sigma_e+\Delta_e\preceq\Sigma_e+2E\lambda\Pi_k^{(e)}+\varepsilon_e I_p$. 
Here, $\varepsilon_e\ge0$ specifies how much extra variance is allowed in any
direction in domain $e$.
\AnchorPCA{} still minimizes the worst-case reconstruction error over this
broader class, though allowing more extra variance (larger $\varepsilon_e$) increases
the worst-case reconstruction error.

\cref{fig:motivating_example_4d_perturbation_path} 
visualizes parts of \cref{thm:minimax_reconstruction}. 
It shows that 
\poolPCA{} is optimal only for small perturbations, whereas
\AnchorPCA{} and \AnchorPCAinfty{} guard against medium-sized and large perturbations.
Further details are provided in \cref{subsec:perturbation_path_illustration}.
The experiment in \cref{subsec:exp_gas_sensor} shows that robustness guarantees like the ones shown in \cref{thm:minimax_reconstruction} may be of practical relevance.

\begin{figure}[tbp]
    \centering
    \includegraphics[width=1\linewidth
    , trim={0 0 0 0}, clip
    ]{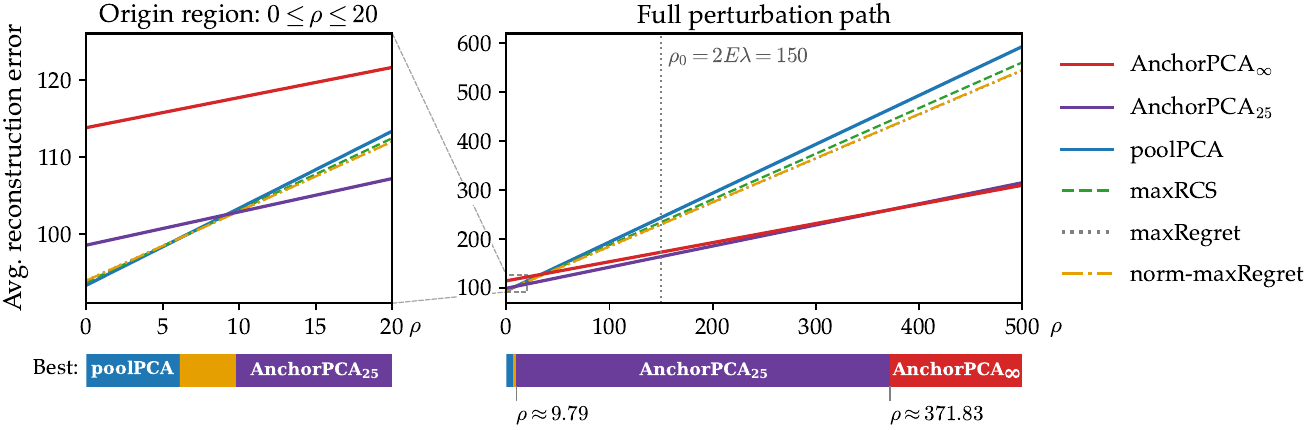}
    \caption{
    Average reconstruction error along the perturbation path
    $\Sigma_e^{(\rho)}=\Sigma_e+\rho\,\Pi_k^{(e)}$ for the example in \cref{subsec:motivating_example_4d}.
    As suggested by \cref{thm:minimax_reconstruction}, \poolPCA{} is best
    only for small perturbations ($0\le \rho \lesssim 6.08$).
    Among the wcPCA~\citep{fries2026worstcaselowrankapproximations} baselines, \texttt{maxRCS} is best only on the tiny interval
    $6.08 \lesssim \rho \lesssim 6.15$, after which
    \texttt{norm-maxRegret}
    is best until $\rho \lesssim 9.79$.
    \AnchorPCA{25} dominates on the broad middle regime
    $9.79 \lesssim \rho \lesssim 371.83$, and \AnchorPCAinfty{} performs best
    for strong perturbations $\rho \gtrsim 371.83$.
    }
    
    \label{fig:motivating_example_4d_perturbation_path}
\end{figure}

\section{Learning methods for Anchor PCA} \label{sec:algorithm}

We now describe the finite-sample methods
for both \AnchorPCAwout{} versions.
Suppose domain \(e\in\cE\) provides observations
\(x_{e,1},\dots,x_{e,n_e}\in\R^p\).
Define 
$\bar x_e:=\frac1{n_e}\sum_{i=1}^{n_e}x_{e,i}$
and
the empirical covariance
$\widehat\Sigma_e
:=
\frac{1}{n_e-1}\sum_{i=1}^{n_e}
(x_{e,i}-\bar x_e)(x_{e,i}-\bar x_e)^\top$.
Let \(\widehat W_e\in\R^{p\times k}\) be a top-\(k\) eigenbasis of
\(\widehat\Sigma_e\), and write the empirical versions of the domain-specific projection matrix as
$\widehat\Pi_k^{(e)}:=\widehat W_e\widehat W_e^\top$, 
average covariance as $\hatSigma:=\sum_{e=1}^E \widehat\Sigma_e/E$, and 
average projection matrix as
$\hatPi:=\sum_{e=1}^E \widehat\Pi_k^{(e)}/E$, where we omit the bar compared to their population-level counterparts $\barSigma, \barPi$ for readability.
We have weighted each domain here equally regardless of sample size.
Alternatively, one may weight domains in proportion to their sample sizes, $w_e = n_e/\sum_j n_j$, which corresponds to weighting individual observations equally across domains.

\subsection{Finite-sample implementation}
\label{subsec:finite_sample_implementation}

We define empirical \AnchorPCA{} 
as follows: we compute
rank-$k$ PCA of
$\widehat M_\lambda:=\hatSigma+2E\lambda\,\hatPi$,
see
Equation~\eqref{eq:anchorpca_spectral}, 
and return an ordered top-$k$ eigenbasis
$\widehat W_k^{\anchor,\lambda}$ of $\widehat M_\lambda$.
\AnchorPCAinfty{} ranks directions first by cross-domain agreement, measured by the eigenvalues of $\barPi$ (with population eigenvalue~$1$ corresponding to $\Sstar$), and then, among directions with equal agreement, by explained variance. A
plug-in implementation replacing population quantities by their empirical counterparts 
would not yield a consistent method:
for example,
due to estimation error, 
all eigenvalues of $\hatPi$ would generally
be distinct, creating 
blocks of equal eigenvalues of size one (see \Cref{subsec:orderedbasis}).
Instead, we group consecutive $\hatPi$-eigenspaces whenever eigenvalues remain within $\mathrm{tol}_n$ of the first eigenvalue of that group, where $n:=\min_{e\in\Ecal} n_e$;
for each of these groups of eigenspaces, we compute the direct sum and denote the resulting spaces by
$\widehat{\operatorname{Eig}}_1, \widehat{\operatorname{Eig}}_2, \ldots$, and rank directions within each group by their variance under $\hatSigma$;
see~\Cref{alg:anchorpca_infty} for details. 
If $\Sstar\neq\{0\}$ and $\mathrm{tol}_n$ separates the empirical eigenvalue-$1$
block from the rest of the spectrum (in our experiments we use $\mathrm{tol}_n = \min\{0.05,0.5n^{-0.4}\}$ as a default choice),
$\widehat{\operatorname{Eig}}_1$ consistently estimates $\Sstar$ (see \Cref{subsec:consistency}).
Thus, the algorithm also returns
estimates of the invariant subspace and its dimension.
\begin{algorithm}[tbp]
\caption{Block-stabilized empirical $\AnchorPCAinfty{}$}
\label{alg:anchorpca_infty}
\begin{algorithmic}[1]
\Statex \textbf{Input:} 
$\hatSigma$ (average of domain-specific covariances), $\hatPi$ (average rank-$k$ projectors), $k,\mathrm{tol}_n$
\Statex \textbf{Output:} $\widehat W_k^{\anchor,\infty}$, $\widehat\Pi_k^{\anchor,\infty}$, $\widehat{\Sstar}$, $\widehat m_\infty$

\Statex \(\triangleright\) \textit{Step 1: Measure cross-domain agreement.}
\State 
Eigendecompose $\hatPi$ as
$\hatPi=\sum_{j=1}^p \widehat\rho_j \widehat v_j \widehat v_j^\top$
with $\widehat\rho_1\ge\cdots\ge\widehat\rho_p$.
\Statex \(\triangleright\) \textit{Step 2: Group eigenvalues that lie within $\mathrm{tol}_n$ of each other.}
\State Set $a\gets 1$ and $\ell\gets 1$.
\While{$a\le p$}
    \State Set $b\gets a$.
    \While{$b<p$ and $\widehat\rho_a-\widehat\rho_{b+1}\le \mathrm{tol}_n$}
        \State Set $b\gets b+1$.
    \EndWhile
    \State Set $I_\ell\gets\{a,\ldots,b\}$ and
    $\widehat{\operatorname{Eig}}_\ell
    \gets \operatorname{span}\{\widehat v_j:j\in I_\ell\}$.
    \State Set $a\gets b+1$ and $\ell\gets \ell+1$.
\EndWhile

\Statex \(\triangleright\) \textit{Step 3: Fill $k$ directions block by block, ranking by variance within a block.}
\State Let $r$ be the first block index such that
$\sum_{\ell\le r}\dim(\widehat{\operatorname{Eig}}_\ell)\ge k$, and set
$q\gets k-\sum_{\ell<r}\dim(\widehat{\operatorname{Eig}}_\ell)$.

\State For all $\ell<r$, take an ordered eigenbasis of $\hatSigma$ within
$\widehat{\operatorname{Eig}}_\ell$.
\State 
Take the top-$q$
eigenvectors of $\hatSigma$ within $\widehat{\operatorname{Eig}}_r$. 
\State Concatenate these vectors to obtain $\widehat W_k^{\anchor,\infty}$.
\State Set
$\widehat\Pi_k^{\anchor,\infty}
\gets \widehat W_k^{\anchor,\infty}
(\widehat W_k^{\anchor,\infty})^\top$.
\Statex \(\triangleright\) \textit{Step 4: Estimate the invariant subspace.}
\State Set
$\widehat{\Sstar}\gets\widehat{\operatorname{Eig}}_1$ and
$\widehat m_\infty\gets\dim(\widehat{\operatorname{Eig}}_1)$.
\end{algorithmic}
\end{algorithm}

\textbf{Choosing $k$.}
The dimension $k$ determines the output dimension
and the local top-$k$ spaces whose intersection defines $\Sstar$.
Similar to PCA, we suggest to use domain knowledge when available,
inspect pooled and source-domain scree plots~\citep{cattell1966scree}, or choose the smallest $k$ meeting a prespecified source explained variance
requirement. For a downstream task, source-domain validation can select the
smallest $k$ within one standard error of the best score.

\textbf{Selecting the penalty.}
The penalty $\lambda$ controls the trade-off between explained variance and agreement
across domains. For fixed $k$, we propose leave-one-source-domain-out cross-validation
of $2E\lambda$, rather than $\lambda$, since $2E \lambda$ is the coefficient that weights $\barPi$ in
$M_\lambda=\barSigma+2E\lambda\barPi$. With at least three source domains,
we select the candidate with the highest mean held-out proportion of 
variance explained, then refit on all sources. We call this \AnchorPCA{\widehat{\lambda}}. If source labels and a downstream task are available, the held-out
score can instead be a task metric. The candidate grid and implementation
details are given in \cref{app:gas_penalty_selection}.

\subsection{\texorpdfstring{\texttt{FindS}$_\star$}{FindS*}: Identifying the invariant subspace \texorpdfstring{$\Sstar$}{S*}}\label{subsec:identifying_inv_subspace}

In some settings, the object of interest is not a full 
$k$-dimensional subspace (as estimated by \AnchorPCAwout{}), but only the maximally invariant subspace $\Sstar$. We now describe how to estimate an unordered basis of $\Sstar$ directly.

\textbf{An eigenvalue problem.}
Recall that $\Sstar$ is the eigenspace of $\barPi$ associated with eigenvalue $1$ (\cref{sec:definitions}). A natural approach would therefore be to identify eigenvalues equal to $1$. While this characterization is exact at the population level, 
it cannot be applied directly to finite samples.
Although bootstrapping projection matrices is feasible \citep[e.g.,][]{beran1985bootstrap}, the null hypothesis $
\lambda_i(\barPi)=1$ 
lies on the boundary of the parameter space, and empirical eigenvalues will have value strictly less than $1$, making the use of bootstrap-based confidence intervals difficult. Moreover, bootstrapping the null or using a parametric bootstrap is challenging, as the null is composite and depends on an unknown invariant subspace.

\textbf{Testing the dimension of $\Sstar$.}
We instead directly test the dimension of $\Sstar$. 
If $\dim(\Sstar)=m$ were known, $\Sstar$ is estimated by the eigenspace of $\hatPi$ associated with its $m$ largest eigenvalues. 
Let \(m_{\min}:=\max(0,p-E(p-k))\).
This motivates the 
sequential testing procedure 
\begin{equation}\label{eq:hypothesis-test}
    H_{0,s}: \dim(\Sstar) = s
    \quad \text{vs} \quad
    H_{A,s}: \dim(\Sstar) < s,
\end{equation}
for $s\in\{k, k-1, \dots, m_\mathrm{min}+1\}$,
stopping at the first $s$ for which the null is not rejected; if all nulls are rejected, estimate $\dim(\Sstar)$ as $m_\mathrm{min}$.
To construct such tests, we adapt ideas from partial common principal component analysis (PCPCA) subspaces \citep{schott1999}. In the PCPCA setting, it is assumed that for all $e\in\cE$, we observe a sample of i.i.d.\ Gaussian observations 
with
population
positive-definite covariance $\Sigma_e$. Additionally, $\Sigma_e$ is assumed to satisfy a strict eigengap between its $k$\textsuperscript{th} and $(k+1)$\textsuperscript{st} eigenvalues. 
Under these assumptions, Wald-type tests are available for the dimension of the union of the top-$k$ eigenspaces of $(\Sigma_e)_{e\in\cE}$.
To translate this to our setting, we consider the orthogonal complement projector matrices
$
I_p - \Pi_k^{(e)}$ and
let 
$r$ be the dimension of the union of their eigenspaces.
Since a vector belongs to $\Sstar$ if and only if it lies in the null space of $I_p-\Pi_k^{(e)}$ for every $e\in\cE$, $r = p-\dim(\Sstar)$.
We apply the PCPCA subspace test to the bottom-$(p-k)$ eigenspaces of the covariance matrices, equivalently to $\Imop(I_p-\Pi_k^{(e)})$, and
thus obtain a sequential testing procedure for~\eqref{eq:hypothesis-test}.
Combining this sequential test with the eigenspace estimator above, we define \texttt{FindS}$_\star$ as the procedure that first estimates $\dim(\Sstar)$ and then returns the eigenspace of $\hatPi$ associated with its estimated number of largest eigenvalues.
The test statistic and further details are given in \cref{app:wald-test}.

\subsection{Consistency of Anchor PCA}
\label{subsec:consistency}
Both \AnchorPCAwout{} methods are consistent under mild conditions 
(the full statement and proof are given in \cref{app:consistency}).
Assume that the minimum domain sample size $n := \min_{e \in \cE} n_e \to \infty$, that observations are centered with finite second moments, and that each domain covariance $\Sigma_e$ has a strict eigengap at rank $k$. Then $\hatSigma \to \barSigma$ and $\hatPi \to \barPi$ almost surely, and \AnchorPCA{} is consistent whenever $M_\lambda$ itself has a strict eigengap at rank $k$.
For \AnchorPCAinfty{}, consistency requires a little more care: the block tolerance used in the algorithm must shrink to zero, but slowly enough relative to the estimation error of $\hatPi$. Any tolerance of the form $\mathrm{tol}_n = c n^{-\alpha}$ with $0 < \alpha < 1/2$ works under finite fourth moments. 
Under these conditions, the full rank-$k$ projector is consistent when the population \AnchorPCAinfty{} solution projector is unique; moreover, if $\Sstar\neq\{0\}$, the first block returned by \AnchorPCAinfty{} is a consistent estimator of $\Sstar$.
For \texttt{FindS}$_\star$, Schott's asymptotic theory \citep{schott1999} 
implies that the probability of correctly estimating $\dim(\Sstar)$ converges 
to $1-\alpha$, where $\alpha$ is the nominal level of the tests.

\section{Experiments} \label{sec:experiments}

Code to reproduce all experiments and figures is publicly available
at \url{https://github.com/benesei/anchor-pca}.

\subsection{Recovering the invariant subspace \texorpdfstring{$\Sstar$}{S*}}
\label{subsec:exp_sstar_recovery}

In this experiment, all domains have equal sample size, which we call $N$ (which thus also equals $n=\min_{e\in\cE} n_e$). 
\cref{subsec:consistency}
proves
that the
first grouped eigenspace, $\widehat{\mathrm{Eig}}_1$, of the
block-stabilized \AnchorPCAinfty{} construction consistently estimates
$\Sstar$ when the tolerance shrinks while dominating the empirical perturbation.
 We test this statement
 in a random-subspace model where
$\Sstar$ is known, using the default 
block tolerance ($\mathrm{tol}_N = \min\{0.05, 0.5N^{-0.4}\}$).
Below we choose the value
$m=\dim(\Sstar)$ and let $\widehat m_\infty$ denote the dimension of the estimate $\widehat{\Sstar}$. 
We measure subspace error by
$\|\Pi_{\widehat S_\star}-\Pi_{\Sstar}\|_\mathrm{op}$.

\looseness-1 
The 
distribution draws
are as follows. We draw a random
$m$-dimensional subspace $\Sstar\subset\R^p$ and, for each domain, a
different rank-$k$ principal subspace containing $\Sstar$. The covariance in
that domain is then constructed so that this subspace is exactly its
top-$k$ eigenspace. We consider two DGP regimes. 
In the easy
setting, variances along shared and domain-specific top directions are comparable; in the difficult
setting, shared directions have lower variance.
\Cref{app:random_subspace_sstar_recovery}
provides details and two 
more challenging configurations: a smaller invariant
subspace ($m=1$) and 
a configuration with only two domains $(E=2 < k=5)$.

\begin{figure}[tbp]
    \centering
    \includegraphics[width=1\linewidth]{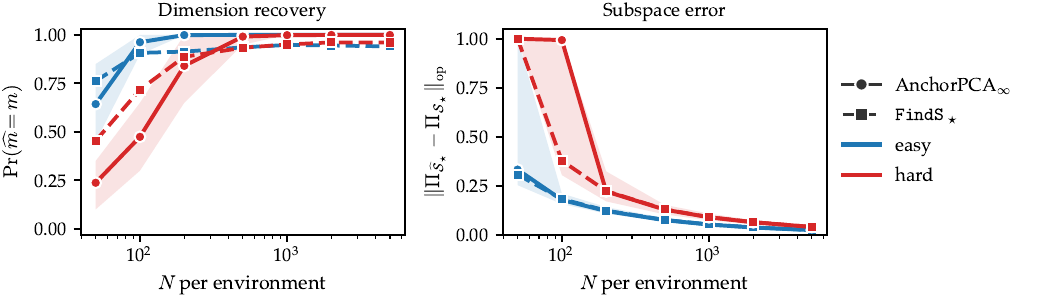}
    \caption{\looseness-1 
    \textbf{Recovery of $\Sstar$ by \AnchorPCAinfty{} and $\texttt{FindS}_\star$ (\cref{subsec:identifying_inv_subspace}).}
    Gaussian data with $E=5$, $p=10$, $k=5$, and $m=\dim(\Sstar)=2$.
    We show mean correct-dimension probability (left) and  
    median
    operator-norm projector error (right) against the true $\Sstar$. 
    As predicted by \cref{subsec:consistency}, 
    the first grouped eigenspace of
    \AnchorPCAinfty{} (solid) recovers $\Sstar$ with growing sample size and
     \texttt{FindS}$_\star$ (dashed)
    has dimension-recovery error close to the nominal test size of $0.05$. 
   For each sample size $N$, results are aggregated
   over $100$ independent distribution draws, 
   for each of which we draw 
   $20$
independent datasets. Shaded bands show $10$--$90\%$ quantiles across
distribution draws (shown only for \AnchorPCAinfty{}). 
    }
    \label{fig:infty_sstar_recovery_main}
\end{figure}    

\looseness-1
\Cref{fig:infty_sstar_recovery_main} is in agreement with the consistent recovery
of $\Sstar$. In the easy setting, the correct block dimension is
already recovered in about $64\%$ of runs at $N=50$ and reaches probability $1$
by $N=5000$; in the hard setting, the corresponding probability is
about $24\%$ at $N=50$ but also reaches $1$ by $N=5000$. The median projector
error decreases to about $0.02$ in the easy setting and $0.04$ in the
difficult setting at $N=5000$. 
The dashed curves in \cref{fig:infty_sstar_recovery_main} show
that $\texttt{FindS}_\star$ achieves similarly low subspace error, while its dimension-recovery error remains close to $0.05$ for large sample sizes, 
the nominal level of the individual tests, consistent with \cref{subsec:consistency}.
Our theory (\cref{subsec:consistency}) 
for 
Anchor PCA
does not assume Gaussianity. Indeed, 
\cref{app:sstar_recovery_exp}
shows the same
qualitative behavior
for Gaussian-mixture samples 
and \cref{app:unseen_target_ev} additionally compares explained variance on an unseen domain.
\AnchorPCAinfty{} and \poolPCA{} retain nearly identical target explained
variance (EV), but \AnchorPCAinfty{} recovers the true invariant subspace
$\Sstar$ more accurately than \poolPCA{} when using each method's first
$m$ directions.

\subsection{Gas sensor array drift}
\label{subsec:exp_gas_sensor}
We evaluate \AnchorPCAwout{} on gas-sensor drift data
\citep{gas_sensor_array_drift,fonollosa2015chemicalgas}, consisting of
$13{,}910$ recordings from $16$ chemical sensors, summarized by $p=128$
features and grouped into ten temporal batches. We treat batches as domains:
source batches are used for fitting and later batches are held out. The data
are known to exhibit sensor drift, i.e., gradual changes in sensor response
caused by aging, poisoning, or environmental effects, so that the same gas can
induce different response distributions over time
\citep{vergara2012chemical,fonollosa2015chemicalgas}.
The hope is that the projector agreement of \AnchorPCAwout{} 
across early batches 
downweights drift-specific high-variance directions and retains directions that
remain informative under future temporal shifts. We evaluate by normalized
batch explained variance
$100\cdot\Tr(W^\top\widehat\Sigma_b W)/\Tr(\widehat\Sigma_b)$, where $W$ is fit
only on source batches.

\begin{figure}[tbp]
    \centering
    \includegraphics[width=1\linewidth]{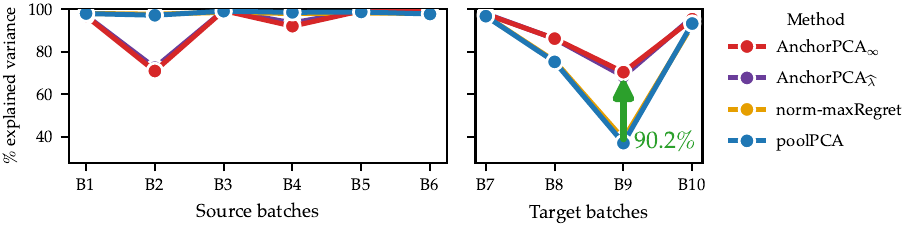}
    \caption{
    \looseness-1 
    \textbf{Gas-sensor drift, source B1--B6 and target B7--B10.}
All methods use $k=20$ and source-only preprocessing/fitting;
\texttt{norm-maxRegret} is the normalized regret wcPCA baseline
\citep{fries2026worstcaselowrankapproximations}. \AnchorPCA{\widehat{\lambda}} 
(which chooses $2E\lambda=100$)
exhibits the source--target compromise: it retains more
source explained variance than \AnchorPCAinfty{} but yields smaller average
target gains. \AnchorPCAinfty{} gives up more source explained variance,
achieves the highest mean target EV, and improves over \poolPCA{} on every held-out target batch; the arrow marks its largest relative
target gain over \poolPCA{} (B9, $90.2\%$).
}
    \label{fig:gas_sensor_b1_b6_k20}
\end{figure}

\Cref{fig:gas_sensor_b1_b6_k20} shows the results for the split B1--B6 $\to$ B7--B10 for
$k=20$. 
\AnchorPCA{\widehat{\lambda}}  (which chooses $2E\lambda = 100)$ 
illustrates
the following trade-off:
its mean source explained variance of $93.4\%$ lies between those of \poolPCA{} ($98.1\%$) and
\AnchorPCAinfty{} ($92.8\%$), but its mean target explained variance ($87.1\%$) outperforms
\poolPCA{} ($75.5\%$). 
\AnchorPCAinfty{} has the weakest source fit (mainly due to B2)
but 
yields a slightly higher target explained variance ($87.3\%$)
and improves
over \poolPCA{} on every held-out target batch. 
In \cref{app:gas_sensor_drift_details}, we consider alternative splits and values of $k$:
\AnchorPCA{\widehat{\lambda}} has higher
mean target EV than \poolPCA{} in $14$/18 configurations, while \AnchorPCAinfty{}
improves in $16$/18. 
In \cref{app:gas_b9_diagnostic} we examine the 2 configurations ($s=8$ and $k=20,30$), where \poolPCA{} outperforms Anchor PCA.
We also examine weighting domains
in proportion to sample size in \cref{app:gas_weighting}; under both schemes,
\AnchorPCAinfty{} has higher target EV than \poolPCA{} when averaged over the
rolling splits, at all three ranks.

To assess whether the representations also help a downstream task, we predict the six gas classes. Linear classifiers on the $20$-dimensional representations raise mean target accuracy from $68.6\%$ for \poolPCA{} to $74.6\%$ for \AnchorPCA{\widehat{\lambda}} and $75.6\%$ for \AnchorPCAinfty{}, competitive with unreduced features (\cref{app:gas_classification}).

\section{Discussion and future work} \label{sec:disc_and_future_work}
We introduced Anchor PCA, a novel framework for PCA on multiple domains. We prove recovery guarantees and robustness results and show its effectiveness on simulated and real data.
Our invariance notion (membership in the local top-$k$ principal subspace of every domain) depends on the choice of~$k$ and can be sensitive to near eigenvalue ties. The minimax guarantee in \cref{thm:minimax_reconstruction} covers bounded variance inflations inside local top-$k$ eigenspaces, but not mean shifts, variance deflations, or rotations of the principal subspaces. 
While we have focused on robust linear dimension reduction based on PCA, we discuss possible nonlinear extensions of \AnchorPCA{} based on autoencoders in~\cref{app:nonlinear_extensions}.

\section*{Acknowledgments}
The authors used GPT-5.4--6.1
and Claude Opus 4.6--5.5 for assistance with manuscript drafting, proofreading of human-written text and mathematical proofs, and development of simulations, figures, and diagrams, but take full responsibility for the content.
JvK is supported by The
Branco Weiss Fellowship—Society in Science.
\bibliographystyle{plainnat}
\bibliography{references}

@article{cattell1966scree,
  author  = {Cattell, Raymond B.},
  title   = {The Scree Test for the Number of Factors},
  journal = {Multivariate Behavioral Research},
  volume  = {1},
  number  = {2},
  pages   = {245--276},
  year    = {1966},
  doi     = {10.1207/s15327906mbr0102_10}
}

@article{shen2024distributional,
  title={Distributional Principal Autoencoders},
  author={Shen, Xinwei and Meinshausen, Nicolai},
  journal={arXiv preprint arXiv:2404.13649},
  year={2024}
}

@article{hinton2006reducing,
  title   = {Reducing the Dimensionality of Data with Neural Networks},
  author  = {Hinton, Geoffrey E. and Salakhutdinov, Ruslan R.},
  journal = {Science},
  volume  = {313},
  number  = {5786},
  pages   = {504--507},
  year    = {2006},
  doi     = {10.1126/science.1127647}
}

@article{freni2025maximum,
  title={Maximum Risk Minimization with Random Forests},
  author={Freni, Francesco and Fries, Anya and K{\"u}hne, Linus and Reichstein, Markus and Peters, Jonas},
  journal={arXiv preprint arXiv:2512.10445},
  year={2025}
}

@inproceedings{agarwal2022minimax,
  title     = {Minimax Regret Optimization for Robust Machine Learning under Distribution Shift},
  author    = {Agarwal, Alekh and Zhang, Tong},
  booktitle = {Proceedings of the 35th Conference on Learning Theory},
  pages     = {2704--2729},
  year      = {2022},
  volume    = {178},
  series    = {Proceedings of Machine Learning Research},
  publisher = {PMLR}
}

@article{robbins1951stochastic,
  title   = {A Stochastic Approximation Method},
  author  = {Robbins, Herbert and Monro, Sutton},
  journal = {The Annals of Mathematical Statistics},
  volume  = {22},
  number  = {3},
  pages   = {400--407},
  year    = {1951},
  doi     = {10.1214/aoms/1177729586}
}

@article{rumelhart1986learning,
  title   = {Learning Representations by Back-Propagating Errors},
  author  = {Rumelhart, David E. and Hinton, Geoffrey E. and Williams, Ronald J.},
  journal = {Nature},
  volume  = {323},
  number  = {6088},
  pages   = {533--536},
  year    = {1986},
  doi     = {10.1038/323533a0}
}

@article{gneiting2007strictly,
  title   = {Strictly Proper Scoring Rules, Prediction, and Estimation},
  author  = {Gneiting, Tilmann and Raftery, Adrian E.},
  journal = {Journal of the American Statistical Association},
  volume  = {102},
  number  = {477},
  pages   = {359--378},
  year    = {2007},
  doi     = {10.1198/016214506000001437}
}

@book{horn2012matrix,
  title     = {Matrix Analysis},
  author    = {Horn, Roger A. and Johnson, Charles R.},
  edition   = {2},
  publisher = {Cambridge University Press},
  year      = {2012},
  doi       = {10.1017/CBO9780511810817}
}

@article{rothenhausler2020anchor,
  title   = {Anchor Regression: Heterogeneous Data Meet Causality},
  author  = {Rothenh{\"a}usler, Dominik and Meinshausen, Nicolai and B{\"u}hlmann, Peter and Peters, Jonas},
  journal = {Journal of the Royal Statistical Society: Series B},
  volume  = {83},
  number  = {2},
  pages   = {215--246},
  year    = {2021},
  doi     = {10.1111/rssb.12398}
}

@article{yu2015DKSinTheta,
  title   = {A Useful Variant of the {Davis--Kahan} Theorem for Statisticians},
  author  = {Yu, Yi and Wang, Tengyao and Samworth, Richard J.},
  journal = {Biometrika},
  volume  = {102},
  number  = {2},
  pages   = {315--323},
  year    = {2015},
  doi     = {10.1093/biomet/asv008}
}

@inproceedings{norman2026unsupervised,
  title     = {Unsupervised Representation Learning -- an Invariant Risk Minimization Perspective},
  author    = {Norman, Yotam and Meir, Ron},
  booktitle = {International Conference on Learning Representations},
  year      = {2026}
}

@article{JMLR:v26:24-0699,
  title   = {Invariant Subspace Decomposition},
  author  = {Lazzaretto, Margherita and Peters, Jonas and Pfister, Niklas},
  journal = {Journal of Machine Learning Research},
  volume  = {26},
  number  = {95},
  pages   = {1--56},
  year    = {2025}
}

@article{pan2011domain,
  title   = {Domain Adaptation via Transfer Component Analysis},
  author  = {Pan, Sinno Jialin and Tsang, Ivor W. and Kwok, James T. and Yang, Qiang},
  journal = {IEEE Transactions on Neural Networks},
  volume  = {22},
  number  = {2},
  pages   = {199--210},
  year    = {2011},
  doi     = {10.1109/TNN.2010.2091281}
}

@InProceedings{pmlr-v162-wang22x,
  title = 	 {Provable Domain Generalization via Invariant-Feature Subspace Recovery},
  author =       {Wang, Haoxiang and Si, Haozhe and Li, Bo and Zhao, Han},
  booktitle = 	 {Proceedings of the 39th International Conference on Machine Learning},
  pages = 	 {23018--23033},
  year = 	 {2022},
  editor = 	 {Chaudhuri, Kamalika and Jegelka, Stefanie and Song, Le and Szepesvari, Csaba and Niu, Gang and Sabato, Sivan},
  volume = 	 {162},
  series = 	 {Proceedings of Machine Learning Research},
  month = 	 {17--23 Jul},
  publisher =    {PMLR}
}

@article{migliavacca2021three,
  title   = {The Three Major Axes of Terrestrial Ecosystem Function},
  author  = {Migliavacca, Mirco and Musavi, Talie and Mahecha, Miguel D. and Nelson, Jacob A. and Knauer, J{\"u}rgen and Baldocchi, Dennis D. and Perez-Priego, Oscar and Christiansen, Rune and Peters, Jonas and others},
  journal = {Nature},
  volume  = {598},
  number  = {7881},
  pages   = {468--472},
  year    = {2021},
  doi     = {10.1038/s41586-021-03939-9}
}

@book{Jolliffe2002,
  title     = {Principal Component Analysis},
  author    = {Jolliffe, Ian T.},
  edition   = {2},
  series    = {Springer Series in Statistics},
  publisher = {Springer},
  address   = {New York},
  year      = {2002},
  doi       = {10.1007/b98835}
}

@article{flury1984common,
  title   = {Common Principal Components in {K} Groups},
  author  = {Flury, Bernhard N.},
  journal = {Journal of the American Statistical Association},
  volume  = {79},
  number  = {388},
  pages   = {892--898},
  year    = {1984},
  doi     = {10.1080/01621459.1984.10477108}
}

@article{fries2026worstcaselowrankapproximations,
  title   = {Worst-Case Low-Rank Approximations},
  author  = {Fries, Anya and Reichstein, Markus and Blei, David and Peters, Jonas},
  journal = {arXiv preprint arXiv:2603.11304},
  year    = {2026}
}

@article{wang2026stablepcadistributionallyrobustlearning,
  title   = {{StablePCA}: Distributionally Robust Learning of Shared Representations from Multi-Source Data},
  author  = {Wang, Zhenyu and Liu, Molei and Lei, Jing and Bach, Francis and Guo, Zijian},
  journal = {arXiv preprint arXiv:2505.00940},
  year    = {2025}
}

@article{schott1999,
  title   = {Partial Common Principal Component Subspaces},
  author  = {Schott, James R.},
  journal = {Biometrika},
  volume  = {86},
  number  = {4},
  pages   = {899--908},
  year    = {1999},
  doi     = {10.1093/biomet/86.4.899}
}

@article{beran1985bootstrap,
  title   = {Bootstrap Tests and Confidence Regions for Functions of a Covariance Matrix},
  author  = {Beran, Rudolf and Srivastava, Muni S.},
  journal = {The Annals of Statistics},
  volume  = {13},
  number  = {1},
  pages   = {95--115},
  year    = {1985},
  doi     = {10.1214/aos/1176346579}
}

@book{athreya_lahiri_2006,
  author    = {Athreya, Krishna B. and Lahiri, Soumendra N.},
  title     = {Measure Theory and Probability Theory},
  series    = {Springer Texts in Statistics},
  publisher = {Springer},
  address   = {New York},
  year      = {2006},
  doi       = {10.1007/978-0-387-35434-7}
}

@misc{gas_sensor_array_drift,
  title        = {Gas Sensor Array Drift at Different Concentrations},
  author       = {Vergara, Alexander},
  year         = {2012},
  howpublished = {UCI Machine Learning Repository},
  note         = {Dataset},
  doi          = {10.24432/C5MK6M}
}

@article{fonollosa2015chemicalgas,
  author  = {Fonollosa, Jordi and Rodr{\'i}guez-Luj{\'a}n, Irene and Huerta, Ram{\'o}n},
  title   = {Chemical gas sensor array dataset},
  journal = {Data in Brief},
  volume  = {3},
  pages   = {85--89},
  year    = {2015},
  doi     = {10.1016/j.dib.2015.01.003}
}

@article{vergara2012chemical,
  author  = {Vergara, Alexander and Vembu, Shankar and Ayhan, Tuba and Ryan, Margie A. and Homer, Margie L. and Huerta, Ram{\'o}n},
  title   = {Chemical gas sensor drift compensation using classifier ensembles},
  journal = {Sensors and Actuators B: Chemical},
  volume  = {166--167},
  pages   = {320--329},
  year    = {2012},
  doi     = {10.1016/j.snb.2012.01.074}
}

@inproceedings{samadi2018price,
  title     = {The Price of Fair {PCA}: One Extra Dimension},
  author    = {Samadi, Samira and Tantipongpipat, Uthaipon and Morgenstern, Jamie and Singh, Mohit and Vempala, Santosh},
  booktitle = {Advances in Neural Information Processing Systems},
  volume    = {31},
  year      = {2018}
}

@inproceedings{tantipongpipat2019multi,
  title     = {Multi-Criteria Dimensionality Reduction with Applications to Fairness},
  author    = {Tantipongpipat, Uthaipon and Samadi, Samira and Singh, Mohit and Morgenstern, Jamie and Vempala, Santosh},
  booktitle = {Advances in Neural Information Processing Systems},
  volume    = {32},
  year      = {2019}
}

@inproceedings{muandet2013domain,
  title     = {Domain Generalization via Invariant Feature Representation},
  author    = {Muandet, Krikamol and Balduzzi, David and Sch{\"o}lkopf, Bernhard},
  booktitle = {Proceedings of the 30th International Conference on Machine Learning},
  pages     = {10--18},
  year      = {2013},
  editor    = {Dasgupta, Sanjoy and McAllester, David},
  volume    = {28},
  series    = {Proceedings of Machine Learning Research},
  publisher = {PMLR}
}

@article{flury1986algorithm,
  title   = {An Algorithm for Simultaneous Orthogonal Transformation of Several Positive Definite Symmetric Matrices to Nearly Diagonal Form},
  author  = {Flury, Bernhard N. and Gautschi, Walter},
  journal = {SIAM Journal on Scientific and Statistical Computing},
  volume  = {7},
  number  = {1},
  pages   = {169--184},
  year    = {1986},
  doi     = {10.1137/0907013}
}

@article{wang2021semiparametric,
  title   = {Semiparametric Partial Common Principal Component Analysis for Covariance Matrices},
  author  = {Wang, Bingkai and Luo, Xi and Zhao, Yi and Caffo, Brian},
  journal = {Biometrics},
  volume  = {77},
  number  = {4},
  pages   = {1175--1186},
  year    = {2021},
  doi     = {10.1111/biom.13369}
}

@article{lock2013joint,
  title   = {Joint and Individual Variation Explained ({JIVE}) for Integrated Analysis of Multiple Data Types},
  author  = {Lock, Eric F. and Hoadley, Katherine A. and Marron, J. S. and Nobel, Andrew B.},
  journal = {The Annals of Applied Statistics},
  volume  = {7},
  number  = {1},
  pages   = {523--542},
  year    = {2013},
  doi     = {10.1214/12-AOAS597}
}

@article{feng2018angle,
  title   = {Angle-Based Joint and Individual Variation Explained},
  author  = {Feng, Qing and Jiang, Meilei and Hannig, Jan and Marron, J. S.},
  journal = {Journal of Multivariate Analysis},
  volume  = {166},
  pages   = {241--265},
  year    = {2018},
  doi     = {10.1016/j.jmva.2018.03.008}
}

@article{abid2018exploring,
  title   = {Exploring Patterns Enriched in a Dataset with Contrastive Principal Component Analysis},
  author  = {Abid, Abubakar and Zhang, Martin J. and Bagaria, Vivek K. and Zou, James},
  journal = {Nature Communications},
  volume  = {9},
  number  = {1},
  pages   = {2134},
  year    = {2018},
  doi     = {10.1038/s41467-018-04608-8}
}

@article{haavelmo1943statistical,
  title   = {The Statistical Implications of a System of Simultaneous Equations},
  author  = {Haavelmo, Trygve},
  journal = {Econometrica},
  volume  = {11},
  number  = {1},
  pages   = {1--12},
  year    = {1943},
  doi     = {10.2307/1905714}
}

@book{pearl2009causality,
  title     = {Causality: Models, Reasoning, and Inference},
  author    = {Pearl, Judea},
  edition   = {2},
  publisher = {Cambridge University Press},
  year      = {2009},
  doi       = {10.1017/CBO9780511803161}
}

@inproceedings{scholkopf2012causal,
  title     = {On Causal and Anticausal Learning},
  author    = {Sch{\"o}lkopf, Bernhard and Janzing, Dominik and Peters, Jonas and Sgouritsa, Eleni and Zhang, Kun and Mooij, Joris M.},
  booktitle = {Proceedings of the 29th International Conference on Machine Learning},
  pages     = {1255--1262},
  year      = {2012},
  publisher = {Omnipress}
}

@article{peters2016causal,
  title   = {Causal Inference by Using Invariant Prediction: Identification and Confidence Intervals},
  author  = {Peters, Jonas and B{\"u}hlmann, Peter and Meinshausen, Nicolai},
  journal = {Journal of the Royal Statistical Society: Series B},
  volume  = {78},
  number  = {5},
  pages   = {947--1012},
  year    = {2016},
  doi     = {10.1111/rssb.12167}
}

@article{rojascarulla2018invariant,
  title   = {Invariant Models for Causal Transfer Learning},
  author  = {Rojas-Carulla, Mateo and Sch{\"o}lkopf, Bernhard and Turner, Richard and Peters, Jonas},
  journal = {Journal of Machine Learning Research},
  volume  = {19},
  number  = {36},
  pages   = {1--34},
  year    = {2018}
}

@article{arjovsky2019invariant,
  title   = {Invariant Risk Minimization},
  author  = {Arjovsky, Martin and Bottou, L{\'e}on and Gulrajani, Ishaan and Lopez-Paz, David},
  journal = {arXiv preprint arXiv:1907.02893},
  year    = {2019}
}

@misc{french49industry,
  author = {Fama, Eugene F. and French, Kenneth R.},
  title = {{49 Industry Portfolios [Daily] [Data set]}},
  year = {2026},
  howpublished = {Kenneth R. French Data Library, Tuck School of Business, Dartmouth College},
  note = {Accessed July 26, 2026}
}

@inproceedings{lu2022invariant,
  title     = {Invariant Causal Representation Learning for Out-of-Distribution Generalization},
  author    = {Lu, Chaochao and Wu, Yuhuai and Hern{\'a}ndez-Lobato, Jos{\'e} Miguel and Sch{\"o}lkopf, Bernhard},
  booktitle = {International Conference on Learning Representations},
  year      = {2022}
}

@inproceedings{krueger2021out,
  title     = {Out-of-Distribution Generalization via Risk Extrapolation ({REx})},
  author    = {Krueger, David and Caballero, Ethan and Jacobsen, Joern-Henrik and Zhang, Amy and Binas, Jonathan and Zhang, Dinghuai and Le Priol, Remi and Courville, Aaron},
  booktitle = {Proceedings of the 38th International Conference on Machine Learning},
  pages     = {5815--5826},
  year      = {2021},
  volume    = {139},
  series    = {Proceedings of Machine Learning Research},
  publisher = {PMLR}
}

@inproceedings{sagawa2020distributionally,
  title     = {Distributionally Robust Neural Networks for Group Shifts: On the Importance of Regularization for Worst-Case Generalization},
  author    = {Sagawa, Shiori and Koh, Pang Wei and Hashimoto, Tatsunori B. and Liang, Percy},
  booktitle = {International Conference on Learning Representations},
  year      = {2020}
}

@article{christiansen2022causal,
  title   = {A Causal Framework for Distribution Generalization},
  author  = {Christiansen, Rune and Pfister, Niklas and Jakobsen, Martin Emil and Gnecco, Nicola and Peters, Jonas},
  journal = {IEEE Transactions on Pattern Analysis and Machine Intelligence},
  volume  = {44},
  number  = {10},
  pages   = {6614--6630},
  year    = {2022},
  doi     = {10.1109/TPAMI.2021.3094760}
}

@article{novembre2008genes,
  title   = {Genes Mirror Geography within {Europe}},
  author  = {Novembre, John and Johnson, Toby and Bryc, Katarzyna and Kutalik, Zolt{\'a}n and Boyko, Adam R. and Auton, Adam and Indap, Amit and King, Karen S. and Bergmann, Sven and Nelson, Matthew R. and Stephens, Matthew and Bustamante, Carlos D.},
  journal = {Nature},
  volume  = {456},
  number  = {7218},
  pages   = {98--101},
  year    = {2008},
  doi     = {10.1038/nature07331}
}

@article{wright2004worldwide,
  title   = {The Worldwide Leaf Economics Spectrum},
  author  = {Wright, Ian J. and Reich, Peter B. and Westoby, Mark and Ackerly, David D. and Baruch, Zdravko and others},
  journal = {Nature},
  volume  = {428},
  number  = {6985},
  pages   = {821--827},
  year    = {2004},
  doi     = {10.1038/nature02403}
}

@article{turk1991eigenfaces,
  title   = {Eigenfaces for Recognition},
  author  = {Turk, Matthew and Pentland, Alex},
  journal = {Journal of Cognitive Neuroscience},
  volume  = {3},
  number  = {1},
  pages   = {71--86},
  year    = {1991},
  doi     = {10.1162/jocn.1991.3.1.71}
}

@book{bhatia1997matrix,
  title     = {Matrix Analysis},
  author    = {Bhatia, Rajendra},
  series    = {Graduate Texts in Mathematics},
  volume    = {169},
  publisher = {Springer},
  address   = {New York},
  year      = {1997},
  doi       = {10.1007/978-1-4612-0653-8}
}

@book{boyd2004convex,
  author    = {Boyd, Stephen and Vandenberghe, Lieven},
  title     = {Convex Optimization},
  publisher = {Cambridge University Press},
  year      = {2004},
  doi       = {10.1017/CBO9780511804441}
}

@article{drineas2019lowrank,
  title   = {Low-Rank Matrix Approximations Do Not Need a Singular Value Gap},
  author  = {Drineas, Petros and Ipsen, Ilse C. F.},
  journal = {SIAM Journal on Matrix Analysis and Applications},
  volume  = {40},
  number  = {1},
  pages   = {299--319},
  year    = {2019},
  doi     = {10.1137/18M1163658}
}

@article{wang2026multicontext,
  title   = {Multi-context principal component analysis},
  author  = {Wang, Kexin and Bhate, Salil and Pereira, Jo{\~a}o M. and Kileel, Joe and Figlerowicz, Matylda and Seigal, Anna},
  journal = {arXiv preprint arXiv:2601.15239},
  year    = {2026},
}

@article{hyvarinen2026goal,
  title   = {What is the goal of unsupervised machine learning?},
  author  = {Hyv{\"a}rinen, Aapo},
  journal = {arXiv preprint arXiv:2610.11697},
  year    = {2026}
}

\newpage
\appendix

\startcontents[appendices]

\section*{Contents of the Appendix}
{
\hypersetup{linkcolor=black}
\printcontents[appendices] %
  {l}                      %
  {1}                      %
  {\setcounter{tocdepth}{2}} %
}
\vspace{10mm}

\section{Further related work} \label{subsec:related_work}

\AnchorPCAwout{} is a dimensionality reduction over several source domains that builds on ideas of invariance. We summarize work in multi-domain dimension reduction and related ideas in prediction that use invariance.

\subsection{Multi-domain dimension reduction}
Multi-domain dimension reduction can be separated into three groups: worst-case methods, invariance-based methods, and other methods (e.g., those which do not relate to generalization, but consider multiple groups and PCA).

\paragraph{Worst-case methods.} Worst-case methods replace pooled explained variance or reconstruction by a max-min (or min-max) objective, which seeks to optimize performance on the worst domain, group, or criterion. 
FairPCA \citep{samadi2018price} minimizes the maximum increase in reconstruction error incurred by using a shared projection instead of the domain-specific optimal projection.
Multi-criteria dimensionality reduction \citep{tantipongpipat2019multi} generalizes this idea to several criteria over multiple groups. 
However, neither provides a robustness perspective.
In contrast, wcPCA~\citep{fries2026worstcaselowrankapproximations}
includes out-of-sample guarantees for target covariances in the convex hull of the observed source covariances. 
Along similar lines, StablePCA~\citep{wang2026stablepcadistributionallyrobustlearning} 
casts multi-source PCA as a group-DRO problem.
\AnchorPCA{} is also worst-case optimal, but over a different uncertainty set of perturbed covariances (\cref{thm:minimax_reconstruction}). Importantly, it does not only optimize a worst-case objective: it also recovers the maximal invariant subspace $\Sstar$ (\cref{thm:invariant_subspace_limit}), a target none of the methods above consider.

\paragraph{Invariance-based methods.} 
Invariance-based dimension reduction methods seek representations whose distributions, moments, or predictive relations are stable across domains, often with the goal of improving downstream prediction for a target \(Y\).
Transfer Component Analysis (TCA)~\citep{pan2011domain} aligns source and target distributions through maximum mean discrepancy (MMD) criteria while preserving variance; it requires target domain information while Anchor PCA does not.
Domain-Invariant Component Analysis (DICA)~\citep{muandet2013domain} learns transformations that reduce dissimilarity across training domains while preserving predictive structure; it requires a downstream task and response $Y$, while Anchor PCA is fully unsupervised.
In contrast, Unsupervised DICA (UDICA) \citep{muandet2013domain} and Principal Invariant Component Analysis (PICA) \citep{norman2026unsupervised} do not require information about a response $Y$ or the target domain. UDICA retains variance
in kernel features while 
penalizing the distances between the mean embeddings of the source domains.
PICA
finds linear components whose projections have the same distribution across domains. 
For a single, unit norm component $v$, with centered Gaussian data, PICA's objective
is to maximize pooled explained variance subject to $v^\top\Sigma_e v=v^\top\Sigma_f v$ for all source domains $e,f$.
Both PICA and UDICA aim for distributional invariance; we instead consider a geometric notion of invariance, 
namely agreement between
the domain-specific top-$k$ principal projectors.
A direction can therefore be invariant even if its variance, or its rank within the top $k$, differs across domains.
Additionally, unlike PICA, which imposes exact invariance constraints, we always return a non-empty subspace of the chosen dimension $k$, even when $\Sstar = \{0\}$.
While DICA provides generalization guarantees for downstream prediction,
TCA, DICA, UDICA, and PICA
do not
provide unsupervised PCA-style robustness guarantees for explained variance or reconstruction of the kind studied here.

\paragraph{Other methods.}
Other multi-group dimensionality reduction methods model shared structure for statistical or exploratory purposes rather than for OOD robustness.
Common principal component analysis (CPCA) assumes simultaneous diagonalizability of the covariance matrices \citep{flury1984common,flury1986algorithm} (that is, there exists an orthogonal $U$ such that for all domains $e$, $\Sigma_e = UD_e U^\top$ with $D_e$ diagonal and domain-specific); partial common principal component methods relax this assumption and estimate or test for shared eigendirections or shared principal subspaces \citep{schott1999,wang2021semiparametric}.
Our \texttt{FindS}$_\star$ procedure builds on this testing 
procedure
but the population target is different: $\Sstar$ is the intersection of local top-$k$ eigenspaces, not a set of eigenvectors that are shared as individual axes or ordered in the same way across domains.
Multi-context PCA (MCPCA) \citep{wang2026multicontext} also seeks shared directions of variation and extends the CPCA model $\Sigma_e=UD_eU^\top$ by allowing nonorthogonal columns of $U$ that may appear in only some domains.
\citet{wang2026multicontext} establish identifiability up to sign and ordering when the columns are linearly independent and no two have proportional weights across domains (that is, $((D_e)_{jj})_{e\in\mathcal E}\not\propto((D_e)_{kk})_{e\in\mathcal E}$ for all $j\neq k$).
Anchor PCA instead learns a single subspace by balancing explained
variance and agreement with the local top-$k$ subspaces, without
assuming this covariance factorization.
Multi-block methods such as JIVE~\citep{lock2013joint} and AJIVE~\citep{feng2018angle} decompose data sources into joint and individual low-rank variation, whereas contrastive PCA~\citep{abid2018exploring} searches for directions enriched in one dataset relative to a background dataset.
These methods provide useful decompositions, but they do not couple shared-structure estimation to the projector-agreement and structured-perturbation guarantees studied here.
Nonlinear PCA and autoencoder variants provide another axis of generalization \citep{shen2024distributional}; \cref{app:nonlinear_extensions} discusses how the same agreement principle could be ported to autoencoder-based representations.

\subsection{Invariance 
for prediction}
In supervised domain generalization, invariance is usually formulated as stability of a conditional law, a prediction rule, or a risk across domains. This principle is rooted in causal invariance \citep{haavelmo1943statistical,pearl2009causality,scholkopf2012causal,peters2016causal} and appears in invariant models for causal transfer learning, invariant risk minimization, invariant-feature subspace recovery, invariant causal representation learning, and invariant subspace decomposition for time-varying prediction problems \citep{rojascarulla2018invariant,arjovsky2019invariant,pmlr-v162-wang22x,lu2022invariant,JMLR:v26:24-0699}.
Several prediction methods relax strict invariance rather than enforcing it exactly: 
for example, anchor regression interpolates between ordinary least squares and an invariant solution through a penalty parameter~\citep{rothenhausler2020anchor,christiansen2022causal}.
\AnchorPCA{} imports this soft-invariance logic into unsupervised dimension reduction: finite $\lambda$ penalizes projector disagreement, while \AnchorPCAinfty{} recovers the hard agreement target and, when it exists, the maximal invariant subspace $\Sstar$.

\section{Further details on Anchor PCA} \label{app:further_details}

\subsection{Ordered bases for direction-level statements}
\label{app:ordered_representatives}

The optimization problems defining \AnchorPCA{} (\cref{def:anchorpca}) and \AnchorPCAinfty{} (\cref{def:anchorpca_infty}) depend on
$W$ only through the projection matrix $\Pi_W=WW^\top$. Hence, the canonical
population objects are solution subspaces or solution projectors. 
However, the projection-based formulation 
allows us to fix an ordered representation, see \Cref{sec:projection_based_formulation}. We now state their precise construction.

\subsubsection{\texorpdfstring{\AnchorPCA{}}{AnchorPCA}} 

Solutions of \AnchorPCA{} can be obtained by solving rank-$k$ PCA on $M_\lambda=\barSigma+2E\lambda\,\barPi$ (\cref{sec:projection_based_formulation}). Hence, we use the standard PCA convention and choose
$W_k^{\anchor,\lambda}=(w_1^{\anchor,\lambda},\dots,w_k^{\anchor,\lambda})$ as
an ordered orthonormal eigenbasis of $M_\lambda$ 
associated with its $k$ largest eigenvalues. 
Thus,
$w_1^{\anchor,\lambda}$ maximizes $w^\top M_\lambda w$ over unit vectors,
$w_2^{\anchor,\lambda}$ maximizes the same quantity subject to being orthogonal
to $w_1^{\anchor,\lambda}$, and so on. 
If $M_\lambda$ has eigenvalue ties, the ordered basis is not unique and we
choose any orthonormal basis inside tied eigenspaces. If all ties occur within the selected top-$k$ spectrum, this affects only the solution
basis, not the rank-$k$ solution subspace and projector. If there is no eigengap at rank-$k$, i.e. $\lambda_k(M_\lambda)=\lambda_{k+1}(M_\lambda)$, then the
rank-$k$ solution projector is non-unique, and any maximizing rank-$k$
solution projector may be used.

\subsubsection{\texorpdfstring{\AnchorPCAinfty{}}{AnchorPCA-infty}} \label{app:anchor_pca_infty_basis_convention}

We now define an analogous ordered basis for \AnchorPCAinfty{} in a two-stage construction.
First, we determine the solution subspace: directions are ordered by
agreement with the domain-specific top-$k$ principal subspaces
(equivalently, by decreasing eigenvalue of $\barPi$). 
Second, within each eigenspace, we choose a basis by iteratively maximizing the pooled explained variance over directions in that block. 

\smallskip
\noindent
\textbf{Step 1a: Ordering by agreement.}
The agreement
constraint of \AnchorPCAinfty{} in~\eqref{eq:anchorpca_infty_def} is equivalent
to maximizing 
\begin{equation}\label{eq:trace_constraint}
    \Tr(W^\top\barPi W) = \sum_{i=1}^k w_i^\top \barPi w_i
\end{equation}
over $W = [w_1, \ldots, w_k]\in\Ok$ (\cref{sec:projection_based_formulation}). We now characterize this solution.
Let
$
\rho_1>\rho_2>\cdots>\rho_L
$
be the distinct eigenvalues of $\barPi$, and let
\[
{\operatorname{Eig}}_\ell:=\ker(\barPi-\rho_\ell I_p),
\qquad
d_\ell:=\dim({\operatorname{Eig}}_\ell).
\]
${\operatorname{Eig}}_\ell$ is the eigenspace of $\barPi$ corresponding to the eigenvalue $\rho_\ell$ and $d_\ell$ is its dimension. 
Larger $\rho_\ell$ implies larger average agreement with
the domain-specific top-$k$ principal subspaces.
Then, let $r$ be the unique index satisfying
\[
\sum_{\ell<r} d_\ell < k
\le
\sum_{\ell\le r} d_\ell.
\]
By Ky Fan's maximum principle \citep[Problem~I.6.15]{bhatia1997matrix}, the trace
in~\eqref{eq:trace_constraint} is maximized precisely when the columns of $W$
are drawn from the leading eigenspaces of $\barPi$: every maximizer spans a
$k$-dimensional subspace of the form
\[
\operatorname{Eig}_1
\oplus\cdots\oplus
\operatorname{Eig}_{r-1}
\oplus~
\mathcal U_r,
\]
where $\mathcal U_r\le \operatorname{Eig}_r$ 
is an arbitrary
subspace of dimension $q:=k-\sum_{\ell<r}d_\ell$. 

\smallskip
\noindent
\textbf{Step 1b: Determining $\mathcal U_r$.}
The eigenspaces
$\operatorname{Eig}_1,\ldots,\operatorname{Eig}_{r-1}$
are selected completely, 
and
only the final eigenspace $\mathcal U_r$ can be determined via the secondary $\barSigma$-criterion of~\eqref{def:anchorpca_infty}, which chooses $\mathcal U_r$ to maximize
pooled explained variance
among all $q$-dimensional subspaces of ${\operatorname{Eig}}_r$. Equivalently, inside 
${\operatorname{Eig}}_r$, the subspace
$\mathcal U_r$ is chosen by the top-$q$ eigenspace of
$\Pi_{\operatorname{Eig}_r}\barSigma\Pi_{\operatorname{Eig}_r}$. 

\textbf{Step 2: Choosing an ordered basis.}
Within each fully selected eigenspace
${\operatorname{Eig}}_1,\dots,{\operatorname{Eig}}_{r-1}$, we choose an orthonormal eigenbasis ordered by 
non-increasing
eigenvalue 
of $\barSigma$, and concatenate these bases in decreasing order of
$\rho_\ell$. We then append an ordered top-$q$ eigenbasis of $\barSigma$ within
${\operatorname{Eig}}_r$.
With this convention, \AnchorPCAinfty{} directions are ordered first by
invariance, that is, by decreasing $\barPi$-eigenvalue. Inside a fixed
eigenspace, they are ordered by pooled explained variance. 

\begin{remark}[Handling non-uniqueness.]
If $\barSigma$ has eigenvalue ties in one of the fully selected 
eigenspaces
$\operatorname{Eig}_1, \ldots,
\operatorname{Eig}_{r-1}$, the basis within the tied eigenspace is not unique; we make an arbitrary choice, and this affects only direction-level quantities, not the solution subspace or projector. 
If instead $\Pi_{\operatorname{Eig}_r} \barSigma \Pi_{\operatorname{Eig}_r}$
has ties among its top $q$
eigenvalues, the subspace $\mathcal U_r$
itself is not unique, and the rank-$k$ \AnchorPCAinfty{} solution subspace and projector are likewise non-unique; the uniqueness statements elsewhere in the paper (e.g., \cref{thm:consistency_empirical_methods}(ii)) exclude this case by assumption.
\end{remark}

\subsection{Additional theoretical results} \label{app:details}

\paragraph{Notation.}
Throughout the appendix, all matrices are represented with respect to the
standard basis of \(\R^p\), unless otherwise stated. For pairwise orthogonal
subspaces \(\mathcal S_1,\dots,\mathcal S_j\subseteq\R^p\), we write
\(\mathcal S_1\oplus\cdots\oplus\mathcal S_j\) for their orthogonal direct sum.
Suppose that a symmetric matrix \(M\in\R^{p\times p}\) preserves each
\(\mathcal S_i\), that is, \(Mv\in\mathcal S_i\) $\forall i \in \{1,\ldots,j\}$ and $\forall v\in \mathcal S_i$.
Then we can block-diagonalize $M$ with respect to \((\mathcal S_i)_{i=1}^j\) as follows.
Let \(B_i\in\R^{p\times d_i}\) be an orthonormal basis matrix for
\(\mathcal S_i\), where \(d_i:=\dim(\mathcal S_i)\), and let
\(Q:=[B_1\ \cdots\ B_j]\). In this aligned basis, $M$ is block diagonal, that is
\[
Q^\top M Q
=
\operatorname{diag}(A_1,\ldots,A_j)
=
\begin{bmatrix}
A_1 &        & 0 \\
    & \ddots &   \\
0   &        & A_j
\end{bmatrix},
\qquad
A_i := B_i^\top M B_i .
\]
We abbreviate this block-diagonal representation by writing
\(M=A_1\oplus\cdots\oplus A_j\) with respect to the decomposition
\(\mathcal S_1\oplus\cdots\oplus\mathcal S_j\).

\subsubsection{\texorpdfstring{\texttt{AnchorPCA}$_{\infty}$ as a limit of \texttt{AnchorPCA}$_{\lambda}$}{AnchorPCA infinity as a limit of AnchorPCA lambda}} \label{sec:relation_methods}

Intuitively, as the penalty $\lambda$ grows, solutions of \AnchorPCA{} approach the
hard-agreement solutions of \AnchorPCAinfty{}. The following proposition states this
at the level of projection matrices. 
\cref{thm:invariant_subspace_limit_extended} gives a complementary result at the level of individual directions, showing that the leading $m=\operatorname{dim}(\Sstar)$ $\AnchorPCA{}$ directions converge to $\Sstar$.

\begin{proposition}[Large-penalty limit]
\label{prop:asymptotic_recovery}
For all $\lambda \ge 0$, let $\Pianchor{k}$ be 
any rank-$k$ \AnchorPCA{} solution projector.
If $\lambda_n\to\infty$, every accumulation point of $(\Pi^{\anchor,\lambda_n}_{k})$ is a rank-$k$ \AnchorPCAinfty{} solution projector.

If, in addition, the population rank-$k$ \AnchorPCAinfty{} solution projector is
unique, denote it by $\Pi^{\anchor,\infty}_{k}$, then
\[
\opnorm{\Pianchor{k}-\Pi^{\anchor,\infty}_{k}} \to 0
\qquad
\text{as }\lambda\to\infty.
\]
\end{proposition}

\begin{proof}
Let
\[
\mathcal P_k
:=
\{X\in \mathbb R^{p \times p}: X = WW^\top,\ W \in \Ok\}
\]
denote the set of rank-$k$ orthogonal projectors. Since $\mathcal P_k$ is closed and bounded in $\mathbb R^{p\times p}$, it is compact.

For $X\in\mathcal P_k$, define
\[
H(X):=\Tr(\barSigma X),
\qquad
G(X):=\Tr(\barPi X),
\qquad
F_\lambda(X):=H(X)+2E\lambda\, G(X).
\]
Since both \AnchorPCAwout{}
objectives depend on $W$ only through $WW^\top$, rank-$k$ \AnchorPCA{} is the problem of maximizing $F_\lambda$ over $\mathcal P_k$, while rank-$k$ \AnchorPCAinfty{} is the problem of first maximizing $G$ over $\mathcal P_k$ and then maximizing $H$ among the maximizers of $G$.

Fix a sequence
$\lambda_n\to\infty$, and for each $n$, let $\Pi^{\anchor,\lambda_n}_{k}\in\mathcal P_k$ be a rank-$k$ \AnchorPCA{} solution projector at penalty $\lambda_n$. 
By compactness of $\mathcal P_k$, the sequence $(\Pi^{\anchor,\lambda_n}_{k})$ has accumulation points in $\mathcal P_k$. Let $\Pi_\star$ be any such accumulation point;
passing to a subsequence (and relabeling), $\opnorm{ \Pi^{\anchor,\lambda_n}_{k} - \Pi_\star} \to 0$.

\smallskip
\noindent
\textbf{Step 1: $\Pi_\star$ maximizes $G$ over $\mathcal P_k$.}
By contradiction, assume that $\Pi_\star$ does not maximize $G$ and that there exist $Y\in\mathcal P_k$ and $\delta>0$ such that
$
G(Y)\ge G(\Pi_\star)+2\delta.
$
By continuity of $G$, 
\[
G(\Pi^{\anchor,\lambda_n}_{k})\le G(\Pi_\star)+\delta
\]
for all sufficiently large $n$.
Further, for all $X\in\mathcal P_k$,
$
|H(X)|=|\Tr(\barSigma X)|\le k\opnorm{\barSigma}
$
(because $X$ is a rank-$k$ orthogonal projector), and hence
$
H(Y)-H(\Pi^{\anchor,\lambda_n}_{k})\ge -2k\opnorm{\barSigma}.
$
Therefore, for all sufficiently large $n$,
\begin{align*}
    F_{\lambda_n}(Y)-F_{\lambda_n}(\Pi^{\anchor,\lambda_n}_{k})
    &= \bigl(H(Y)-H(\Pi^{\anchor,\lambda_n}_{k})\bigr)+2E\lambda_n\bigl(G(Y)-G(\Pi^{\anchor,\lambda_n}_{k})\bigr) \\
    &\ge -2k\opnorm{\barSigma}+2E\lambda_n\delta.
\end{align*}
where $G(Y)-G(\Pi^{\anchor,\lambda_n}_{k}) \ge G(\Pi_\star)+2\delta - (G(\Pi_\star)+\delta) = 2\delta-\delta=\delta$.
Since $\lambda_n\to\infty$, the right-hand side eventually is positive, contradicting the optimality of $\Pi^{\anchor,\lambda_n}_{k}$. Thus, $\Pi_\star$ maximizes $G$ over $\mathcal P_k$.

\smallskip
\noindent
\textbf{Step 2: $\Pi_\star$ maximizes $H$ among the maximizers of $G$.}
By contradiction, assume that there exists $Y\in\mathcal P_k$ and $\delta>0$ such that $Y$ maximizes $G$ and
$
H(Y)\ge H(\Pi_\star)+2\delta.
$
Since $Y$ maximizes $G$ over $\mathcal P_k$, $G(Y)\ge G(\Pi^{\anchor,\lambda_n}_{k})$ for all $n$. Further,
by continuity of $H$,
\[
H(\Pi^{\anchor,\lambda_n}_{k})\le H(\Pi_\star)+\delta
\]
for all sufficiently large $n$. 
Hence, for all sufficiently large $n$,
\[
F_{\lambda_n}(Y)-F_{\lambda_n}(\Pi^{\anchor,\lambda_n}_{k})
=\bigl(H(Y)-H(\Pi^{\anchor,\lambda_n}_{k})\bigr)+2E\lambda_n\bigl(G(Y)-G(\Pi^{\anchor,\lambda_n}_{k})\bigr)
\ge \delta,
\]
which contradicts the optimality of $\Pi^{\anchor,\lambda_n}_{k}$. Therefore, $\Pi_\star$ is a rank-$k$ \AnchorPCAinfty{} solution projector.

\smallskip
\noindent
\textbf{Step 3: full convergence under uniqueness.}
Assume now that the rank-$k$ \AnchorPCAinfty{} solution projector is unique and equals $\Pi^{\anchor,\infty}_{k}$. By contradiction, assume that there exist $\varepsilon>0$ and a sequence $\lambda_n\to\infty$ such that
\[
\opnorm{\Pi^{\anchor,\lambda_n}_{k}-\Pi^{\anchor,\infty}_{k}}\ge \varepsilon
\qquad
\text{for all }n.
\]
By compactness of $\mathcal P_k$, the sequence $(\Pi^{\anchor,\lambda_n}_{k})$ has an accumulation point $\Pi_\star$. By the first part of the proposition, $\Pi_\star$ is a rank-$k$ \AnchorPCAinfty{} solution projector, hence $\Pi_\star=\Pi^{\anchor,\infty}_{k}$, a contradiction. Therefore
\[
\opnorm{\Pianchor{k} -\Pi^{\anchor,\infty}_{k}} \to 0.
\]
\end{proof}

\subsubsection{Recovery of \texorpdfstring{$\Sstar$}{S*} by \texorpdfstring{$\AnchorPCA{}$}{AnchorPCA} with finite penalty \texorpdfstring{$\lambda$}{}}
\label{app:exact_finite_recovery}

The following proposition shows that \AnchorPCA{} recovers the invariant subspace at two levels:
if $\lambda>\lambda^\star$, every rank-$k$ solution subspace contains $\Sstar$, and if
$\lambda>\lambda^{\star\star}$, 
the leading rank-$m$ \AnchorPCA{} directions (in the convention of \cref{app:ordered_representatives}) span~$\Sstar$.

\begin{proposition}[Finite-penalty recovery]
\label{prop:exact_recovery}
Assume $1\le m\le k<p$ and that $\barSigma$ is block-diagonal with respect to
$\Sstar\oplus\SstarPerp$. Let
$\barSigma_{\Sstar}:=\Pi_{\Sstar}\barSigma\Pi_{\Sstar}$ and
$\barSigma_{\SstarPerp}:=\Pi_{\SstarPerp}\barSigma\Pi_{\SstarPerp}$
be the diagonal blocks of $\barSigma$ with respect to $\Sstar\oplus\SstarPerp$ 
and let
$\barR:=\Pi_{\SstarPerp}\barPi\Pi_{\SstarPerp}$.
If~ $\Pianchor{k}$ is a rank-$k$ \AnchorPCA{} solution projector, then the following statements hold.
\begin{enumerate}[leftmargin=*]
\item[(i)] If
$\lambda>
\lambda^\star:=
{\bigl(\lambda_1(\barSigma_{\SstarPerp})-\lambda_m(\barSigma_{\Sstar})\bigr)}/
{2E\bigl(1-\lambda_{k-m+1}(\barR)\bigr)}$,
then
$\Pianchor{k}\Pi_{\Sstar}=\Pi_{\Sstar}$, equivalently
$\Sstar\subseteq\Sanchor{k}$.

\item[(ii)] If
$\lambda>
\lambda^{\star\star}:=
{\bigl(\lambda_1(\barSigma_{\SstarPerp})-\lambda_m(\barSigma_{\Sstar})\bigr)}/
{2E\bigl(1-\lambda_1(\barR)\bigr)}$,
then the top $m$-dimensional eigenspace of
$M_\lambda=\barSigma+2E\lambda\barPi$ is $\Sstar$
.
\end{enumerate}
\end{proposition}

\begin{proof}
Since rank-$k$ \AnchorPCA{} is rank-$k$ PCA on $M_\lambda=\barSigma+2E\lambda\,\barPi$, the proof proceeds by block-diagonalizing $M_\lambda$ along $\Sstar\oplus\SstarPerp$ and then showing that, once $\lambda$ is large enough, the $\Sstar$-block eigenvalues dominate the relevant $\SstarPerp$-block eigenvalues.

\smallskip
\noindent\textbf{Step 1: block-diagonalize $M_\lambda$ along $\Sstar\oplus\SstarPerp$.}
Let $S\in\mathcal O_{p\times m}$ and $Z\in \mathcal O_{p\times (p-m)}$ have orthonormal columns respectively spanning $\Sstar$ and $\SstarPerp$.
Let
\[
A:=S^\top\barSigma S,\qquad
B:=Z^\top\barSigma Z,\qquad
R:=Z^\top\barPi Z
\]
Then, in the orthonormal coordinates $(S,Z)$,
\[
(S,Z)^\top M_\lambda(S,Z)
=
\begin{pmatrix}
A+2E\lambda I_m & 0\\
0 & B+2E\lambda R
\end{pmatrix}
\]
because $S^\top\barSigma Z=0$ (because, by assumption, $\barSigma$ is block-diagonal with respect to
$\Sstar\oplus\SstarPerp$) and 
$S^\top\barPi S=I_m$ and $S^\top\barPi Z=0$ (as $\Sstar$ is the eigenspace of $\barPi$ associated with eigenvalue $1$). 

\smallskip
\noindent\textbf{Step 2: compare the top eigenvalues across the two blocks.} First, observe that
\[
\barSigma_{\Sstar}=SAS^\top,\qquad
\barSigma_{\SstarPerp}=ZBZ^\top,\qquad
\barR=ZRZ^\top.
\]
Hence the eigenvalues of $A, B$ and $R$ that we will require are as follows
\[
\lambda_m(A)=\lambda_m(\barSigma_{\Sstar}),\qquad
\lambda_1(B)=\lambda_1(\barSigma_{\SstarPerp}),\qquad
\lambda_{i}(R)=\lambda_{i}(\barR),\ i\in\{1, \ldots, p-m\}.
\]

Furthermore,
$\lambda_1(R)<1$: to prove this, assume, by contradiction, that $Rv=v$ for some nonzero $v\in\R^{p-m}$, and set $z:=Zv\in\SstarPerp$, so that $\|z\|_2^2 = \|v\|_2^2$. Since $R=Z^\top\barPi Z$, $z^\top\barPi z=v^\top Rv=\|v\|_2^2=\|z\|_2^2=z^\top z$. As $0\preceq\barPi\preceq I_p$, $I_p -\barPi$ is positive semi-definite and $z^\top(I_p-\barPi)z=0$ gives $\barPi z=z$, i.e. $\Pi_k^{(e)}z=z$ for every $e$, so $z\in\Sstar$ — contradicting $0\ne z\in\SstarPerp$.

The smallest eigenvalue of the invariant block and the eigenvalues of the non-invariant block are thus
\[
a_m(\lambda):=\lambda_m(A)+2E\lambda,
\qquad
b_j(\lambda):=\lambda_j(B+2E\lambda R) \ \mathrm{for}\ j\in\{1,\ldots, p-m\}.
\]

\smallskip
\noindent\textbf{Step 3: show inclusion of $\Sstar$ for $\lambda>\lambda^\star$ (part~(i)).}
By Weyl's inequality \citep[Theorem~4.3.7]{horn2012matrix} used in decreasing-order notation, applied to
$B,~2E\lambda R$ with indices $(1,k-m+1)$ (recall that we are using decreasing order notation of eigenvalues, whereas Thm.~4.3.7 in~\citep{horn2012matrix} uses increasing order),
\[
b_{k-m+1}(\lambda)
\le
\lambda_1(B)+2E\lambda\,\lambda_{k-m+1}(R).
\]
Since
$
\lambda^\star
=
\frac{\lambda_1(B)-\lambda_m(A)}
{2E\{1-\lambda_{k-m+1}(R)\}},
$
the condition $\lambda>\lambda^\star$ is equivalent to
\[
2E\lambda\{1-\lambda_{k-m+1}(R)\}
>
\lambda_1(B)-\lambda_m(A).
\]
and therefore
\[
a_m(\lambda) = \lambda_m(A)+2E\lambda
>
\lambda_1(B)+2E\lambda\,\lambda_{k-m+1}(R)
\ge
b_{k-m+1}(\lambda).
\]
Hence at most $k-m$ eigenvalues from the
non-invariant coordinate block are at least $a_m(\lambda)$, whereas all $m$
eigenvalues from the invariant coordinate block are at least $a_m(\lambda)$.
Therefore every top-$k$ eigenspace of $M_\lambda$ contains $\Sstar$, that is,
$
\Pianchor{k}\Pi_{\Sstar}=\Pi_{\Sstar}.
$

\smallskip
\noindent\textbf{Step 4: show that the leading $m$ directions span $S^\star$ for $\lambda > \lambda^{\star\star}$ (part~(ii)).}
Again by Weyl's inequality \citep[Theorem~4.3.7]{horn2012matrix} used in decreasing-order notation, with indices $(1,1)$,
\[
b_1(\lambda)
\le
\lambda_1(B)+2E\lambda\,\lambda_1(R).
\]
Since
$
\lambda^{\star\star}
=
\frac{\lambda_1(B)-\lambda_m(A)}
{2E\{1-\lambda_1(R)\}},
$
the condition $\lambda>\lambda^{\star\star}$ is equivalent to
\[
2E\lambda\{1-\lambda_1(R)\}
>
\lambda_1(B)-\lambda_m(A),
\]
and therefore
\[
a_m(\lambda)
=
\lambda_m(A)+2E\lambda
>
\lambda_1(B)+2E\lambda\,\lambda_1(R)
\ge
b_1(\lambda).
\]
Thus every eigenvalue of the invariant block strictly exceeds every eigenvalue
of the non-invariant block. Consequently, the leading $m$-dimensional
eigenspace of $M_\lambda$ is $\Sstar$.

\end{proof}

\begin{remark}[Basis-level interpretation]
Part~(ii) has a basis-level interpretation: the first $m$ directions of any ordered \AnchorPCA{} representative basis (in the sense of \cref{app:ordered_representatives}) span $\Sstar$ and diagonalize $\barSigma$ within $\Sstar$.
Under the weaker threshold in part~(i), only the subspace-level statement $\Sstar\subseteq\Sanchor{k}$ is guaranteed. If the eigenvalues of an invariant block and a non-invariant block of $M_\lambda$ tie, the ordered eigenbasis is determined only up to a rotation inside the tied eigenspace. Hence, the $k$ ordered directions need not contain $m$ vectors spanning $\Sstar$. Specifically, one can always choose a basis whose first $m$ columns span $\Sstar$, but this need not hold for every ordered basis.
\end{remark}

\subsubsection{Invariance of \texorpdfstring{$S_{\star}$}{Sstar} to variance inflation}

\begin{proposition} \label{prop:findsstar_exact_robustness}
    Assume $\lambda_k(\Sigma_e)>\lambda_{k+1}(\Sigma_e)$ for every $e\in\cE$, and
    let $\Sstar=\bigcap_{e\in\cE}\Imop(\Pi_k^{(e)})$ be the maximal invariant
    subspace of $(\Sigma_e)_{e\in\cE}$. Then $\Sstar$ is also the maximal invariant
    subspace of $(\Sigma_e+\Delta_e)_{e\in\cE}$ for every choice of symmetric
    matrices $\Delta_e$ with $0\preceq\Delta_e$ and
    $\Imop(\Delta_e)\subseteq\Imop(\Pi_k^{(e)})$, for all $e\in\cE$. 
\end{proposition}
\begin{proof}
We first show that each admissible perturbation leaves the rank-$k$
principal subspaces unchanged. Intersecting these unchanged subspaces then
shows that $\Sstar$ is unchanged. 
Fix $e\in\cE$ and write $\mathcal S_e:=\Imop(\Pi_k^{(e)})$. Since $\Delta_e$ is
symmetric with $\Imop(\Delta_e)\subseteq\mathcal S_e$, its kernel contains
$\mathcal S_e^\perp$, so both $\Sigma_e$ and $\Delta_e$ are block-diagonal
with respect to $\R^p=\mathcal S_e\oplus\mathcal S_e^\perp$. On
$\mathcal S_e^\perp$, the perturbed covariance $\Sigma_e+\Delta_e$ agrees
with $\Sigma_e$ and so has largest eigenvalue $\lambda_{k+1}(\Sigma_e)$.
On $\mathcal S_e$, since $\Delta_e\succeq 0$, the smallest eigenvalue of
$\Sigma_e+\Delta_e$ is at least $\lambda_k(\Sigma_e)$. The gap
$\lambda_k(\Sigma_e)>\lambda_{k+1}(\Sigma_e)$ then forces the rank-$k$
principal subspace of $\Sigma_e+\Delta_e$ to be exactly $\mathcal S_e$,
so its rank-$k$ principal projector, 
which we denote by $\Pi_{k,\Delta}^{(e)}$ equals $\Pi_k^{(e)}$. Intersecting over
$e$ gives that the maximal invariant subspace of the perturbed covariances
is $\Sstar$, and averaging gives 
$E^{-1}\sum_{e=1}^E\Pi_{k,\Delta}^{(e)}=\barPi$.

\end{proof}

\subsection{Test statistic of \texttt{FindS}\texorpdfstring{$_\star$}{ S*}}\label{app:wald-test}

We now state the Wald-type test statistic used in the sequential procedure for estimating \(\dim(\Sstar)\). Since our construction is adapted from Schott's partial common principal component subspace test \citep{schott1999}, we first recall the standard version in our notation and then explain the modification for \(\Sstar\). We present only the reduced representation, since this is the form used in
computation.
To simplify notation, throughout this appendix section we hide the dependence on $k$. For example, we let \(\Pi_e := \Pi_k^{(e)}\) be the projector onto the top-\(k\) eigenspace in domain \(e\).

\subsubsection{Standard construction.}
Suppose that, for each \(e\in\{1,\dots,E\}\), we observe 
an i.i.d.\ Gaussian sample
with population covariance matrix \(\Sigma_e\), and assume that the \(k\)th and \((k+1)\)st eigenvalues of \(\Sigma_e\) are separated by a strict eigengap. Let \(\Pi_e\) denote the projector onto the top-\(k\) eigenspace of \(\Sigma_e\), and define
$r$ as the dimension of the union
of the top-k eigenspaces of $(\Sigma_e)_{e\in\cE}$, that is, $r := \dim\!\left(\operatorname{span}\bigl(\Imop(\Pi_1)\cup\cdots\cup\Imop(\Pi_E)\bigr)\right).$
Schott considers the hypotheses
\[
H_{0,s}: r=s
\qquad\text{versus}\qquad
H_{A,s}: r>s.
\]
We now show how the test statistic is constructed. Let \(\widehat\Pi_e\) be the sample projector onto the top-\(k\) eigenspace of the sample covariance matrix \(\widehat\Sigma_e\), and let \(\widehat\Pi_0\) be the projector onto the span of the top \(s\) eigenvectors of \(\widehat\Pi_1+\cdots+\widehat\Pi_E\). 
Let \(\widehat C_e\in\mathbb{R}^{p\times k}\) be an orthonormal basis of \(\mathrm{Im}(\widehat\Pi_e)\), and let \(\widehat C_0\in\mathbb{R}^{p\times (p-s)}\) be an orthonormal basis of \(\mathrm{Im}(I_p-\widehat\Pi_0)\). 
The statistic depends on three objects: the deviation vector \(v^\star\), its estimated asymptotic covariance matrix \(\widehat W^\star\), and the projector \(\widehat F\) onto the eigenspace of \(\widehat W^\star\) corresponding to its \(n_s=(Ek-s)(p-s)\) largest eigenvalues.

\paragraph{Deviation vector \(v^\star\).} Define
\[
v^\star
:=
\begin{pmatrix}
\mathrm{vec}(\widehat C_0^\top \widehat C_1)\\
\vdots\\
\mathrm{vec}(\widehat C_0^\top \widehat C_E)
\end{pmatrix}.
\]

\paragraph{Covariance matrix \(\widehat W^\star\).} Here, we need several intermediate components. First, let
\[
\widehat\Pi_\star
:=
\sum_{j=1}^s \widehat\lambda_j \widehat q_j \widehat q_j^\top,
\qquad
\widehat\Pi_\star^+
:=
\sum_{j=1}^s \widehat\lambda_j^{-1} \widehat q_j \widehat q_j^\top,
\]
where \(\widehat\lambda_1\ge\cdots\ge\widehat\lambda_s\) and \(\widehat q_1,\dots,\widehat q_s\) are the top-$s$ eigenvalues and corresponding eigenvectors of \(\widehat\Pi_1+\cdots+\widehat\Pi_E\).
Further, for all \(e\in\{1,\dots,E\}\), let \(\widehat\lambda_{e,1} \ge \cdots \ge \widehat\lambda_{e,p}\) and \(\widehat q_{e,1},\dots,\widehat q_{e,p}\) denote the eigenvalues and corresponding eigenvectors of \(\widehat\Sigma_e\), in decreasing order.
Then, for all $e\in\{1,\dots,E\}$, let
\[
\widehat Y_e^\star
:=
\sum_{j=1}^k \sum_{\ell=k+1}^p
\frac{n\,\widehat\lambda_{e,j}\widehat\lambda_{e,\ell}}
     {n_e(\widehat\lambda_{e,j}-\widehat\lambda_{e,\ell})^2}
\left(
c_j c_j^\top \otimes
\widehat C_0^\top \widehat q_{e,\ell}\widehat q_{e,\ell}^\top \widehat C_0
\right),
\]
where \(c_j\) is the \(j\)th standard basis vector in \(\mathbb{R}^k\), \(n:=\sum_{e=1}^E n_e\), and $\otimes$ denotes the Kronecker product. The eigengap enters here through the factors \((\widehat\lambda_{e,j}-\widehat\lambda_{e,\ell})^{-2}\): if the gap is small, the estimated eigenspace becomes unstable.
Now, for $h,i \in \{1,\ldots, E\}$, define the $k(p-s)\times k(p-s)$ matrix
\[
\widehat V_{hi}
:=
\sum_{f=1}^E
(\widehat C_h^\top \widehat\Pi_\star^+ \widehat C_f \otimes I_{p-s})
\widehat Y_f^\star
(\widehat C_f^\top \widehat\Pi_\star^+ \widehat C_i \otimes I_{p-s})
-
(\widehat C_h^\top \widehat\Pi_\star^+ \widehat C_i \otimes I_{p-s})
\widehat Y_i^\star
-
\widehat Y_h^\star
(\widehat C_h^\top \widehat\Pi_\star^+ \widehat C_i \otimes I_{p-s})
\]
and let the \(Ek(p-s)\times Ek(p-s)\) partitioned matrix whose \((h,i)\)th block is \(\widehat V_{hi}\) be denoted by \(\widehat V\). Finally, let
\[
\widehat W^\star
:=
\operatorname{diag}(\widehat Y_1^\star,\dots,\widehat Y_E^\star)+\widehat V.
\]

\paragraph{Projector onto the leading eigenspace.}
Let \(\widehat F\) be the eigenprojection of \(\widehat W^\star\) corresponding to its
$n_s=(Ek-s)(p-s)$
largest eigenvalues. 

\paragraph{Reduced Wald statistic.}
Schott’s reduced Wald statistic is then
\[
T_s
:=
n\,(v^\star)^\top
(\widehat F \widehat W^\star \widehat F)^+
v^\star,
\]
where \((\cdot)^+\) denotes the Moore--Penrose pseudoinverse. Under the Gaussian assumptions and the eigengap condition, \(T_s\) is asymptotically \(\chi^2_{n_s}\) under \(H_{0,s}\).

\subsubsection{Adaptation to \texorpdfstring{$\Sstar$}{S*}}

For \texttt{FindS}$_\star$, we apply the same construction not to the top-\(k\) projectors \(\Pi_e\), but to their orthogonal complements
\[
\Pi_e^\perp := I_p-\Pi_e.
\]
A vector belongs to \(\Sstar\) if and only if it lies in the null space of every \(\Pi_e^\perp\).
Hence, $\Sstar$ is the common null space of $\Pi_1^\perp,\dots,\Pi_E^\perp$, and the span of these complementary projectors has dimension $r:=p-\dim(\Sstar)$; estimating $\dim(\Sstar)$ is therefore equivalent to estimating $r$.

In practice, this means that Schott’s statistic is applied exactly as above, but with the selected eigenspaces replaced by the \((p-k)\)-dimensional complementary eigenspaces, that is, by $\Pi^\perp_e$, and for $t=p-k,\dots,p-m_\mathrm{min}-1$, test
\[
H_{0,t}: r= t
\quad\text{versus}\quad
H_{A,t}: r>t,
\qquad
r=p-\dim(\Sstar).
\]
Let $\widehat r$ be the first $t$ for which we do not reject. The estimated invariant dimension is then $\widehat m=p-\widehat r$.
This is the same procedure as in the main text under the change of variables $s=p-t$.
For the complementary-space test, $\widehat C_e$ is formed from the empirical eigenvectors
$\widehat q_{e,k+1},\dots,\widehat q_{e,p}$. Accordingly, in the covariance formula
the selected indices range over $j=k+1,\dots,p$ and the non-selected indices over
$\ell=1,\dots,k$; equivalently, one may apply Schott's formula after reordering the
spectrum so that the complementary eigenspace is treated as the selected subspace.

Under the complement parametrization, the degrees of freedom become $n_t = (E(p-k) - t)(p-t)$, so $T_t$ is asymptotically $\chi^2_{n_t}$ under $H_{0,t}$. The eigengap assumption $\lambda_k(\Sigma_e) > \lambda_{k+1}(\Sigma_e)$ separates the top-$k$ and bottom-$(p-k)$ eigenspaces symmetrically, so Schott's regularity conditions apply verbatim to the complementary eigenspaces.

\section{Proofs} \label{app:proofs}

\subsection{Additional intermediate results}
\label{app:additional_intermediate_results}
Here we collect additional results used in the proofs of the main text results. We first introduce principal angles and the corresponding projector-distance identity, then restate the Davis--Kahan theorem, and finally record the three lemmas used in the consistency proof \cref{app:consistency}.

\begin{definition}[Principal angles]
Let $U,V\in\Ok$ and let $\sigma_1\ge\cdots\ge \sigma_k$ be the singular values of $U^\top V$. The \emph{principal angles} between $U$ and $V$ are
\[
\cos^{-1}(\sigma_1),\dots,\cos^{-1}(\sigma_d).
\]
We write $\Theta(U,V)$ for the diagonal matrix of principal angles and $\sin\Theta(U,V)$ for its entrywise sine.
\end{definition}

\begin{proposition}[Projector distance and principal angles]
\label{prop:projector_distance_principal_angles}
Let $\mathcal U,\mathcal V\subseteq\R^p$ be subspaces of equal dimension
$k$, and let $U,V\in\R^{p\times k}$ have orthonormal columns spanning
$\mathcal U$ and $\mathcal V$, respectively. Then
\begin{equation}
    \label{eq:pi_distance_to_sin_distance}
    \opnorm{\Pi_{\mathcal U}-\Pi_{\mathcal V}}
    =
    \opnorm{\sin\Theta(U,V)}.
\end{equation}
\end{proposition}
\begin{proof}
This is the standard projector-distance identity for equal-dimensional
subspaces; see, for example, \citep[Thm.~2.1]{drineas2019lowrank}. Since
$U$ and $V$ have orthonormal columns, their orthogonal projectors are
$\Pi_{\mathcal U}=UU^\top$ and $\Pi_{\mathcal V}=VV^\top$.
\end{proof}

\begin{theorem}[Davis--Kahan $\sin\Theta$ theorem \citep{yu2015DKSinTheta}]
\label{thm:davis_kahan}
Let $\Sigma,\widehat{\Sigma}\in\R^{p\times p}$ be symmetric, with eigenvalues $\lambda_1\ge\cdots\ge\lambda_p$ and $\widehat{\lambda}_1\ge\cdots\ge\widehat{\lambda}_p$
and let $r, s \in \{1,\ldots, p\}$ such that $r \le s$.
Let 
$V$ and $\widehat{V}$ contain orthonormal eigenvectors associated with the eigenvalues 
$\lambda_r,\dots,\lambda_s$ and $\widehat{\lambda}_r,\dots,\widehat{\lambda}_s$, respectively. Let
\[
\delta
:=
\inf\{|\widehat{\lambda}-\lambda|:\lambda\in[\lambda_s,\lambda_r],\ \widehat{\lambda}\in(-\infty,\widehat{\lambda}_{s+1}]\cup[\widehat{\lambda}_{r-1},\infty)\},
\]
with $\widehat{\lambda}_0:=\infty$ and $\widehat{\lambda}_{p+1}:=-\infty$. If $\delta>0$, then
\begin{equation}
\label{eq:davis_kahan}
\opnorm{\sin\Theta(\widehat{V},V)}
\le
\frac{\opnorm{\widehat{\Sigma}-\Sigma}}{\delta}.
\end{equation}
\end{theorem}

\begin{lemma}[Top-$d$ eigenspace perturbation bound]
\label{lem:top_projector_perturbation}
Let \(A,\widehat A\) 
be symmetric \(p \times p\) matrices, and let $d\in\{1,\ldots,p-1\}$. Let \(\Pi_d(A)\) and
\(\Pi_d(\widehat A)\) denote the 
orthogonal projector
onto the spans of the top \(d\) eigenvectors
of
\(A\) and \(\widehat A\), respectively. If
$
\gamma_d(A):=\lambda_d(A)-\lambda_{d+1}(A)>0
$
and
$
\opnorm{\widehat A-A}<\frac{\gamma_d(A)}{2},
$
then
\[
\opnorm{\Pi_d(\widehat A)-\Pi_d(A)}
\le
\frac{2\opnorm{\widehat A-A}}{\gamma_d(A)}.
\]
\end{lemma}
\begin{proof}
Let
$
\varepsilon:=\opnorm{\widehat A-A}.
$
By Weyl's inequality \citep[Theorem~4.3.7]{horn2012matrix} in decreasing-order notation with indices $(d+1,1)$ applied to $(A, \widehat A - A)$,
$
\lambda_{d+1}(\widehat A)\le \lambda_{d+1}(A)+\varepsilon,
$
so
\[
\lambda_d(A)-\lambda_{d+1}(\widehat A)
\ge
\lambda_d(A)-\lambda_{d+1}(A)-\varepsilon
=
\gamma_d(A)-\varepsilon
> 0,
\]
where the last inequality uses \(\varepsilon < \gamma_d(A)/2\) which holds by assumption.
Let $V$ and $\widehat V$ have orthonormal columns spanning the top-$d$
eigenspaces of $A$ and $\widehat A$, respectively. Applying
\cref{thm:davis_kahan} with $\Sigma=A$, $\widehat\Sigma=\widehat A$, $r=1$,
and $s=d$. The separation quantity is \(\delta=\inf\{|\widehat\lambda-\lambda|:\lambda\in[\lambda_d(A),\lambda_1(A)],\,\widehat\lambda\in(-\infty,\lambda_{d+1}(\widehat A)]\}=\lambda_d(A)-\lambda_{d+1}(\widehat A)\ge\gamma_d(A)-\varepsilon>\gamma_d(A)/2\).
Therefore,
\[
\opnorm{\Pi_d(\widehat A)-\Pi_d(A)}
=
\opnorm{\sin\Theta(\widehat V,V)}
\le
\frac{\varepsilon}{\delta}
\le
\frac{2\varepsilon}{\gamma_d(A)},
\]
where the first equality follows from~\eqref{eq:pi_distance_to_sin_distance}.
\end{proof}

Recall the following asymptotic notation. Let $(b_N)_{N\ge1}$ be a deterministic
positive sequence and let $(R_N)_{N\ge1}$ be a sequence of random scalars or
matrices. Fix a norm $\|\cdot\|$ on the relevant space; for scalars this is the
absolute value. 
We write
$R_N=O(b_N)$ almost surely if
there exists
an event of 
probability one on which, for each realization, there exist a finite constant
$C< \infty$ and index $N_0 \in \mathbb{N}_+$ (both allowed to depend on the realization) such that
$\|R_N\|\le Cb_N$ for all $N\ge N_0$. We write $R_N=o(b_N)$ almost surely if
$\|R_N\|/b_N\to0$ almost surely. 
For a matrix-valued $R_N$, an entrywise $O(b_N)$ (resp. $o(b_N)$) statement means that every entry $(R_N)_{ij}$ is itself $O(b_N)$ (resp. $o(b_N)$) almost surely, with the constant and index threshold allowed to depend on $(i,j)$.

\begin{lemma}[Polynomial rate for the averaged projectors]
\label{lem:local_projector_polynomial_rate}
Assume that, for all $e\in\cE$, the observations $x_{e,1},\dots,x_{e,n_e}$ are i.i.d.\ copies of $x_e$ with $\EE[x_e]=0$, $\EE\|x_e\|_2^4<\infty$, and $\lambda_k(\Sigma_e)>\lambda_{k+1}(\Sigma_e)$. Then, for all
$\alpha\in(0,1/2)$ and all $e\in\cE$,
\[
\opnorm{\widehat\Sigma_e-\Sigma_e}=o(n_e^{-\alpha})
\qquad\text{and}\qquad
\opnorm{\widehat\Pi_k^{(e)}-\Pi_k^{(e)}}=o(n_e^{-\alpha})
\]
almost surely. Consequently, with
$n:=\min_{e\in\cE}n_e$,
\[
\opnorm{\hatSigma-\barSigma}=o(n^{-\alpha})
\qquad\text{and}\qquad
\opnorm{\hatPi-\barPi}=o(n^{-\alpha})
\]
almost surely.
\end{lemma}
\begin{proof}

\smallskip
\noindent
\textbf{Step 1: convergence of coordinates in each domain.}
Let $e\in\cE$ and $\alpha\in(0,1/2)$. In this proof only, let $N:=n_e$ and
$X_i:=x_{e,i}$. Choose $s\in(1,2)$ such that
$\alpha<1-1/s$.
For all $a,b\in\{1,\ldots,p\}$ and all $i\in\{1,\ldots,N\}$, let
$Y_i^{ab}:=(X_i)_a(X_i)_b$ (where $(X_i)_a$ denotes the $a$-th coordinate of $X_i$). Since $s<2$ and $\EE\|x_e\|_2^4<\infty$, we have
\[
\EE |Y_i^{ab}|^s
=\EE\!\left[|(X_i)_a|^s|(X_i)_b|^s\right]
\le
\EE\|X_i\|_2^{2s}
<\infty 
\]
where the final bound uses $2s<4$.
By the Marcinkiewicz--Zygmund strong law of large numbers
\citep[Theorem~8.4.4, statement~(ii)]{athreya_lahiri_2006},
\[
\frac{1}{N^{1/s}}\sum_{i=1}^{N}\left(Y_i^{ab}-\EE Y_i^{ab}\right)
\]
converges almost surely,
and is therefore bounded almost surely. Dividing by
$N^{1-1/s}$ gives
\[
\frac1N\sum_{i=1}^{N}Y_i^{ab}-\EE Y_i^{ab}
=
O\!\left(N^{-(1-1/s)}\right)
=
o(N^{-\alpha})
\qquad\text{almost surely,}
\]
where the last equality uses $\alpha<1-1/s$.
Applying the same strong law to the coordinate variables $(X_i)_a$ and using
$\EE(X_i)_a=0$, gives
$\overline{X}_a:=\frac{1}{N}\sum_{i=1}^N (X_i)_a=O(N^{-(1-1/s)})$ almost surely, and hence
$\overline{X}_a\overline{X}_b=O(N^{-2(1-1/s)})=o(N^{-\alpha})$ almost surely, since $2(1-1/s) > \alpha$.

\smallskip
\noindent
\textbf{Step 2: convergence of covariances in each domain.}
Let $A_N:=N^{-1}\sum_{i=1}^{N}X_iX_i^\top$. Since
$\widehat\Sigma_e
=
\frac{N}{N-1}\left(A_N-\overline{X}~\overline{X}^\top\right)$,
we have
\[
\widehat\Sigma_e-\Sigma_e
=
(A_N-\Sigma_e)-\overline{X}~\overline{X}^\top
+
\frac1{N-1}\left(A_N-\overline{X}~\overline{X}^\top\right).
\]
The first two terms are entrywise $o(N^{-\alpha})$ almost surely by the bounds of Step~1: for $A_N-\Sigma_e$ this uses $\EE[X_iX_i^\top]=\Sigma_e$, which holds
since $\EE[X_i]=0$, so that the entrywise bound on $A_N$ from Step~1 is a bound
on $A_N-\Sigma_e$. The last term is entrywise $O(N^{-1})=o(N^{-\alpha})$ almost surely,
because $A_N-\overline{X}~\overline{X}^\top=O(1)$ entrywise almost surely and
$\alpha<1$. Since $p$ is fixed, entrywise convergence at rate
$o(N^{-\alpha})$ implies operator-norm convergence at the same rate. Therefore
\begin{equation} \label{eq:conv_cov}
    \opnorm{\widehat\Sigma_e-\Sigma_e}=o(n_e^{-\alpha})
    \qquad\text{almost surely.}
\end{equation}

\smallskip
\noindent
\textbf{Step 3: convergence of projection matrices in each domain.}
Let $\gamma_k(\Sigma_e):=\lambda_k(\Sigma_e)-\lambda_{k+1}(\Sigma_e)>0$, by assumption. By~\eqref{eq:conv_cov}, for all
sufficiently large $n_e$,
$\opnorm{\widehat\Sigma_e-\Sigma_e}<\gamma_k(\Sigma_e)/2$ almost surely. Applying
\cref{lem:top_projector_perturbation} with $A=\Sigma_e$,
$\widehat A=\widehat\Sigma_e$, and $d=k$ gives
\[
\opnorm{\widehat\Pi_k^{(e)}-\Pi_k^{(e)}}
\le
\frac{2}{\gamma_k(\Sigma_e)}\opnorm{\widehat\Sigma_e-\Sigma_e}
=
o(n_e^{-\alpha})
\qquad\text{almost surely,}
\]
where the last equality follows from~\eqref{eq:conv_cov}.

\smallskip
\noindent
\textbf{Step 4: convergence of the average covariance and average projection matrix.}
Since $E$ is fixed and $n_e\ge n$ for all $e\in\cE$, averaging over
$e\in\cE$ yields
\[
\opnorm{\hatSigma-\barSigma}
\le
\frac1E\sum_{e=1}^E\opnorm{\widehat\Sigma_e-\Sigma_e}
=
o(n^{-\alpha})
\]
and
\[
\opnorm{\hatPi-\barPi}
\le
\frac1E\sum_{e=1}^E\opnorm{\widehat\Pi_k^{(e)}-\Pi_k^{(e)}}
=
o(n^{-\alpha})
\]
almost surely.
\end{proof}

\begin{lemma}[Recovery of the 
eigenvalue
blocks of $\barPi$] \label{lem:barPi_block_recovery}
Let $1\ge \rho_1>\cdots>\rho_L\ge0$ be the distinct eigenvalues of $\barPi$, and let $\Pi_\ell$ denote the orthogonal projector onto the eigenspace of 
$\rho_\ell$,
with multiplicity $d_\ell:=\operatorname{rank}(\Pi_\ell)$. Listing the
eigenvalues of $\barPi$ in nonincreasing order with multiplicity, $\rho_\ell$
occupies a block of $d_\ell$ consecutive indices which we denote $I_\ell$, so that
$I_1,\dots,I_L$ partition $\{1,\dots,p\}$.
Let $\gamma_{\barPi}$ be the smallest gap between the distinct population eigenvalues, 
\[
\gamma_{\barPi}
:=
\begin{cases}
\min\{\rho-\rho' : \rho,\rho'\in\{\rho_1,\ldots,\rho_L\},\ \rho>\rho'\},
& L\ge 2,\\[0.3em]
+\infty, & L=1.
\end{cases}
\]
Let $\widehat\rho_1\ge\cdots\ge\widehat\rho_p$ be the eigenvalues of $\hatPi$ (including duplicates). 
Suppose that $\| \hatPi-\barPi\|_\mathrm{op}<\mathrm{tol}_n/4$ and $\mathrm{tol}_n<\gamma_{\barPi}/2$. 
Then the block-stabilized construction of \AnchorPCAinfty{} in \cref{subsec:finite_sample_implementation} 
partitions the indices
$1,\dots,p$ into empirical blocks that coincide with $I_1,\dots,I_L$.
\end{lemma} 

\begin{proof} 
Let $\ell\in\{1,\dots,L-1\}$ and $c_\ell:=\sum_{j<\ell}d_j$, so that $I_\ell=\{c_\ell+1,\dots,c_\ell+d_\ell\}$. Let $\varepsilon_n:=\|{\hatPi-\barPi}\|_\mathrm{op}$.
Recall that the construction in \cref{subsec:finite_sample_implementation} scans the ordered eigenvalues of $\hatPi$ from largest to smallest to group them into empirical 
blocks: each block starts at an eigenvalue and includes subsequent eigenvalues as long as they are at most $\mathrm{tol}_n$ smaller than the first eigenvalue of the block.

\smallskip
\noindent
\textbf{Step 1: no block $I_\ell$ is split.}
By Weyl's inequality \citep[Theorem~4.3.1]{horn2012matrix} in decreasing-order notation applied to the two symmetric matrices $\barPi$ and $\hatPi-\barPi$, we have that the eigenvalues of
$\hatPi$ and $\barPi$, matched by rank order and counted with multiplicity,
satisfy
$|\lambda_j(\hatPi)-\lambda_j(\barPi)|\le\|\hatPi-\barPi\|_\mathrm{op} =\varepsilon_n$ for all $j\in\{1,\dots,p\}$. 
Since $\lambda_j(\barPi)=\rho_\ell$ for all $j\in I_\ell$, it follows that $|\widehat\rho_j-\rho_\ell|\le\varepsilon_n$ for all $j\in I_\ell$.
Hence, for all 
$j,j'\in I_\ell$,
\[
|\widehat\rho_j-\widehat\rho_{j'}|\le |\widehat\rho_j-\rho_\ell|+|\widehat\rho_{j'}-\rho_\ell| \le2\varepsilon_n<\mathrm{tol}_n/2<\mathrm{tol}_n,
\]
where the third inequality follows by the assumption $\varepsilon_n = \| \hatPi-\barPi \|_\mathrm{op} < \mathrm{tol}_n / 4$.
In particular, every eigenvalue in $I_\ell$ lies within $\mathrm{tol}_n$ of the
first eigenvalue of the block, $\widehat\rho_{c_\ell+1}$, so the rule keeps all indices in $I_\ell$
in the same empirical block.

\smallskip
\noindent
\textbf{Step 2: adjacent blocks $I_\ell$ and $I_{\ell+1}$ are not merged.} 
Applying Weyl's inequality \citep[Theorem~4.3.1]{horn2012matrix} in decreasing-order notation to $\barPi$ and $\hatPi - \barPi$, for any index \(j\) we have \(\lambda_j(\overline{\Pi})+\lambda_p(\hatPi - \barPi)\le \lambda_j(\widehat{\Pi})\le \lambda_j(\overline{\Pi})+\lambda_1(\hatPi - \barPi)\), and hence \(\lambda_j(\overline{\Pi})-\varepsilon_n\le \lambda_j(\widehat{\Pi})\le \lambda_j(\overline{\Pi})+\varepsilon_n\). Taking \(j=c_\ell+1\) and using \(\lambda_{c_\ell+1}(\overline{\Pi})=\rho_\ell\) gives \(\widehat{\rho}_{c_\ell+1}=\lambda_{c_\ell+1}(\widehat{\Pi})\ge \rho_\ell-\varepsilon_n\). Similarly, taking \(j=c_\ell+d_\ell+1\) and using \(\lambda_{c_\ell+d_\ell+1}(\overline{\Pi})=\rho_{\ell+1}\) gives \(\widehat{\rho}_{c_\ell+d_\ell+1}=\lambda_{c_\ell+d_\ell+1}(\widehat{\Pi})\le \rho_{\ell+1}+\varepsilon_n\).
Therefore, 
\[ 
\widehat\rho_{c_\ell+1}-\widehat\rho_{c_\ell+d_\ell+1} \ge \rho_\ell-\rho_{\ell+1}-2\varepsilon_n \ge \gamma_{\barPi}-2\varepsilon_n > \gamma_{\barPi}-\frac{\mathrm{tol}_n}{2} > \mathrm{tol}_n, 
\] 
where the second last inequality uses that 
$\varepsilon_n<\mathrm{tol}_n/4$, and the final inequality uses $\mathrm{tol}_n<\gamma_{\barPi}/2$ (both hold by assumption). 
Thus the first eigenvalue past $I_\ell$ lies more than $\mathrm{tol}_n$ below
$\widehat\rho_{c_\ell+1}$, the first eigenvalue of the current block, so the rule
closes that block at index $c_\ell+d_\ell$ and starts a new one.

Going through the eigenvalues from largest to smallest, Steps 1 and 2 show that
each block starts at index $c_\ell+1$ and ends at index $c_\ell+d_\ell$: Step 1
keeps a block from ending too early, and Step 2 forces it to end at the right
place. The blocks the construction produces are therefore exactly
$I_1,\dots,I_L$.
\end{proof}

\subsection{Proof of \texorpdfstring{\cref{thm:invariant_subspace_limit}}{Theorem~\ref{thm:invariant_subspace_limit}}}

\cref{thm:invariant_subspace_limit}
is a projector-level statement. We prove~\cref{thm:invariant_subspace_limit_extended}, which is a
stronger version that uses the ordered
convention of \cref{app:ordered_representatives} and additionally records what happens to the ordered
directions. 
Part~\textup{(i)} concerns \AnchorPCAinfty{}, and
Part~\textup{(ii)} concerns \AnchorPCA{} as $\lambda\to\infty$; both are stated first at the
projector level and then at the direction level under a simple-spectrum \cref{ass:simple_sstar_block}. Together, the two parts imply
\cref{thm:invariant_subspace_limit}. If $m=0$, the statements below are
vacuous, so we assume $m\ge 1$.

\begin{assumption}[Simple spectrum on the invariant block]
\label{ass:simple_sstar_block}
The restriction of $\barSigma_{\Sstar}$ to $\Sstar$ (i.e., $\barSigma_{\Sstar}:=\Pi_{\Sstar}\barSigma\Pi_{\Sstar}$)
has simple spectrum within $\Sstar$: there exist
$\mu_1>\cdots>\mu_m$ and an orthonormal basis
$s_1,\dots,s_m\in\Sstar$ such that
$\barSigma_{\Sstar}s_i=\mu_i s_i$ for all $i\in\{1,\dots,m\}$.
\end{assumption}

\begin{theorem}[\cref{thm:invariant_subspace_limit}, extended version]
\label{thm:invariant_subspace_limit_extended}
Assume $1\le m:=\dim(\Sstar)\le k<p$, and let
$
\barSigma_{\Sstar}=\Pi_{\Sstar}\barSigma\Pi_{\Sstar}.
$ 
Let $\Pi_k^{\anchor,\infty}$ and $\Pianchor{k}$ be any rank-$k$ solution
projectors of \AnchorPCAinfty{} and \AnchorPCA{}, respectively, and let
\[
W_k^{\anchor,\infty}=(u_1,\dots,u_k),
\qquad
\Wanchor{k}=(w_1^\lambda,\dots,w_k^\lambda)
\]
be corresponding ordered bases
(\cref{app:ordered_representatives}).

\begin{enumerate}
\item[(i)] \emph{(\AnchorPCAinfty.)} 
The first $m$ vectors $u_1,\dots,u_m$ 
form an
ordered orthonormal basis of $\Sstar$ that diagonalizes $\barSigma_{\Sstar}$ within $\Sstar$;
in particular,
$
\operatorname{span}(u_1,\dots,u_m)=\Sstar
$ and $\Pi_k^{\anchor,\infty}\Pi_{\Sstar}=\Pi_{\Sstar}$. If additionally \cref{ass:simple_sstar_block} holds, then $u_i\in\{+s_i, -s_i\}$ for all $i\in\{1,\dots,m\}$.

\item[(ii)] \emph{(\AnchorPCA, $\lambda\to\infty$.)}
Let $\Pianchor{m}$ denote the projector onto
$\operatorname{span}(w_1^\lambda,\dots,w_m^\lambda)$, that is, the leading rank-$m$ spectral projector of $M_\lambda$. Then,
\[
\opnorm{\Pianchor{m}-\Pi_{\Sstar}}\to0
\qquad
\text{as }\lambda\to\infty.
\]
If additionally \cref{ass:simple_sstar_block} holds,
then, with a suitable choice of
signs, as $\lambda\to\infty$,
\[
\lambda_i(M_\lambda)-2E\lambda\to \mu_i,
\qquad
w_i^\lambda\to s_i
\quad
\text{for all } i\in\{1,\dots,m\}.
\]
\end{enumerate}
\end{theorem}

\begin{proof}

\smallskip
\noindent
\textbf{Part~\textup{(i)} $\AnchorPCAinfty{}$.}
Since $\Sstar$ is the eigenspace of $\barPi$ corresponding to its largest
eigenvalue $1$, and $\dim(\Sstar)=m\le k$, the ordering convention of
\cref{app:ordered_representatives} draws the first $m$ directions
$u_1,\dots,u_m$ entirely from $\Sstar$. Hence
\[
\operatorname{span}(u_1,\dots,u_m)=\Sstar,
\qquad\text{so}\qquad
\Pi_k^{\anchor,\infty}\Pi_{\Sstar}=\Pi_{\Sstar}.
\]
Within this block, the convention orders directions by diagonalizing
$\barSigma$, equivalently $\barSigma_{\Sstar}$, within $\Sstar$. Thus $u_1,\dots,u_m$ is an
orthonormal eigenbasis of $\barSigma_{\Sstar}$ on $\Sstar$, ordered by
decreasing eigenvalue. Under \cref{ass:simple_sstar_block}, this eigenbasis
is unique up to sign, so $u_i\in\{+s_i,-s_i\}$ for each $i\in\{1,\dots,m\}$.

\smallskip
\noindent
\textbf{Part~\textup{(ii)} \AnchorPCA{}.}
We first show the claimed subspace convergence in steps 1--2 invoking the Davis--Kahan theorem (\cref{thm:davis_kahan}). The direction-level convergence under \cref{ass:simple_sstar_block} is proved in steps 3--6.

\smallskip
\noindent
\textbf{Step 1: Uniqueness of the top-$m$ eigenspace of $(2E\lambda)^{-1}M_\lambda$.}
Let $S\in\R^{p\times m}$ have orthonormal columns spanning $\Sstar$, and
complete it to an orthogonal matrix $(S,Z)$ with
$Z\in\R^{p\times(p-m)}$ spanning $\SstarPerp$. Since $\Sstar$ is the
eigenvalue-$1$ eigenspace of $\barPi$, we have
\[
(S,Z)^\top\barPi(S,Z)
=
\begin{pmatrix}
I_m & 0\\
0 & R
\end{pmatrix},
\qquad
R:=Z^\top\barPi Z.
\]

Let
$
\widetilde M_\lambda
:=
\frac{1}{2E\lambda}M_\lambda
=
\barPi+\frac{1}{2E\lambda}\barSigma.
$
By Weyl's inequality \citep[Theorem 4.3.1]{horn2012matrix} in decreasing-order notation applied to $(\barPi, (2E\lambda)^{-1}\barSigma)$, we have, for each index \(j\), \(\lambda_j(\overline{\Pi})-\|(2E\lambda)^{-1}\overline{\Sigma}\|_{\mathrm{op}}\le \lambda_j(\widetilde M_\lambda)\le \lambda_j(\overline{\Pi})+\|(2E\lambda)^{-1}\overline{\Sigma}\|_{\mathrm{op}}\). Since \(\|(2E\lambda)^{-1}\overline{\Sigma}\|_{\mathrm{op}}=(2E\lambda)^{-1}\|\overline{\Sigma}\|_{\mathrm{op}}\), applying this with \(j=m\) and using \(\lambda_m(\overline{\Pi})=1\) gives \(\lambda_m(\widetilde M_\lambda)\ge 1-(2E\lambda)^{-1}\|\overline{\Sigma}\|_{\mathrm{op}}\). Applying the same bound with \(j=m+1\) and using \(\lambda_{m+1}(\overline{\Pi})=\lambda_1(R)\) gives \(\lambda_{m+1}(\widetilde M_\lambda)\le \lambda_1(R)+(2E\lambda)^{-1}\|\overline{\Sigma}\|_{\mathrm{op}}\).
Hence, 
for all $\lambda > \lambda_0 :=\frac{\opnorm{\barSigma}}{E(1-\lambda_1(R))}$, we have 
$\lambda_m(\widetilde M_\lambda)>\lambda_{m+1}(\widetilde M_\lambda)$.
Thus, the leading rank-$m$ spectral projector of $M_\lambda$ is unique for all $\lambda > \lambda_0$.

\smallskip
\noindent
\textbf{Step 2: Apply Davis--Kahan to obtain subspace convergence.}
We have that
\begin{equation} \label{eq:delta_bound}
1-\lambda_{m+1}(\widetilde M_\lambda)
\ge
1-\lambda_1(R)-\frac{1}{2E\lambda}\opnorm{\barSigma}
>0,
\end{equation}
where the first inequality follows from step 1 and the second holds for all $\lambda > \lambda_0$ as $\lambda_1(R) < 1$ (shown in step 2 in the proof of~\cref{prop:exact_recovery}).
We apply \cref{thm:davis_kahan} to
$\Sigma:=\barPi$ and $\widehat\Sigma:=\widetilde M_\lambda$
with $r=1$ and $s=m$. Hence, $V = S$, $\widehat V = (w_1^\lambda,\ldots,w_m^\lambda) =: W_m^\lambda$. 
With these choices, the
separation parameter
$
\delta
=
\inf\{|\widehat\lambda-\lambda|:
\lambda\in[\lambda_m(\barPi)=1,\lambda_1(\barPi)=1],\
\widehat\lambda\in(-\infty,\lambda_{m+1}(\widetilde M_\lambda)]\}
=
|1-\lambda_{m+1}(\widetilde M_\lambda)|,$
so the theorem states that
\begin{align*}
    \opnorm{\sin\Theta(W_m^\lambda,S)}
&\le
\frac{\opnorm{\barPi - \widetilde M_\lambda}}{|1-\lambda_{m+1}(\widetilde M_\lambda)|}. 
\end{align*}
We apply~\eqref{eq:delta_bound} for all $\lambda > \lambda_0$ and then consider $\lambda \to \infty$, which gives
\begin{align*}
\opnorm{\sin\Theta(W_m^\lambda,S)}
&\le
\frac{(2E\lambda)^{-1}\opnorm{\barSigma}}
{1-\lambda_1(R)-(2E\lambda)^{-1}\opnorm{\barSigma}}
\to 0.
\end{align*}
Since $\Pianchor{m} = W_m^\lambda (W_m^{\lambda})^{\top}$ and $\Pi_{\Sstar} = SS^\top$, by~\eqref{eq:pi_distance_to_sin_distance}, as $\lambda \to\infty$,
\[
\opnorm{\Pianchor{m}-\Pi_{\Sstar}}\to0.
\]

\smallskip
\noindent
\textbf{Step 3: Write the eigenvalue equation in $\Sstar\oplus\SstarPerp$ coordinates.}
Assume now, in addition, that \cref{ass:simple_sstar_block} holds. 
For the remainder of the proof, set
$S:=(s_1,\dots,s_m)$ and complete it to an orthogonal matrix $(S,Z)$.
In the basis $(S,Z)$,
\[
(S,Z)^\top M_\lambda(S,Z)
=
\begin{pmatrix}
2E\lambda I_m+A & C\\
C^\top & 2E\lambda R+B
\end{pmatrix},
\]
where 
\[
A:=S^\top\barSigma_{\Sstar}S
  =S^\top\barSigma S
  =\operatorname{diag}(\mu_1,\dots,\mu_m),
\qquad
B:=Z^\top\barSigma Z,
\qquad
C:=S^\top\barSigma Z.
\]
The equality $S^\top\barSigma_{\Sstar}S=S^\top\barSigma S$ follows from
$\Pi_{\Sstar}S=S$.

Let $i\in\{1,\dots,m\}$. Recall that $M_\lambda w_i^\lambda = \lambda_i(M_\lambda)w_i^\lambda$, i.e. $w_i^\lambda$ is the $i$th leading eigenvector of $M_\lambda$. 
Since the columns of $(S,Z)$ form an orthonormal basis of $\R^p$, the vector
$w_i^\lambda$ has unique coordinates $u_i^\lambda\in\R^m$ and
$v_i^\lambda\in\R^{p-m}$ in this basis, so that
\[
w_i^\lambda=Su_i^\lambda+Zv_i^\lambda.
\]
We also let $\eta_i(\lambda):=\lambda_i(M_\lambda)-2E\lambda$.
The eigenvalue equation
$M_\lambda w_i^{\lambda}
=\lambda_i(M_\lambda)w_i^{\lambda}$ then is equivalent to
\begin{align}
(A-\eta_i(\lambda)I_m)u_i^\lambda+Cv_i^\lambda&=0, \label{eq:anchor_block_1}\\
C^\top u_i^\lambda
-
\bigl(\underbrace{2E\lambda(I_{p-m}-R)
+\eta_i(\lambda)I_{p-m}-B}_{=:L_i(\lambda)}\bigr)v_i^\lambda
&=0.
\label{eq:anchor_block_2}
\end{align}

\smallskip
\noindent
\textbf{Step 4: 
Reduce to an eigenvalue equation in $\Sstar$.}
Rearranging~\eqref{eq:anchor_block_2} gives
\[
L_i(\lambda)v_i^\lambda=C^\top u_i^\lambda.
\]

Because $R\preceq \lambda_1(R) I_{p-m}$ with $\lambda_1(R)<1$, we have
$I_{p-m}-R\succeq (1-\lambda_1(R))I_{p-m}$. 
By Weyl's inequality \cite[Theorem 4.3.1]{horn2012matrix}, used in decreasing-order notation and applied to $(2E\lambda\overline{\Pi},~\overline{\Sigma})$, for each \(i\le m\) we have \(\lambda_i(2E\lambda\overline{\Pi})+\lambda_p(\overline{\Sigma})\le \lambda_i(M_\lambda)\le \lambda_i(2E\lambda\overline{\Pi})+\lambda_1(\overline{\Sigma})\). Since \(S_\star\) is the eigenvalue-\(1\) eigenspace of \(\overline{\Pi}\), \(\lambda_i(2E\lambda\overline{\Pi})=2E\lambda\) for \(i\le m\). Hence \(\lambda_p(\overline{\Sigma})\le \lambda_i(M_\lambda)-2E\lambda\le \lambda_1(\overline{\Sigma})\), and therefore 
\begin{equation} \label{eq:eta_i_bound}
|\eta_i(\lambda)|\le \|\overline{\Sigma}\|_{\mathrm{op}}.
\end{equation}
Combining this with $\opnorm{B}\le \opnorm{\barSigma}$, we have
\[
L_i(\lambda)
=
2E\lambda(I_{p-m}-R)+\eta_i(\lambda)I_{p-m}-B
\succeq
\bigl(2E\lambda(1-\lambda_1(R))-2\opnorm{\barSigma}\bigr)I_{p-m}.
\]
Thus, for all $\lambda > \lambda_0$, the matrix $L_i(\lambda)$ is
positive definite and therefore invertible. In addition, for all $\lambda > \lambda_0$,
\begin{equation}\label{eq:Li_lambda_inverse_bound}
    \opnorm{L_i(\lambda)^{-1}}
\le
\frac{1}{2E\lambda(1-\lambda_1(R))-2\opnorm{\barSigma}}
.
\end{equation}
Since $L_i(\lambda)v_i^\lambda=C^\top u_i^\lambda$ and
$\|u_i^\lambda\|_2\le 1$, it follows that
\[
\|v_i^\lambda\|_2
\le
\opnorm{L_i(\lambda)^{-1}}\opnorm{C}\|u_i^\lambda\|_2
\le \opnorm{L_i(\lambda)^{-1}}\opnorm{C}
.
\]
Hence, $\|v_i^\lambda\|_2\to 0$ as $\lambda\to\infty$ by~\eqref{eq:Li_lambda_inverse_bound}.
Because $w_i^{\lambda}=Su_i^\lambda+Zv_i^\lambda$ has unit norm and
$(S,Z)$ is orthogonal, we have
$\|u_i^\lambda\|_2^2+\|v_i^\lambda\|_2^2=1$. Therefore
\begin{equation}\label{eq:ui_limit}
    \|u_i^\lambda\|_2\to1
\end{equation}
as $\lambda\to\infty$.

Substituting
$
v_i^\lambda=L_i(\lambda)^{-1}C^\top u_i^\lambda
$
into~\eqref{eq:anchor_block_1} gives
$
(A-\eta_i(\lambda)I_m)u_i^\lambda
+
CL_i(\lambda)^{-1}C^\top u_i^\lambda
=
0,
$
or equivalently
\begin{equation} \label{eq:reduced_eigenvalue_eq}
(A+CL_i(\lambda)^{-1}C^\top)u_i^\lambda=\eta_i(\lambda)u_i^\lambda.
\end{equation}

\smallskip
\noindent
\textbf{Step 5: Identify 
the limiting eigenvalues and directions along subsequences.}
Fix any sequence $\lambda_n\to\infty$. For all $i\in\{1,\dots,m\}$
, the sequence
$\eta_i(\lambda_n)$ is bounded (see~\eqref{eq:eta_i_bound}). By~\eqref{eq:ui_limit}, after discarding
finitely many terms we have $\|u_i^{\lambda_n}\|_2>0$ for all
$i\in\{1,\dots,m\}$. For such $n$, define
$
\widehat u_i^{\lambda_n}
:=
\frac{u_i^{\lambda_n}}{\|u_i^{\lambda_n}\|_2}
$
, which belongs to the compact unit sphere of $\R^m$. 
By the Bolzano--Weierstrass theorem, each of these sequences has a convergent
subsequence. Passing to a common subsequence (which we continue to index by
$n$) along which all $2m$ sequences converge, it holds that for all
$i\in\{1,\dots,m\}$, there exist $\eta_i^\star\in\R$ and a unit vector
$\widetilde u_i\in\R^m$ such that
\[
\eta_i(\lambda_n)\to\eta_i^\star,
\qquad
\widehat u_i^{\lambda_n}\to\widetilde u_i.
\]
Dividing the reduced eigenvalue equation~\eqref{eq:reduced_eigenvalue_eq} by $\|u_i^{\lambda_n}\|_2$ gives
\[
\bigl(A+CL_i(\lambda_n)^{-1}C^\top\bigr)\widehat u_i^{\lambda_n}
=\eta_i(\lambda_n)\widehat u_i^{\lambda_n}.
\]
Passing to the limit $\lambda_n\to\infty$ along the chosen subsequence, the term $CL_i(\lambda_n)^{-1}C^\top$
vanishes since $\opnorm{CL_i(\lambda)^{-1}C^\top}
\le
\opnorm{C}^2\opnorm{L_i(\lambda)^{-1}} \to 0$ by~\eqref{eq:Li_lambda_inverse_bound}, and
we obtain
\[
A\widetilde u_i=\eta_i^\star\widetilde u_i.
\]
This implies that $\widetilde u_i$ is a unit eigenvector of $A$ with eigenvalue $\eta_i^\star$. Since $A=\operatorname{diag}(\mu_1,\dots,\mu_m)$ with
$\mu_1>\cdots>\mu_m$, there exists a unique index
$j_i\in\{1,\dots,m\}$ such that
\[
\eta_i^\star=\mu_{j_i},
\qquad
\widetilde u_i=\pm e_{j_i}.
\]
It remains to show that $j_i=i$, which we do using orthogonality and
eigenvalue ordering. Let $i,\ell\in\{1,\dots,m\}$ with $i\neq\ell$.
Orthogonality of
$w_i^{\lambda_n}$ and $w_\ell^{\lambda_n}$ 
implies
\[
0
=\langle w_i^{\lambda_n},w_\ell^{\lambda_n}\rangle
=
\langle u_i^{\lambda_n},u_\ell^{\lambda_n}\rangle
+
\langle v_i^{\lambda_n},v_\ell^{\lambda_n}\rangle .
\]
Since $v_i^{\lambda_n},v_\ell^{\lambda_n}\to0$ and
$\|u_i^{\lambda_n}\|_2,\|u_\ell^{\lambda_n}\|_2\to1$, passing to the limit yields
$\langle\widetilde u_i,\widetilde u_\ell\rangle=0$. Hence
$j_1,\dots,j_m$ is a permutation of $\{1,\dots,m\}$.
On the other hand, by the ordering convention,
\[
\lambda_1(M_{\lambda_n})\ge\cdots\ge\lambda_m(M_{\lambda_n})
\quad\Longrightarrow\quad
\eta_1(\lambda_n)\ge\cdots\ge\eta_m(\lambda_n).
\]
Passing to the limit gives
$\mu_{j_1}\ge\cdots\ge\mu_{j_m}$. Because
$\mu_1>\cdots>\mu_m$ are strictly decreasing, the only possible permutation $(j_1,\dots,j_m)$ of $\{1,\dots,m\}$ for which $\mu_{j_1}\ge\cdots\ge\mu_{j_m}$
holds is the identity.
Therefore, for all $i\in\{1,\dots,m\}$
\[
\eta_i^\star=\mu_i,
\qquad
\widetilde u_i=\pm e_i.
\]
\smallskip
\noindent
\textbf{Step 6: Full convergence of eigenvalues and eigenvectors.}
Since the original sequence $\lambda_n\to\infty$ was arbitrary, every
convergent subsequence of $\eta_i(\lambda)$ 
has limit $\mu_i$.
A bounded
sequence whose convergent subsequences all share the same limit converges to
that limit, so for all $i\in\{1,\dots,m\}$
\[
\lambda_i(M_\lambda)-2E\lambda
=
\eta_i(\lambda)
\to
\mu_i
\qquad
\text{as}
\qquad
\lambda\to\infty.
\]
Similarly, every subsequential limit of $\widehat u_i^{\lambda}$ equals
$\pm e_i$. We can choose the sign of $w_i^\lambda$ so that
$\langle w_i^\lambda,s_i\rangle\ge 0$ rules out the $-e_i$ limit and
$\widehat u_i^{\lambda}\to e_i$. 
As $\|u_i^\lambda\|_2\to 1$ by~\eqref{eq:ui_limit}, this gives $u_i^\lambda\to e_i$. Since $v_i^\lambda\to 0$,
\[
w_i^{\lambda}
=
Su_i^\lambda+Zv_i^\lambda
\to
Se_i
=
s_i
\qquad
\text{as}
\qquad
\lambda\to\infty,
\]
which is the asserted direction-level convergence.
\end{proof}

\subsection{Proof of \texorpdfstring{\cref{thm:minimax_reconstruction}}{Theorem~\ref{thm:minimax_reconstruction}}}

We first show~\eqref{eq:worst_case_reconstruction_error_decomposition} and then use it to prove statements~(i) and~(ii).
Let $W\in\Ok$ and 
$\Pi_W := WW^\top$. By definition of $\mathcal C_\lambda$,
for every perturbed covariance $\Sigma'_e$, there exists $\Delta_e \in \Spp$ with
$0\preceq\Delta_e\preceq 2E\lambda\,\Pi_k^{(e)}$ such that 
$\Sigma'_e=\Sigma_e+\Delta_e$. 
Hence,
\[
\max_{(\Sigma'_1,\dots,\Sigma'_E)\in\mathcal C_\lambda}
\mathcal R(W;\Sigma'_1,\dots,\Sigma'_E)
=
\frac1E\sum_{e=1}^E
\left[
\Tr\bigl(\Sigma_e(I_p-\Pi_W)\bigr)
+
\max_{\substack{\Delta_e \in \Spp, \\ 0\preceq\Delta_e\preceq 2E\lambda\,\Pi_k^{(e)}}}
\Tr\bigl(\Delta_e(I_p-\Pi_W)\bigr)
\right].
\]
Applying the trace-monotonicity inequality $\Tr(AC)\le\Tr(BC)$ (valid for $A\preceq B$ and $C\succeq 0$; see, e.g., \citet[Example 2.24]{boyd2004convex}) with $A=\Delta_e$, $B=2E\lambda\,\Pi_k^{(e)}$, and $C=I_p-\Pi_W$, we obtain
\[
\Tr\bigl(\Delta_e(I_p-\Pi_W)\bigr)
\le
\Tr\bigl(2E\lambda\,\Pi_k^{(e)}(I_p-\Pi_W)\bigr),
\]
with equality
for $\Delta_e^\star=2E\lambda\,\Pi_k^{(e)}$. Therefore,
\begin{align*}
    \max_{(\Sigma'_1,\dots,\Sigma'_E)\in\mathcal C_\lambda}
\mathcal R(W;\Sigma'_1,\dots,\Sigma'_E)
&=
\frac1E\sum_{e=1}^E \Tr\bigl(\Sigma_e(I_p-\Pi_W)\bigr)
+
2\lambda\sum_{e=1}^E \Tr\bigl(\Pi_k^{(e)}(I_p-\Pi_W)\bigr) \\
&=\frac1E\sum_{e=1}^E \Tr\bigl(\Sigma_e(I_p-\Pi_W)\bigr)
+
\lambda\sum_{e=1}^E \|\Pi_W-\Pi_k^{(e)}\|_{\mathrm F}^2,
\end{align*}
where the second line follows
since $\Pi_W$ and $\Pi_k^{(e)}$ are rank-$k$ orthogonal projectors, so
\[
2\Tr\bigl(\Pi_k^{(e)}(I_p-\Pi_W)\bigr)
=2k-2\Tr(\Pi_W\Pi_k^{(e)})
=\Tr\bigl((\Pi_W-\Pi_k^{(e)})^2\bigr)
=\|\Pi_W-\Pi_k^{(e)}\|_{\mathrm F}^2.
\]
This proves~\eqref{eq:worst_case_reconstruction_error_decomposition}.

\emph{Part (i).}  A direct computation gives
$
\frac1E\sum_{e=1}^E \Tr\bigl(\Sigma_e(I_p-\Pi_W)\bigr)
=
\Tr(\barSigma)-\Tr(W^\top\barSigma W)
$
so minimizing the worst-case reconstruction error is equivalent to maximizing
\[
\Tr(W^\top\barSigma W)
-
\lambda\sum_{e=1}^E \|\Pi_W-\Pi_k^{(e)}\|_{\mathrm F}^2,
\]
which is Definition~\ref{def:anchorpca}.

\emph{Part (ii).} Using the fact that $\|\Pi_W-\Pi_k^{(e)}\|_{\mathrm F}^2 = 2k - 2\Tr(W^\top \Pi_k^{(e)} W)$ for all $e\in\cE$ and summing over $e$,
\[
\sum_{e=1}^E \|\Pi_W-\Pi_k^{(e)}\|_{\mathrm F}^2
=
2Ek-2E\Tr(W^\top\barPi W).
\]
Hence
\[
\mathcal B
=
\argmin_{\widetilde W\in\Ok}\sum_{e=1}^E
\|\Pi_{\widetilde W}-\Pi_k^{(e)}\|_{\mathrm F}^2
=
\argmax_{\widetilde W\in\Ok}\Tr(\widetilde W^\top\barPi \widetilde W).
\]
Combining this with $\mathcal R(W;\Sigma_1,\dots,\Sigma_E) = \Tr(\barSigma)-\Tr(W^\top\barSigma W)$, minimizing $\mathcal R$ over $\mathcal B$ amounts to maximizing $\Tr(W^\top\barSigma W)$ over the maximizers of $\Tr(W^\top\barPi W)$, which is Definition~\ref{def:anchorpca_infty}.

\begin{corollary}[Full-rank covariance perturbations]
\label{cor:full_rank_robustness}
Fix $\lambda\ge0$ and $\varepsilon:=(\varepsilon_1,\ldots,\varepsilon_E)$ such that
$\varepsilon_e\ge0$ for all $e\in\{1,\ldots, E\}$. 
Define
\[
\mathcal C_{\lambda,\varepsilon}
:=\Bigl\{(\Sigma'_e)_{e=1}^E:\;
0\preceq\Sigma'_e\preceq\Sigma_e+2E\lambda\Pi_k^{(e)}+\varepsilon_e I_p
\text{ for all }e\in\cE\Bigr\}.
\]
For all $W\in\Ok$, 
\begin{align}\label{eq:worst_case_error_relaxed}
&\max_{(\Sigma'_1,\ldots,\Sigma'_E)\in\mathcal C_{\lambda,\varepsilon}}
\mathcal R(W;\Sigma'_1,\ldots,\Sigma'_E) \\
&=\mathcal R(W;\Sigma_1,\ldots,\Sigma_E) +\lambda\sum_{e=1}^E\fnorm{\Pi_W-\Pi_k^{(e)}}^2
+\frac{p-k}{E}\sum_{e=1}^E\varepsilon_e.\notag
\end{align}
As the additional term in the worst-case error is constant in $W$,
the minimizers of~\eqref{eq:worst_case_error_relaxed} are the rank-$k$ \AnchorPCA{} solutions.
However, the
worst-case error increases with the noise levels $\varepsilon_e$.
\end{corollary}

\begin{proof}[Proof of \cref{cor:full_rank_robustness}]
Fix $W\in\Ok$. Since $I_p-\Pi_W\succeq0$, trace monotonicity (\citet[Example 2.24]{boyd2004convex}) gives for all $(\Sigma_e')_{e=1}^E\in\mathcal{C}_{\lambda, \epsilon}$,
\[
\Tr\bigl(\Sigma'_e(I_p-\Pi_W)\bigr)
\le\Tr\bigl((\Sigma_e+2E\lambda\Pi_k^{(e)}+\varepsilon_e I_p)
(I_p-\Pi_W)\bigr)
\quad\forall e\in\{1,\ldots,E\}.
\]
Equality holds for every domain at
$\Sigma'_e=\Sigma_e+2E\lambda\Pi_k^{(e)}+\varepsilon_e I_p$ for all $e\in\{1,\ldots,E\}$,
which are positive semidefinite and admissible. Averaging over domains and using
\mbox{$2\Tr(\Pi_k^{(e)}(I_p-\Pi_W))=\|\Pi_W-\Pi_k^{(e)}\|_F^2$}
and $\Tr(I_p-\Pi_W)=p-k$ gives the stated worst-case value.
Its added term does not depend on $W$, so the minimizers are unchanged.
\end{proof}

\begin{remark}
The class $\mathcal{C}_{\lambda, \varepsilon}$ in particular covers bounded covariance-estimation error. Suppose each
covariance $\Sigma_e+\Delta_e$, with
$0\preceq\Delta_e\preceq2E\lambda\Pi_k^{(e)}$, is further perturbed by a
symmetric error $A_e$ satisfying $\opnorm{A_e}\le\varepsilon_e$.
Since $A_e\preceq\varepsilon_e I_p$, we have
\[
\Sigma_e+\Delta_e+A_e\preceq
\Sigma_e+2E\lambda\Pi_k^{(e)}+\varepsilon_e I_p.
\]
Thus $(\Sigma_e+\Delta_e+A_e)_{e=1}^E$ belongs to
$\mathcal C_{\lambda,\varepsilon}$, provided each resulting covariance is
positive semidefinite.
\end{remark}

\subsection{Consistency of Anchor PCA \label{app:consistency}}

We now state the consistency results of the empirical \AnchorPCA{} estimator and of the
block-stabilized empirical \AnchorPCAinfty{} construction (\cref{subsec:finite_sample_implementation}), already stated in
\cref{subsec:consistency} informally. 
\Cref{cor:consistency_empirical_sstar_block} then shows that, when
$\Sstar\ne\{0\}$,
the first block returned by \AnchorPCAinfty{}
consistently estimates $\Sstar$.
Throughout, for all \(e\in\cE\), let  
\(x_{e,1},x_{e,2},\dots,x_{e,n_e}\) be i.i.d.\ copies of \(x_e\), and assume
$n := \min_{e\in \Ecal} n_e \to \infty$.

\begin{theorem}[Consistency of the empirical projector-agreement estimators]
\label{thm:consistency_empirical_methods}
Assume that, for all \(e\in\cE\),
$\EE{[x_e]} = 0$, 
$\EE{\|x_e\|_2^4}<\infty$,
 and $\lambda_k(\Sigma_e)>\lambda_{k+1}(\Sigma_e)$. Then the following statements hold.
\begin{enumerate}[leftmargin=1.75em]
\item[(i)]
Let \(\lambda\ge0\), and let \(\widehat{\Pi}^{\anchor,\lambda}_k\) be an
empirical rank-\(k\) \AnchorPCA{} solution projector.
If
$\lambda_k(M_\lambda)>\lambda_{k+1}(M_\lambda)$
then
\[
\opnorm{\widehat{\Pi}^{\anchor,\lambda}_k-\Pi^{\anchor,\lambda}_k}\to 0
\qquad
\text{almost surely as } n\to\infty.
\]
\item[(ii)]
Let $\widehat\Pi_k^{\anchor,\infty}$ be the block-stabilized
\AnchorPCAinfty{} estimator with tolerance $\mathrm{tol}_n>0$. 
Assume that there exists $\alpha \in (0, 1/2)$ and $c>0$ such that for all $n$, $\mathrm{tol}_n = cn^{-\alpha}$ 
and the population
rank-$k$ \AnchorPCAinfty{} solution projector is unique,
then
\[
\opnorm{\widehat\Pi_k^{\anchor,\infty}-\Pi_k^{\anchor,\infty}}\to0
\qquad\text{almost surely as } n\to\infty.
\]
\end{enumerate}
\end{theorem}

\begin{remark} 
The uniqueness assumption in part~(ii) concerns the population projector, i.e.\ the optimal subspace, not any particular basis. For example, it holds automatically if there is a $\barPi$-eigenvalue block boundary at rank-$k$.
Otherwise, it requires an eigengap condition~\eqref{eq:uniqueness_gap_not_at_block_boundary}.
\end{remark}

To prove \cref{thm:consistency_empirical_methods}, we use three auxiliary lemmas collected in \cref{app:additional_intermediate_results}:~\cref{lem:top_projector_perturbation}  transfers operator-norm perturbation bounds to leading eigenspace projectors,~\cref{lem:local_projector_polynomial_rate} gives polynomial almost-sure rates for \(\hatSigma\) and \(\hatPi\), and~\cref{lem:barPi_block_recovery} guarantees recovery of the population spectral blocks of \(\barPi\) by the empirical block clustering.

\begin{proof}[Proof of \cref{thm:consistency_empirical_methods}]

Choose any $\alpha_0\in(0,1/2)$. By
\cref{lem:local_projector_polynomial_rate},
\begin{equation}
\label{eq:empirical_average_convergence}
\opnorm{\hatSigma-\barSigma}\to0
\qquad\text{and}\qquad
\opnorm{\hatPi-\barPi}\to0
\qquad\text{almost surely.}
\end{equation}

\smallskip
\noindent
\textbf{Statement~(i): \AnchorPCA{}.} By \eqref{eq:empirical_average_convergence}, there is an event
$\Omega_0$ with $\PP(\Omega_0) = 1$ on which
$\opnorm{\widehat M_\lambda-M_\lambda}
\le
\opnorm{\hatSigma-\barSigma}
+
2E\lambda\,\opnorm{\hatPi-\barPi}
\to 0$
almost surely, where
$\widehat M_\lambda=\hatSigma+2E\lambda\,\hatPi,
M_\lambda=\barSigma+2E\lambda\,\barPi$.
Let 
$\gamma_k(M_\lambda):=\lambda_k(M_\lambda)-\lambda_{k+1}(M_\lambda)
$;
by assumption, $\gamma_k(M_\lambda) > 0$.
Fix a realization in $\Omega_0$.
Then $\opnorm{\widehat M_\lambda - M_\lambda} < \gamma_k(M_\lambda)/2$ for all $n \ge N$
(with $N$ depending on the realization), so
\cref{lem:top_projector_perturbation} (with $A = M_\lambda$,
$\widehat A = \widehat M_\lambda$, $d = k$) gives
\[
\opnorm{\widehat{\Pi}^{\anchor,\lambda}_k - \Pi^{\anchor,\lambda}_k}
\le \frac{2}{\gamma_k(M_\lambda)}\,\opnorm{\widehat M_\lambda - M_\lambda} \to 0.
\]
As this holds on $\Omega_0$, the convergence is almost sure.

\medskip
\noindent
\textbf{Statement~(ii): \AnchorPCAinfty{}.}
Recall that the block-stabilized construction of the solution to $\AnchorPCAinfty$ 
(\cref{subsec:finite_sample_implementation} ) scans the ordered eigenvalues of $\hatPi$ from largest to smallest to group them into empirical 
blocks: each empirical block starts at an eigenvalue and includes subsequent eigenvalues as long as they are at most $\mathrm{tol}_n$ smaller than the first eigenvalue of the block. 

Let $1\ge \rho_1>\cdots>\rho_L\ge0$ be the distinct eigenvalues of $\barPi$, and let $\Pi_\ell$ denote the orthogonal projector onto the eigenspace corresponding to $\rho_\ell$.
Let $d_\ell:=\operatorname{rank}(\Pi_\ell)$ and $c_\ell:=\sum_{j<\ell}d_j$.
Listing the
eigenvalues of $\barPi$ in nonincreasing order with multiplicity, $\rho_\ell$
occupies a block of $d_\ell$ consecutive indices $I_\ell:=\{c_\ell+1,\dots,c_\ell+d_\ell\}$, so that
$I_1,\dots,I_L$ partition $\{1,\dots,p\}$.
Let $\varepsilon_n:=\|\hatPi-\barPi\|_\mathrm{op}$ and $\xi_n:=\|{\hatSigma-\barSigma}\|_\mathrm{op}$.

\smallskip
\noindent
\textbf{Step 1: Match the indices of the empirical and population
blocks.}
By \cref{lem:local_projector_polynomial_rate}, with the same $\alpha$ as in $\mathrm{tol}_n=cn^{-\alpha}$, we have $\varepsilon_n=o(n^{-\alpha})$ and $\xi_n=o(n^{-\alpha})$ almost surely.
Since $\mathrm{tol}_n=cn^{-\alpha}$ with $c>0$, this implies $\varepsilon_n=o(\mathrm{tol}_n)$ almost surely.
We now work on an event $\Omega_1$ with $\PP(\Omega_1)=1$ on which these convergences hold.
Fix a realization in $\Omega_1$; every convergence in the remainder of the proof
is along this realization, and each ``for all sufficiently large $n$'' refers to
a realization-dependent threshold.
Then, for all sufficiently large  
$n$, 
\begin{equation} \label{eq:tol_n_bound}
    \varepsilon_n<\mathrm{tol}_n/4.
\end{equation}
Moreover, since $\mathrm{tol}_n$ decreases monotonically to $0$ deterministically,
for all sufficiently large $n$ we also have 
\begin{equation} \label{eq:gamma_bound}
    \mathrm{tol}_n<\gamma_{\barPi}/2, 
\end{equation}
where $\gamma_{\barPi}$ is the smallest gap between distinct eigenvalues of $\barPi$ as defined in \cref{lem:barPi_block_recovery}.
Hence \cref{lem:barPi_block_recovery} applies, and the block-stabilized construction in \cref{subsec:finite_sample_implementation} 
partitions the indices $1,\dots,p$ into empirical blocks that coincide with
$I_1,\dots,I_L$.
For each $\ell\in\{1,\dots,L\}$, let $\widehat\Pi_\ell$ denote the orthogonal projector onto the span of the eigenvectors of $\hatPi$
with indices in $I_\ell$,
i.e. onto $\widehat{\operatorname{Eig}}_\ell$ in the notation of~\cref{subsec:finite_sample_implementation};
$\widehat{\operatorname{Eig}}_\ell$ estimates 
$\Imop(\Pi_\ell)$,
and it has dimension $d_\ell$.

\smallskip
\noindent
\textbf{Step 2: Show consistency of each block projector.}
We show
$
\opnorm{\widehat\Pi_\ell-\Pi_\ell}\to0
$
for all $\ell\in\{1,\dots,L\}$.
If $L=1$, then $\barPi$ has a single eigenvalue, so $\Pi_1=I_p$.
For all sufficiently large $n$, the empirical block construction also returns a single block, so $\widehat\Pi_1=I_p=\Pi_1$, which proves the claim.
Assume now that $L\ge2$, and let $\ell\in\{1,\dots,L\}$.
Apply \cref{thm:davis_kahan} with $\Sigma=\barPi$, $\widehat\Sigma=\hatPi$, $r=c_\ell+1$, and $s=c_\ell+d_\ell$:
then $V$ spans $\Imop(\Pi_\ell)$, $\widehat V$ spans $\Imop(\widehat\Pi_\ell)$, and 
$[\lambda_s(\barPi),\lambda_r(\barPi)]=
\{\rho_\ell\}$.
It remains to lower bound the separation
$\delta$ between $\rho_\ell$ and the empirical eigenvalues outside $I_\ell$ (defined in \cref{thm:davis_kahan}).
By
Weyl's inequality \citep[Theorem~4.3.1]{horn2012matrix} in decreasing-order notation applied to $\barPi$ and $\hatPi-\barPi$ at indices $c_\ell$ and $c_\ell + d_\ell + 1$,
\[
\widehat\rho_{c_\ell}\ge\rho_{\ell-1}-\varepsilon_n
\qquad\text{and}\qquad
\widehat\rho_{c_\ell+d_\ell+1}\le\rho_{\ell+1}+\varepsilon_n,
\]
where for $\ell=1$ there is no eigenvalue above the block and for $\ell=L$ none
below, so only the relevant one-sided bound is used. In either case, the empirical
eigenvalues outside $I_\ell$ lie at distance at least $\gamma_{\barPi}-\varepsilon_n$
from $\rho_\ell$, so $\delta\ge\gamma_{\barPi}-\varepsilon_n > 0$, where the last inequality follows from $\varepsilon_n < \operatorname{tol}_n/2 < \gamma_{\barPi}/4$ using~\eqref{eq:tol_n_bound} and~\eqref{eq:gamma_bound}. Thus, we get
\[
\opnorm{\widehat\Pi_\ell-\Pi_\ell}
=\opnorm{\sin\Theta(\widehat V,V)}
\le\frac{\varepsilon_n}{\gamma_{\barPi}-\varepsilon_n}\to0,
\]
where we additionally used~\cref{prop:projector_distance_principal_angles} for the first equality as $\widehat\Pi_\ell$ and $\Pi_\ell$ have equal rank $d_\ell$ (Step 1).

\smallskip
\noindent
\textbf{Step 3: Assemble the rank-$k$ projector.}
Let $r\in\{1,\dots,L\}$ be the unique index such that $c_r<k\le c_r+d_r$, and let $q:=k-c_r$, so, in the population case, blocks $1,\dots,r-1$ are taken in full
and $q$ directions are taken from block $r$.

\emph{Case 1: $q=d_r$.} 
The first $r$
blocks are fully taken for the population solution.
By Steps 1a and 1b in~\cref{app:anchor_pca_infty_basis_convention},
the population \AnchorPCAinfty{} projector is $\Pi^{\anchor,\infty}_k=\sum_{\ell=1}^{r}\Pi_\ell$.
For all sufficiently large $n$, the empirical construction selects the first $r$ empirical blocks, so $\widehat\Pi^{\anchor,\infty}_k=\sum_{\ell=1}^{r}\widehat\Pi_\ell$.
By Step~2,
\[
\opnorm{\widehat\Pi^{\anchor,\infty}_k-\Pi^{\anchor,\infty}_k}
\le
\sum_{\ell=1}^{r}\opnorm{\widehat\Pi_\ell-\Pi_\ell}
\to0.
\]

\emph{Case 2: $q<d_r$.}
Blocks $1,\dots,r-1$ are taken in full,
and the secondary variance criterion selects $q$ directions from block $r$.
Let $B_r:=\Pi_r\barSigma\Pi_r$ and $\widehat B_r:=\widehat\Pi_r\hatSigma\widehat\Pi_r$,
and let $Q_r$ and $\widehat Q_r$ denote the orthogonal projectors onto the spans of the top $q$ eigenvectors of $B_r$ and $\widehat B_r$, respectively.
We now show that $\widehat B_r\to B_r$ in operator norm.
Note that
\[
\widehat B_r-B_r
=
\widehat\Pi_r(\hatSigma-\barSigma)\widehat\Pi_r
+
(\widehat\Pi_r-\Pi_r)\barSigma\widehat\Pi_r
+
\Pi_r\barSigma(\widehat\Pi_r-\Pi_r),
\]
together with $\opnorm{\widehat\Pi_r}=\opnorm{\Pi_r}=1$ and~\eqref{eq:empirical_average_convergence}
gives
\begin{equation} \label{eq:B_r_convergence}
    \|{\widehat B_r-B_r}\|_\mathrm{op}\le \xi_n+2\|{\barSigma}\|_\mathrm{op}\|{\widehat\Pi_r-\Pi_r}\|_\mathrm{op}\to0.
\end{equation}
By Steps 1a and 1b in \cref{app:anchor_pca_infty_basis_convention}, the population \AnchorPCAinfty{} projector is $\Pi^{\anchor,\infty}_k=\sum_{\ell=1}^{r-1}\Pi_\ell + Q_r$, where the first sum is a
fixed projector independent of the selection within block $r$. 
Hence
$Q_r=\Pi^{\anchor,\infty}_k-\sum_{\ell=1}^{r-1}\Pi_\ell$ is determined by
$\Pi^{\anchor,\infty}_k$, so the uniqueness assumption forces the top-$q$
eigenspace of $B_r$ to be unique; equivalently, $B_r$ has a positive eigengap at rank $q$:
\begin{equation}
\label{eq:uniqueness_gap_not_at_block_boundary}
\eta_r:=\lambda_q(B_r)-\lambda_{q+1}(B_r)>0.
\end{equation}
Thus, for all sufficiently large $n$, $\|{\widehat B_r-B_r}\|_\mathrm{op}<\eta_r/2$.
Applying \cref{lem:top_projector_perturbation} with $A=B_r$, $\widehat A=\widehat B_r$, and $d=q$, and then using~\eqref{eq:B_r_convergence} gives
\[
\opnorm{\widehat Q_r-Q_r}
\le
\frac{2\|{\widehat B_r-B_r}\|_\mathrm{op}}{\eta_r}
\to0.
\]
For all sufficiently large $n$, the empirical and population \AnchorPCAinfty{} projectors can be written as
\[
\widehat\Pi^{\anchor,\infty}_k
=
\sum_{\ell=1}^{r-1}\widehat\Pi_\ell+\widehat Q_r,
\qquad
\Pi^{\anchor,\infty}_k
=
\sum_{\ell=1}^{r-1}\Pi_\ell+Q_r.
\]
Hence
\[
\opnorm{\widehat\Pi^{\anchor,\infty}_k-\Pi^{\anchor,\infty}_k}
\le
\sum_{\ell=1}^{r-1}\opnorm{\widehat\Pi_\ell-\Pi_\ell}
+
\opnorm{\widehat Q_r-Q_r}
\to0.
\]
All displayed convergences hold for the fixed realization in $\Omega_1$. Since
$\PP(\Omega_1)=1$, it follows that
$\opnorm{\widehat\Pi^{\anchor,\infty}_k-\Pi^{\anchor,\infty}_k}\to0$ almost surely.
\end{proof}

\begin{corollary}[Consistency of the empirical invariant block]
\label{cor:consistency_empirical_sstar_block}
Assume the moment and eigengap conditions of \cref{thm:consistency_empirical_methods} and a tolerance $\mathrm{tol}_n=cn^{-\alpha}$ with $\alpha\in(0,1/2)$, $c>0$ (as in part~(ii)); suppose further that $\Sstar\ne{0}$. Let $\widehat\Pi_1$ denote the first empirical block returned by \AnchorPCAinfty{}, i.e., the projector onto
$\widehat{\operatorname{Eig}}_1$ 
Then
\[
\opnorm{\widehat\Pi_1-\Pi_{\Sstar}}\to0
\qquad\text{almost surely.}
\]
Equivalently, $\widehat{\operatorname{Eig}}_1$ consistently estimates
$\Sstar$.
\end{corollary}
\begin{proof}
The eigenvalue-one block of $\barPi$ is $\Sstar$, so $\Pi_1=\Pi_{\Sstar}$. Step~2
in the proof of \cref{thm:consistency_empirical_methods}~\emph{(ii)} gives
$\opnorm{\widehat\Pi_\ell-\Pi_\ell}\to0$ almost surely for every $\ell$; the case
$\ell=1$ is the claim.
\end{proof}

\section{Details on experiments} \label{app:details_experiments}

This appendix provides additional details for the motivating example in \cref{subsec:motivating_example_4d}, the synthetic $\Sstar$ recovery simulations in \cref{subsec:exp_sstar_recovery}, and the gas-sensor drift application in \cref{subsec:exp_gas_sensor}. Code to reproduce all simulations and figures is publicly available
at \url{https://github.com/benesei/anchor-pca}.

\subsection{Details of the motivating example}
\label{app:motivating_example_4d_details}

This appendix records the explicit construction used in the motivating example (\cref{subsec:motivating_example_4d}) which is also used for the perturbation-path plot \cref{fig:motivating_example_4d_perturbation_path}.

\paragraph{Construction.}
Recalling \cref{subsec:motivating_example_4d}, let $c_1,\ldots,c_4$ be the canonical basis vectors of $\R^4$, and set
$
a:=c_1,
b:=c_2,
u:=c_3.
$
We further define two directions in the $c_3$--$c_4$ plane,
\[
v:=\cos(50^\circ)c_3+\sin(50^\circ)c_4,
\qquad
w:=\cos(100^\circ)c_3+\sin(100^\circ)c_4,
\]
and let $u^\perp,v^\perp,w^\perp$  denote the unit directions in the $c_3$--$c_4$ plane that are orthogonal to $u,v,w$, respectively. The population covariance matrices of the three domains are then defined as
\begin{align*}
    \Sigma_1
    &=
    220\,uu^\top + 140\,aa^\top + 90\,bb^\top + 25\,u^\perp u^{\perp\top},\\
    \Sigma_2
    &=
    120\,aa^\top + 90\,vv^\top + 70\,bb^\top + 10\,v^\perp v^{\perp\top},\\
    \Sigma_3
    &=
    320\,ww^\top + 120\,bb^\top + 80\,aa^\top + 10\,w^\perp w^{\perp\top};
\end{align*}
these are shown in Table~\ref{tab:motivating_example_4d}.

\paragraph{Recovered subspaces.}
The solution of
\poolPCA{} spans $\operatorname{span}(a,c_3,c_4)$ and therefore drops $b$.
Both \AnchorPCA{25} and \AnchorPCAinfty{} recover $\operatorname{span}(a,b)$.
Their difference lies only in the third
direction:
\AnchorPCAinfty{} selects 
$v$,
and
\AnchorPCA{25} selects $q =0.246 c_3 + 0.969 c_4$, a direction capturing slightly more variance on average than $v$ (126.47 vs. 111.22), but is less invariant ($q^\top \barPi q = 0.57 < 0.61 = v^\top \barPi v$).

\subsection{Details on the perturbation-path illustration}
\label{subsec:perturbation_path_illustration}

To visualize \cref{thm:minimax_reconstruction}, we reuse the population motivating example and plot, for several methods (fit on the population source covariances), the worst-case average reconstruction error over $\mathcal{C}_\lambda$ as the perturbation strength $\rho:=2E\lambda$ increases. By \cref{eq:worst_case_reconstruction_error_decomposition}, for any rank-$k$ projector $\Pi$, this worst-case error is
\[
\frac1E\sum_{e=1}^E \Tr\bigl((I_p-\Pi)\Sigma_e\bigr)
+
\frac{\rho}{2E}\sum_{e=1}^E \fnorm{\Pi-\Pi_k^{(e)}}^2,
\]
so each method traces a line in $\rho$: the intercept is its average reconstruction error on the source covariances, and the slope is its disagreement with the local top-$k$ projectors. The oracle minimizer at perturbation strength $\rho$ is $\AnchorPCA{\rho/(2E)}$. We plot these lines for \poolPCA{}, \AnchorPCA{25}, \AnchorPCAinfty{}, and the wcPCA baselines \texttt{maxRCS}, \texttt{maxRegret}, and \texttt{norm-maxRegret} from~\citep{fries2026worstcaselowrankapproximations}.

\cref{fig:motivating_example_4d_perturbation_path} is 
an
unsupervised
analog of the perturbation-path plot in Anchor
Regression~\citep{rothenhausler2020anchor}: \poolPCA{} is optimal only near
$\rho=0$, a fixed finite-penalty \AnchorPCA{} is optimal over a broad middle
range of perturbations centered around its penalty, and \AnchorPCAinfty{} is the right
choice once robustness to top-$k$ variance inflation matters more than reconstruction accuracy. The figure additionally shows that, in this example, the wcPCA baselines \texttt{norm-maxRegret} and \texttt{maxRCS} perform best only
on short transition windows near the origin, though they target a different form of distribution shift: shifts within the convex hull of the source covariances.

\subsection{Details on the simulations 
}
\label{app:random_subspace_sstar_recovery}

We now describe the details of the experiment 
described in \cref{subsec:exp_sstar_recovery}.

\subsubsection{Data-generating process} \label{sec:dgp}

\paragraph{Notation.}
We write $E=|\Ecal|$ for the number of domains, $p$ for the number of covariates. Each domain $e\in\cE$ is associated with a covariance $\Sigma_e$, whose construction we detail in the following paragraphs. Let $k$ 
be the target rank for the low-dimensional representation, and $m=\dim(\Sstar)$ be the dimension
of the invariant subspace $\Sstar$. Let
\[
    q:=p-k,
    \qquad
    d:=p-m.
\]
So, $q$ is the dimension of
the complement of each top-$k$ subspace,
and $d$ is the dimension of $\Sstar^\perp$.

Throughout this section, a \emph{configuration} corresponds to one choice of
$(E,p,k,m)$; for a given
configuration and eigenvalue regime (see~\cref{para:eigenval_regimes} below), a \emph{distribution draw} is a
multi-domain distributional instance, namely the full collection of covariances $(\Sigma_e)_{e\in\cE}$ with its own $\Sstar$ (if we talk about multiple distribution draws, we assume they are chosen i.i.d.); a \emph{sample} is one finite dataset drawn from a fixed distribution draw; and individual variables $X_{e,i}$ are called observations or data points.

\paragraph{Generating population covariance matrices.}
Each population covariance is constructed so that the 
top-$k$ eigenspaces of the $E$ covariances
have a nontrivial maximal intersection 
of dimension $m$. We build this in three steps.
\begin{enumerate}
    \item \emph{Invariant subspace.} We draw a Haar orthogonal matrix $U=(u_1,\ldots,u_p)\in\R^{p\times p}$ and split its columns into
    \[
        S:=(u_1,\ldots,u_m),
        \qquad
        R:=(u_{m+1},\ldots,u_p),
    \]
    so that $S$ is an orthonormal basis of $\Sstar:=\Imop(S)$ and $R$ is an orthonormal basis of $\Sstar^\perp$. The remaining construction works in $R$-coordinates inside $\Sstar^\perp$.
    \item \emph{Bottom subspace per domain.}
    For each domain $e\in\cE$, we draw a bottom subspace of dimension $q$ inside $\Sstar^\perp$: sample a $d\times q$ matrix with i.i.d. standard Gaussian entries, take the $Q$ factor of its reduced QR decomposition to obtain an orthonormal matrix $H_e\in\R^{d\times q}$, and set $B_e:=R H_e$. If the resulting collection $(B_1,\ldots,B_E)$ does not jointly span $\Sstar^\perp$, we discard this collection and redraw all $E$ bottom subspaces until
    \[
        \dim \operatorname{span}(B_1,\ldots,B_E)=p-m,
    \]
    which is what makes the intersection of the top-$k$ subspaces equal to $\Sstar$ exactly, rather than something larger.
    \item\emph{Top subspace per domain.}
    The remaining $k-m$ top directions in domain $e\in\cE$ lie in the complement of $B_e$ inside $\Sstar^\perp$. Let $G_e\in\R^{d\times(k-m)}$ be an orthonormal basis of $\ker(H_e^\top)$ and set $C_e:=R G_e$. The top-$k$ eigenspace for all $e\in\cE$ is then
    \[
        \mathcal U_e
        :=
        \Imop(S,C_e)
        =
        \Sstar \oplus \Imop(C_e).
    \]
    By construction, $\bigcap_{e\in\cE} \mathcal U_{e} = \bigcap_{e\in\cE} \Imop(B_e)^\perp = \big(\bigcup_{e\in\cE} \Imop(B_e)\big)^\perp = (\Sstar^\perp)^\perp = \Sstar$.
\end{enumerate}
For all $e\in\cE$, given the orthonormal basis $(S,C_e,B_e)$, we define the covariance matrix as
\[
    \Sigma_e
    :=
    S\operatorname{diag}(\theta_{e,\star})S^\top
    +
    C_e\operatorname{diag}(\theta_{e,\mathrm{env}})C_e^\top
    +
    B_e\operatorname{diag}(\eta_e)B_e^\top,
\]
where $\theta_{e,\star}\in\R^m$ are the eigenvalues on the invariant directions, $\theta_{e,\mathrm{env}}\in\R^{k-m}$ on the domain-specific top directions, and $\eta_e\in\R^q$ on the bottom directions. The eigenvalue regimes below ensure a local $k$ versus $k+1$ eigengap in every domain.

\paragraph{Easy and hard eigenvalue regimes.} \label{para:eigenval_regimes}
We use the same subspace construction with two different choices of eigenvalues. In the \emph{easy} regime, the invariant and domain-specific top eigenvalues are well separated from the bottom eigenvalues:
\[
    \theta_{e,\star,j}\sim U[5,8],
    \qquad
    \theta_{e,\mathrm{env},j}\sim U[5,8],
    \qquad
    \eta_{e,j}\sim U[0.5,3].
\]
In the \emph{hard} regime, the invariant directions still sit in each top-$k$ subspace but carry less variance than the domain-specific top directions:
\[
    \theta_{e,\star,j}\sim U[3.2,4.2],
    \qquad
    \theta_{e,\mathrm{env},j}\sim U[6,9],
    \qquad
    \eta_{e,j}\sim U[0.5,2.5].
\]
All distribution draws are generated with master seed $42$. 
For each
configuration $(p,k,m,E)$ and each eigenvalue regime, we generate $100$ independent
distribution draws; $\Sstar$ is redrawn for every replicate.

\paragraph{Gaussian sampling.}
Given a distribution draw and sample size $N$, we sample $N$ data points independently,
\[
    X_{e,i}\sim \mathcal N(0,\Sigma_e),
    \qquad
    i\in\{1,\ldots,N\}, \qquad e\in\cE.
\]
For each distribution draw and sample size $N$, we generate $20$ independent samples, each consisting of $N$ observations per domain. 
The standard sample-size grid is
\[
    N\in\{50,100,200,500,1000,2000,5000\}.
\]
For the small-$E$ 
configuration described below, we extend the grid to include $\{10000,\allowbreak 30000,\allowbreak 100000\}$.

\paragraph{Non-Gaussian covariance-preserving sampling.} \label{sec:gaussian_mixture_def}
As a robustness check, we also use
a symmetric two-component Gaussian mixture with the same mean and covariance as
the Gaussian model. For each domain, let $r_e$ be the leading eigenvector
of $\Sigma_e$ and set $\mu_e=0.75\sqrt{\lambda_{\max}(\Sigma_e)}\,r_e$; if
needed, this vector is shrunk to ensure that $\Sigma_e-\mu_e\mu_e^\top$ is
positive semidefinite. We then sample for $i\in \{1,\ldots, N\}, e\in\cE,$
\[
    X_{e,i}=Z_{e,i}+\xi_{e,i}\mu_e,
    \qquad
    Z_{e,i}\sim\mathcal N(0,\Sigma_e-\mu_e\mu_e^\top),
    \qquad
    \mathbb P(\xi_{e,i}=1)=\mathbb P(\xi_{e,i}=-1)=1/2.
\]
Thus
\[
    \mathbb E X_{e,i}=0,
    \qquad
    \operatorname{Cov}(X_{e,i})=\Sigma_e.
\]
The mixture preserves the population eigenspaces but is not Gaussian.

\subsubsection{Estimators}
In the experiments, we consider \AnchorPCAinfty{} and $\texttt{FindS}_\star$.
\AnchorPCAinfty{} is fit with the default block tolerance of the package\footnote{\url{https://github.com/benesei/anchor-pca}}.
$\mathrm{tol}_N:=\min\{0.05,\;0.5\,n_{\min}^{-0.4}\}.$
In these experiments, each domain has the same sample size, so
$n_{\min}=N$. The first empirical block contains all eigenvalues
$\widehat\rho_j$ satisfying
$\widehat\rho_1-\widehat\rho_j\le \mathrm{tol}_N$,
where $\widehat\rho_i$ is the $i$th largest eigenvalue of $\hatPi$.
Thus the oracle block corresponding to $\Sstar$ is separated whenever
$\widehat\rho_1-\widehat\rho_m
    \le \mathrm{tol}_N
    <
    \widehat\rho_1-\widehat\rho_{m+1}$.
$\texttt{FindS}_\star$ is used with significance level $\alpha=0.05$.

\subsubsection{Recovery configurations} \label{sec:configurations}

We consider three configurations $(E,p,k,m)$.
\begin{enumerate}
    \item \emph{Main random-subspace configuration:} $(E,p,k,m)=(5,10,5,2)$, shown in the main text.
    \item \emph{Small-$m$ configuration:} $(E,p,k,m)=(5,8,5,1)$ shown in \cref{fig:infty_sstar_recovery_g5_p8_k5_m1}. 
    \item \emph{Small-$E$ configuration:} $(E,p,k,m)=(2,8,5,2)$ shown in \cref{fig:infty_sstar_recovery_g2_p8_k5_m2}. 
    This configuration uses the
    minimal feasible invariant dimension. Indeed, since each domain has a
    $q=p-k=3$ dimensional bottom space and the bottom spaces must span
    $\Sstar^\perp$, feasibility requires
    \[
        p-m \le E q,
        \qquad\text{equivalently}\qquad
        m\ge p-E(p-k)=2 .
    \]
\end{enumerate}

\subsubsection{Recovering the invariant subspace \texorpdfstring{$\Sstar$}{S*}} \label{app:sstar_recovery_exp}
We now repeat the experiments of \cref{fig:infty_sstar_recovery_main}
for non-Gaussian data, small-$m$, and small-$E$ configurations.

\paragraph{Performance quantities and aggregation.}
For each distribution draw and sample size $N$, we compute two quantities over the 20 independent samples: the dimension-recovery indicator, equal to 1 if the estimated first-block dimension equals the true $m$, and the operator-norm projector error for the estimated invariant subspace. Within each distribution draw, we average the 20 dimension-recovery indicators and take the median of the 20 projector errors. The displayed dimension-recovery probability is the mean of these distribution-draw-level averages across distribution draws, while the displayed subspace error is the median of the distribution-draw-level median projector errors. Shaded bands, when shown, indicate the $10\%\text{--}90\%$ quantiles across distribution draws.

\paragraph{Agreement-separation gap of the configurations.}
For a distribution draw, 
we define the agreement-separation gap as
\[
    \gamma_\Pi
    :=
    \rho_m(\barPi)-\rho_{m+1}(\barPi)
    =
    \rho_1(\barPi)-\rho_{m+1}(\barPi) =
    1-\rho_{m+1}(\barPi),
\]
where $\rho_i(\barPi)$ denotes the $i$th largest eigenvalue of $\barPi$ and the second and third equality use $\rho_1(\barPi)=\cdots=\rho_m(\barPi)=1$ on $\Sstar$. The gap $\gamma_\Pi$ measures how cleanly $\Sstar$ separates from the rest of the agreement spectrum and reflects one aspect of difficulty of the configuration. 
Pooling the
$100$ easy and $100$ hard distribution draws for each configuration, the median
$\gamma_\Pi$ values, with $10$--$90\%$ ranges, are
\[
\begin{array}{c|c}
\text{configuration} & \text{median }\gamma_\Pi\;(10\%\text{--}90\%)\\
\hline
\text{main random-subspace} & 0.277\;(0.212\text{--}0.343)\\
\text{small-}m & 0.119\;(0.072\text{--}0.184)\\
\text{small-}E & 9.8\cdot10^{-3}\;(1.8\cdot10^{-4}\text{--}5.5\cdot10^{-2})
\end{array}
\]
The gap decreases across the three configurations, illustrating increased hardness. For the small-$E$ configuration, the agreement gap is often very small, so this configuration is best
viewed as a stress test rather than a typical regime.

At the empirical level, we analogously track
$\widehat\gamma_{\Pi}:=\rho_1(\hatPi)-\rho_{m+1}(\hatPi)$. This is the
quantity relevant for the block-stabilized estimator: the first empirical block has correct dimension $m$ 
if
\[
    \rho_1(\hatPi)-\rho_m(\hatPi)
    \le \mathrm{tol}_N
    <
    \rho_1(\hatPi)-\rho_{m+1}(\hatPi).
\]

\paragraph{Non-Gaussianity.}
\Cref{fig:infty_sstar_recovery_main_mixture} repeats the 
main-text experiment from \cref{fig:infty_sstar_recovery_main} ($E=5$, $p=10$, $k=5$, $m=2$),
but draws data from the covariance-preserving Gaussian mixtures introduced in \cref{sec:gaussian_mixture_def} instead of Gaussians. 
The recovery curves are close to the Gaussian ones in \cref{fig:infty_sstar_recovery_main}. This is in line with \cref{subsec:consistency}, which does not require Gaussianity for \AnchorPCAinfty{} under its moment assumptions.
The dashed
curves show performance of $\texttt{FindS}_\star$ and are included to test its robustness to violations of its Gaussian data assumption.
\begin{figure}[tbp]
    \centering
    \includegraphics[width=1\linewidth]{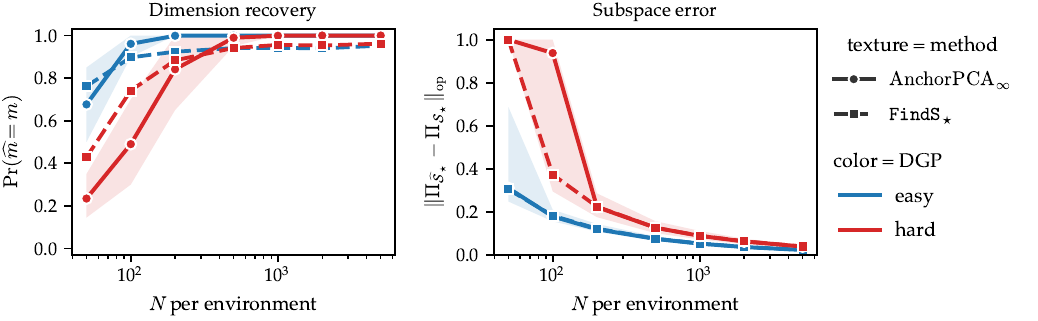}
    \caption{\textbf{Main random-subspace configuration with Gaussian mixtures.}
    Same distribution draws and aggregation as in
    \Cref{fig:infty_sstar_recovery_main}, but observations are drawn from a
    symmetric two-component Gaussian mixture preserving each covariance
    $\Sigma_e$. 
    The panels show 
    the fraction of samples for which $m$ was estimated correctly
    and median projector error. Solid lines use the first
    grouped eigenspace from \AnchorPCAinfty{}, while dashed lines show $\texttt{FindS}_\star$ using significance level $0.05$.
    Shaded bands are $10$--$90\%$ quantiles across distribution draws and are shown for \AnchorPCAinfty{}. 
    As suggested by the consistency results (\Cref{subsec:consistency}), for \AnchorPCAinfty{} the estimated recovery probability converges to one and the subspace error to zero and, for $\texttt{FindS}_\star$, 
    the estimated recovery probability
    approaches $0.95$.
    The close agreement with the Gaussian curves in \cref{fig:infty_sstar_recovery_main} suggests that, in this configuration, the \AnchorPCAinfty{} block-recovery mechanism and $\texttt{FindS}_\star$ are robust to this covariance-preserving non-Gaussian sampling perturbation.}
    \label{fig:infty_sstar_recovery_main_mixture}
\end{figure}

\paragraph{Small-$m$ configuration.}
\Cref{fig:infty_sstar_recovery_g5_p8_k5_m1} shows Gaussian results for the
small-$m$ configuration ($E=5$, $p=8$, $k=5$, $m=1$). This configuration has a
smaller population agreement-separation gap than the main configuration, but the
finite-sample recovery curves are comparable on the standard grid. In both
eigenvalue regimes, \AnchorPCAinfty{} reaches correct first-block dimension
recovery by the largest standard-grid sample sizes.
\begin{figure}[tbp]
    \centering
    \includegraphics[width=1\linewidth]{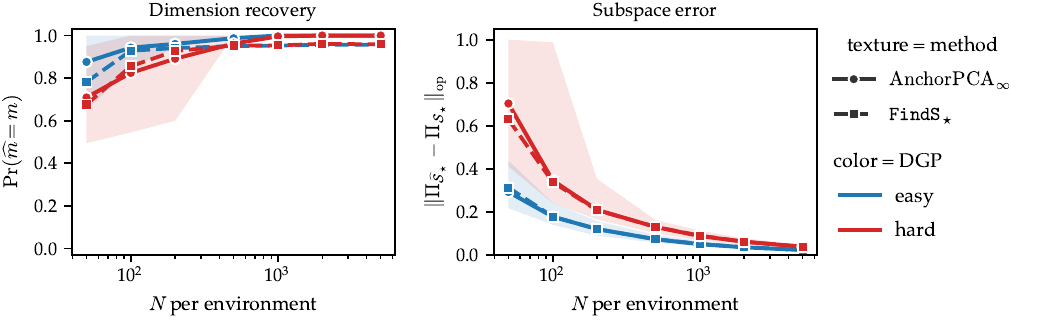}
    \caption{\textbf{Small-$m$ configuration.}
    Plots show the same as \Cref{fig:infty_sstar_recovery_main_mixture} for the small-$m$ configuration. Again, 
    the results are as suggested by \Cref{subsec:consistency}: e.g., in the left plot, 
    the errors of \AnchorPCAinfty{} and $\texttt{FindS}_\star$ approach zero and 0.05, respectively.}

    \label{fig:infty_sstar_recovery_g5_p8_k5_m1}
\end{figure}

\paragraph{Small-$E$ configuration.}
For the small-$E$ configuration ($E=2$, $p=8$, $k=5$, $m=2$), 
we use the extended sample-size grid up to $N=100000$. This configuration has a much smaller population agreement-separation gap than the other configurations, 
so we expect it to be a particularly difficult case.
\Cref{fig:infty_sstar_recovery_g2_p8_k5_m2} 
shows the same quantities as \Cref{fig:infty_sstar_recovery_main_mixture} but also shows
the
tolerance diagnostic.
In the top row, the first-block dimension is correct at $N=100000$ in about $0.68$ of easy-regime runs and $0.57$ of hard-regime runs. $\texttt{FindS}_\star$ improves dimension recovery in the easy regime, reaching about $0.96$ correct recovery, but in the hard regime it tends to overestimate the dimension, with median $\widehat m=3$. 
Intuitively,
this is indeed a difficult case for $\texttt{FindS}_\star$ because it estimates $m$
indirectly, by testing $r=\operatorname{dim} (\operatorname{span} \{I_p - \Pi_1^{(k)}, I_p - \Pi_2^{(k)}\}) = \operatorname{dim}(\SstarPerp) = p-m$ for all $p-k$ to $p-m_\mathrm{min}$; see~\cref{app:wald-test}. Thus, since $p-m = 6$ and $\operatorname{rank}(I_p-\Pi_e^{(k)}) = 3$, 
this configuration has the smallest possible $m=m_\mathrm{min} = 2$; equivalently, $r = p-m$ is maximal, so the correct dimension
$m=2$ is reached only if the test rejects all smaller candidate dimensions
$r=3,4,5$, leading to a small empirical power of the test. If sampling noise makes it fail to reject
$r=5$, it stops there and returns $\widehat m=p-5=3$. \AnchorPCAinfty{} uses a different criterion: it thresholds the empirical eigenvalues directly using $\operatorname{tol}_N$, so it can still return $m=2$ whenever the first two eigenvalues form a separated block, even when sampling noise makes $\SstarPerp$ appear five dimensional. The bottom row 
of \Cref{fig:infty_sstar_recovery_g2_p8_k5_m2}
illustrates why the small-$E$ configuration is also hard for $\AnchorPCAinfty{}$: 
$
\widehat\gamma_\Pi$ is much smaller than in the main configuration and is smaller than the default block tolerance over much of the extended grid.
We hypothesize that the performance of
\AnchorPCAinfty{}
could be improved by a more involved choice of the block tolerance. 
\begin{figure}[tbp]
    \centering
    \includegraphics[width=1\linewidth]{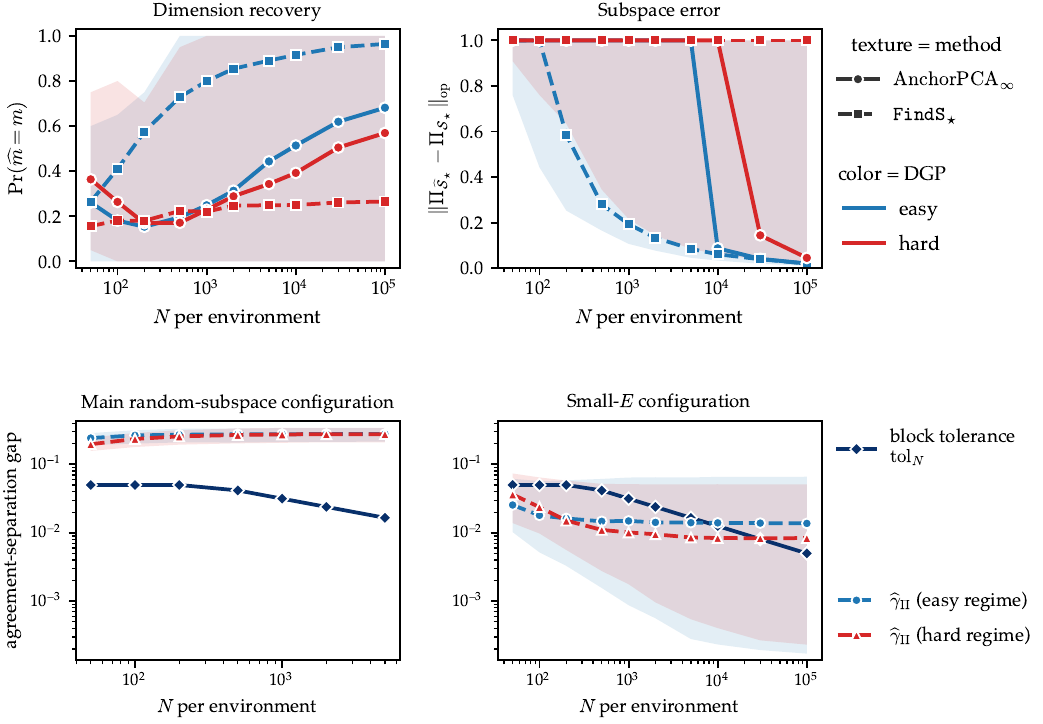}
    \caption{\textbf{Small-$E$ configuration and agreement-separation
    gap. 
    }
    The top row shows the same as \Cref{fig:infty_sstar_recovery_main_mixture} for the small-$E$ configuration. 
    As expected (see the discussion around $\gamma_\Pi$), this is a particularly hard case; indeed, the error of \AnchorPCAinfty{} decreases but this time more slowly. Even for large $N$, some of the distribution draws are difficult in that for none of the $20$ samples the estimated $m$ is correct, which results in wide shaded areas. 
    The 
    bottom row (log scale)
    compares the default block tolerance $\mathrm{tol}_N$ (dark blue) with the median $\widehat\gamma_\Pi$ 
    in the easy and hard regimes for the main random-subspace configuration (left) and the small-$E$ configuration (right). 
    The plot shows that the small-$E$ setting is difficult because
$\widehat\gamma_\Pi$ is 
much smaller than in the main configuration and%
, for many of the considered $N$, is smaller than $\mathrm{tol}_N$.
    }
    \label{fig:infty_sstar_recovery_g2_p8_k5_m2}
\end{figure}

\subsubsection{Explained variance on an unseen domain.}
\label{app:unseen_target_ev}

For the simulations, we additionally measure explained variance on an unseen domain.
For each of the easy and hard settings, we generate $50$ independent
distribution draws, each with six domains sharing the same $m=2$-dimensional
invariant subspace $\Sstar$ but otherwise having different covariances.
Five domains are used as sources and the sixth is held out for evaluation.
We use the covariance-preserving Gaussian mixtures from
\cref{sec:gaussian_mixture_def}, with $p=10$ and $k=5$. We draw $500$
observations per source domain and repeat the source-sample fit five times
per distribution draw. We evaluate rank-$k$ EV using the target population
covariance. 

To assess recovery of $\Sstar$, we also measure the operator-norm
distance between the projector onto each method's first $m=2$ directions
and $\Pi_{\Sstar}$. Here we use the known dimension $m$, whereas the recovery
experiment above estimates the dimension of the first agreement block.
For both metrics, we first average over the five datasets within each
distribution draw and the $50$ draws in each setting.
\Cref{tab:unseen_target_ev} shows that the two methods retain nearly identical
target EV using all $k=5$ directions, while the first $m=2$ directions of
\AnchorPCAinfty{} recover $\Sstar$ more accurately than those of \poolPCA{}.

\begin{table}[tbp]
    \centering
    \caption{
    Evaluation of Gaussian mixtures on
    an unseen domain.
    Target EV evaluates all $k=5$ directions using the target population
    covariance (higher is better). Projector error measures recovery of $\Sstar$ using each
    method's first $m=2$ directions (lower is better). Results are averaged
    over five datasets per distribution draw and $50$ draws per setting.}
    \label{tab:unseen_target_ev}
    \begin{tabular}{llrr}
    \toprule
    Setting & Method & Target EV (\%) & Projector error \\
    \midrule
    Easy & \poolPCA{} & 57.81 & 0.226 \\
    Easy & \AnchorPCAinfty{} & 57.75 & 0.078 \\
    Hard & \poolPCA{} & 49.44 & 0.999 \\
    Hard & \AnchorPCAinfty{} & 49.33 & 0.133 \\
    \bottomrule
    \end{tabular}
\end{table}

\subsection{Gas sensor array drift details and additional analyses}
\label{app:gas_sensor_drift_details}

\paragraph{Data and domains.}
The UCI gas-sensor drift archive contains $13{,}910$ measurements from $16$
metal-oxide chemical sensors exposed to six gases at multiple concentration
levels \citep{gas_sensor_array_drift,fonollosa2015chemicalgas}, see \Cref{fig:gas_sensor_class_composition}. Each
measurement comes from a $16$-channel metal-oxide gas sensor array and is summarized by eight features per channel, giving $p=128$ features in total. Two are steady-state features: the maximal resistance change $DR$ with respect to the baseline, and a normalized version of $DR$, obtained by dividing $DR$ by the sensor resistance measured under gas exposure. The remaining six are transient features computed from rising and decaying exponential moving averages of the signal.
The archive is organized into ten temporal batches: B1 contains months 1--2,
B2 months 3,4,8,9,10, B3 months 11--13, B4 months 14--15, B5 month 16,
B6 months 17--20, B7 month 21, B8 months 22--23, B9 months 24 and 30, and
B10 month 36. We use these batches as domains. 
Data are available from the UCI Machine Learning Repository
\citep{gas_sensor_array_drift} at
\url{https://doi.org/10.24432/C5MK6M} 
and are licensed under CC BY 4.0.

\begin{figure}[tbp]
    \centering
    \includegraphics[width=1\linewidth]{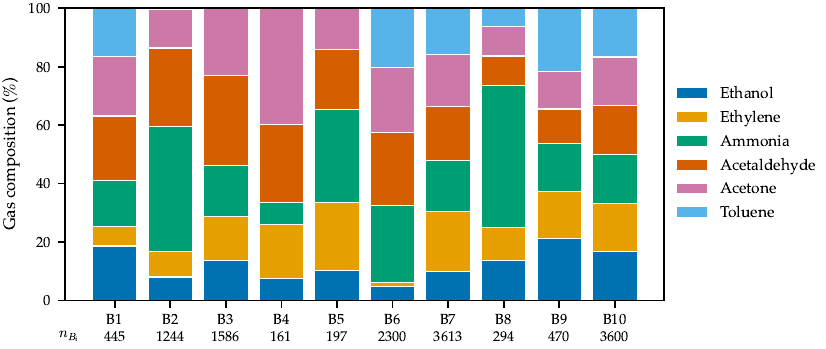}
    \caption{\textbf{Gas composition by temporal batch.}
  Stacked bars show the percentage of each gas class in batches B1--B10; numbers below the bars give the batch sizes $n_{B_i}$. The batches are strongly imbalanced and their gas mixtures vary over time, with some gases absent in several early batches and B10 balanced across the six gases.}
  \label{fig:gas_sensor_class_composition}
\end{figure}

\paragraph{Baselines and hyperparameters.}
We compare \poolPCA{}, fitted on the averaged source covariances; \AnchorPCA{\widehat{\lambda}}, using the selection procedure in \cref{app:gas_penalty_selection}; and \AnchorPCAinfty{}, using the block-stabilized estimator from \cref{subsec:finite_sample_implementation} with default tolerance. Lastly, we consider the worst-case PCA variant 
\texttt{norm-maxRegret} \citep{fries2026worstcaselowrankapproximations}
which is fit with
$10$ restarts, $2000$ iterations, learning rate $0.01$, and deterministic seed
$12020+10000s+101k$.

\paragraph{Protocol.}
For each last source batch $s\in\{3,\ldots,8\}$, and for each $k\in\{10,20,30\}$, we fit the rank-$k$ methods 
on B1--B$s$ and evaluate them on both the source batches and the held-out target
batches B$(s+1)$--B10. 
For each split, feature standardization is 
applied
once using the pooled source observations from B1--B$s$ only, and the resulting transformation is applied unchanged to all source and target batches. After this preprocessing, each batch covariance $\widehat\Sigma_b$ is computed from that batch's observations centered by their own empirical batch mean. For a representation $W\in\mathcal O_{p\times k}$ and batch covariance $\widehat\Sigma_b$, the reported percentage of explained variance (EV) is
\[
    \operatorname{EV}_b(W)
    :=
    100\,\frac{\Tr(W^\top\widehat\Sigma_b W)}{\Tr(\widehat\Sigma_b)}.
\]

\paragraph{Source-only penalty selection.}
\label{app:gas_penalty_selection}

For each $s\in\{3,\ldots,8\}$ and each fixed $k\in\{10,20,30\}$,
we select $2E\lambda$ by the leave-one-source-domain-out procedure in
\cref{subsec:finite_sample_implementation}, giving $18$ configurations.
For all penalty-selection experiments in this paper, we use the grid
\[
    2E\lambda\in\{0,10^{-3},10^{-2},10^{-1},1,10,10^2,10^3,\infty\}.
\]
For each choice of $2E\lambda$, we cross-validate over the source domains, holding each one out in turn.
We estimate feature means and standard deviations from the remaining source observations, apply this transformation to the held-out domain, and fit \AnchorPCA{} using the remaining domains' centered covariances. We score the fraction of total variance explained in the held-out source domain and select the penalty with the highest average score. Finally, we refit standardization and \AnchorPCA{} on all source domains and evaluate on the target domains.

\subsubsection{Rolling temporal splits.}
\Cref{fig:gas_sensor_rolling_summary} repeats the source--target evaluation of \cref{fig:gas_sensor_b1_b6_k20} across all source--target splits $s\in\{3,\ldots,8\}$. The top row shows mean source EV, the middle row mean target EV; shaded bands are min--max ranges across batches (shown only for \poolPCA{} and \AnchorPCAinfty{}%
). The bottom row reports the selected penalties of \AnchorPCA{\widehat{\lambda}}.

On source batches, \poolPCA{} has the highest EV as expected by its objective of maximizing average unnormalized explained variance. 
On target batches, \AnchorPCAinfty{} has higher mean target EV than \poolPCA{} in 16 of the 18 configurations. The two exceptions occur when fitting on B1--B8 with $k=20$ or $30$.
In the following paragraph, we show that this is driven by B9, whose dominant principal direction is captured more strongly by \poolPCA{}.
\AnchorPCA{\widehat{\lambda}} improves over \poolPCA{} in $14$ configurations, ties in two, and loses in two. It selects $\infty$ four times, a positive finite penalty twelve times, and zero (\poolPCA{}) twice. Cross-validation with only three to eight source domains is uncertain, so this is a practical selection rule and not a guarantee of the best target performance.
The overall pattern matches the main B1--B6/B7--B10 result in the main text. 

\begin{figure}[tbp]
    \centering
    \includegraphics[width=1\linewidth]{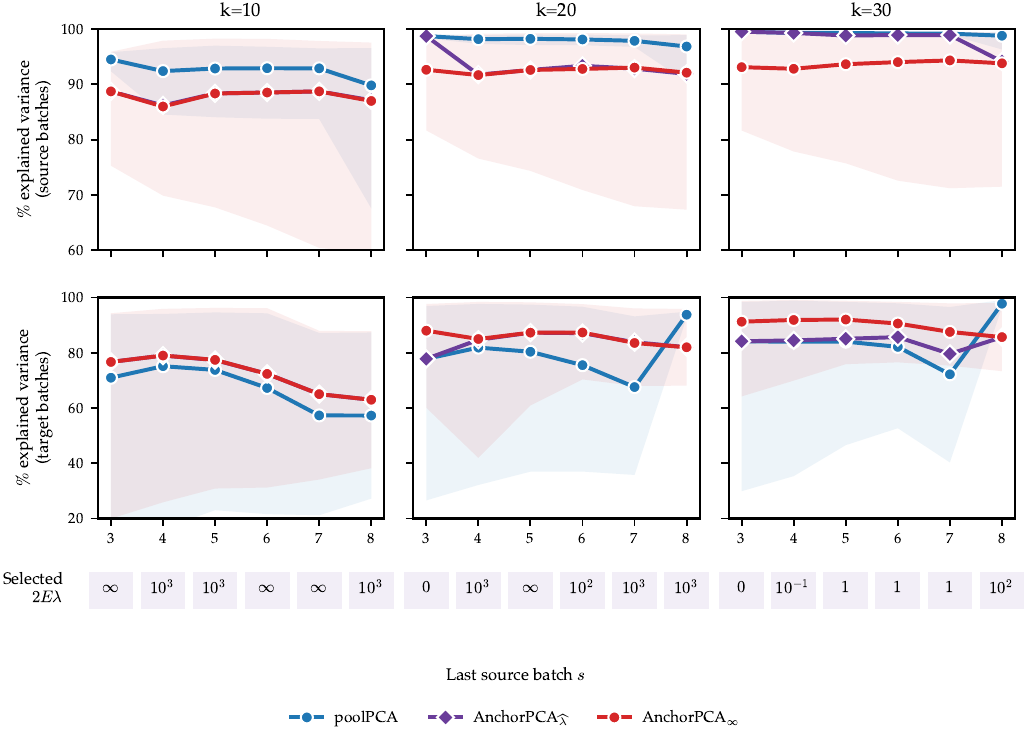}
    \caption{\textbf{Rolling source--target splits on gas-sensor drift data.} Each column fixes $k$; the horizontal axis is the last source batch $s$. Lines show mean EV over source batches (top row) or later target batches (middle row). The purple curve is \AnchorPCA{\widehat{\lambda}}, using \cref{subsec:finite_sample_implementation}; the bottom strip gives its selected $2E\lambda$ for each split. Shaded bands retain the min--max ranges across batches for \poolPCA{} and \AnchorPCAinfty{}.
    }
    \label{fig:gas_sensor_rolling_summary}
\end{figure}

\paragraph{Why poolPCA performs better on B9.}
\label{app:gas_b9_diagnostic}
At $s=8$, $k=20,30$, the methods are fitted on B1--B8 and evaluated
on B9 and B10. \poolPCA{} outperforms \AnchorPCAinfty{} because of B9:
its first principal component alone explains $92.43\%$ of its variance.
For the corresponding unit eigenvector $v_1$, its squared projection
$100\|\Pi_Wv_1\|_2^2$ onto the \AnchorPCAinfty{} solution subspace is
$66.00\%$ and $71.57\%$ at $k=20$ and $k=30$, respectively. The first PC of B9 thus does not seem to be in the invariant subspace $\Sstar$.
In contrast, its squared projection onto the \poolPCA{} solution subspace
is $95.51\%$ and $99.59\%$. Thus, \poolPCA{} captures this dominant target
direction more strongly. Still, \AnchorPCAinfty{} has higher EV on B10
(\cref{tab:gas_b9_diagnostic}). This does not contradict
\cref{thm:minimax_reconstruction}, whose robustness result applies to a
specified class of structured covariance perturbations, not to every
unseen-domain shift. Anchor PCA is not guaranteed to improve EV on every
unseen domain, especially when important target directions lie outside
$\Sstar$.

\begin{table}[tbp]
    \centering
    \caption{B9 diagnostic after fitting on B1--B8. EV and PC1 capture are distinct percentages, as defined in the text.}
    \label{tab:gas_b9_diagnostic}
    \begin{tabular}{rlrrr}
    \toprule
    $k$ & Method & B9 EV (\%) & B10 EV (\%) & B9 PC1 capture (\%) \\
    \midrule
    20 & \poolPCA{} & 94.84 & 92.72 & 95.51 \\
    20 & \AnchorPCAinfty{} & 68.09 & 95.85 & 66.00 \\
    30 & \poolPCA{} & 99.06 & 96.43 & 99.59 \\
    30 & \AnchorPCAinfty{} & 73.32 & 97.99 & 71.57 \\
    \bottomrule
    \end{tabular}
\end{table}

\paragraph{Explained variance of estimated $\Sstar$.}
The \AnchorPCAinfty{} solution has rank $k$, but its first 
grouped eigenspace is a finite-sample estimate of the invariant subspace $\Sstar$ with dimension $\widehat m$
(\cref{thm:consistency_empirical_methods}). 
For $s\in\{3,\ldots,8\}$ and $k\in\{10,20,30\}$, \Cref{fig:gas_sensor_sstar_same_dim} compares the percentage of explained variance contributed by this estimated block with source-only \poolPCA{} using the same dimension $\widehat m$ on B9 and B10.
The comparison is evaluated separately on target batches B9 and B10. Across $s \in \{3,\ldots,8\}$, the estimated dimensions are $\widehat m=(3,3,3,2,2,2)$ for $k=10$, $\widehat m=(5,5,5,5,5,5)$ for $k=20$, and $\widehat m=(9,8,7,7,7,6)$ for $k=30$. On the target batch B9, the estimated block explains more variance than the same-dimensional \poolPCA{} baseline for $k=20$ and $k=30$. On B10, both subspaces explain high variance and are often close. 
\begin{figure}[tbp]
    \centering
    \includegraphics[width=0.96\linewidth]{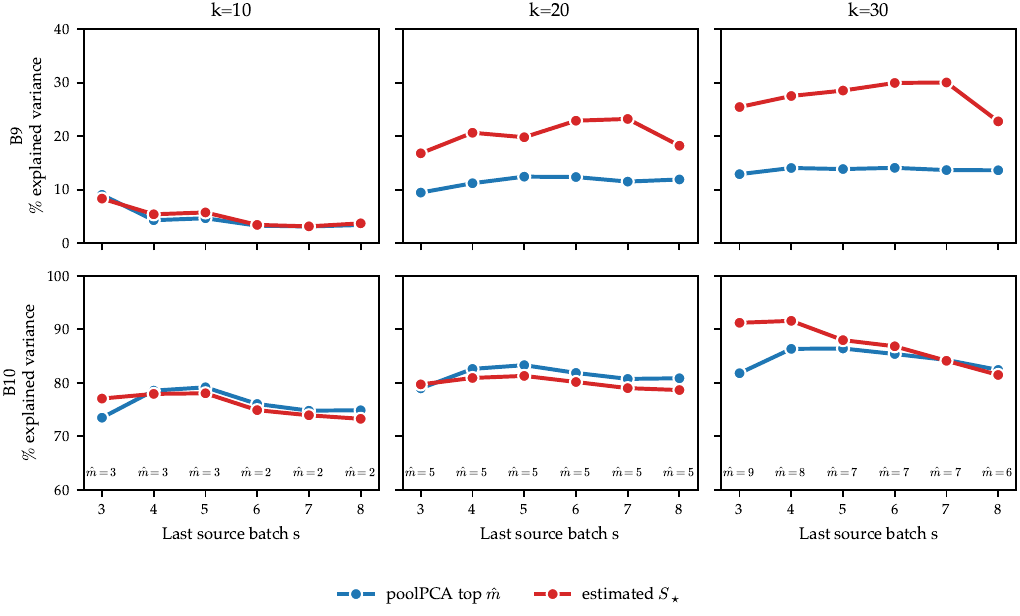}
    \caption{\textbf{Estimated invariant block vs.\ same-dimensional \poolPCA{}.}
    For each $(s,k)$, the red curve evaluates the first grouped eigenspace
    (estimated top eigenspace of $\barPi$, equivalently an estimate of $\Sstar$)
    selected by \AnchorPCAinfty{}; the blue curve evaluates the top
    $\widehat m$ source-only \poolPCA{} directions, where $\widehat m$ is the
    estimated dimension of $\Sstar$. Values are shown separately for target batches
    B9 and B10.}
    \label{fig:gas_sensor_sstar_same_dim}
\end{figure}

\subsubsection{Domain weighting and evaluation.}
\label{app:gas_weighting}
The experiments above give each source domain equal weight, regardless
of sample size. Here we examine whether the results change when domains are
weighted in proportion to sample size. We replace the equal weights
$w_e=1/E$ by $w_e=n_e/\sum_jn_j$ in both empirical averages, $\hatSigma = \sum_e w_e\widehat\Sigma_e$ and $\hatPi = \sum_e w_e\widehat\Pi_k^{(e)}$. The finite-penalty matrix is $\widehat M_\lambda = \sum_e w_e\widehat\Sigma_e+2E\lambda\sum_e w_e\widehat\Pi_k^{(e)}$. At the population level, the corresponding reconstruction objective is
\[
    \sum_e w_e\Tr\bigl(\Sigma_e(I_p-\Pi_W)\bigr)
    +E\lambda\sum_e w_e\|\Pi_W-\Pi_k^{(e)}\|_{\mathrm F}^2.
\]
This is also the worst-case reconstruction error over the original covariance class $\mathcal{C}_\lambda$ in \cref{thm:minimax_reconstruction}, now weighted by sample-size.

\Cref{tab:gas_weighting} compares the two fitting schemes using the
same gas-sensor preprocessing and EV metric as in \cref{subsec:exp_gas_sensor}. For each split $s\in\{3,\ldots,8\}$, we fit on B1--B$s$ and evaluate on the later batches.
We average target EV equally over later batches within
each split, then equally over $s\in\{3,\ldots,8\}$. For both weighting schemes,
\AnchorPCAinfty{} has higher mean target EV than \poolPCA{} at all three ranks, though neither weighting scheme is uniformly better than the other.

\begin{table}[t]
    \centering
    \caption{
    The rolling experiment in \cref{fig:gas_sensor_rolling_summary} weights source batches equally when fitting; here we also fit both methods using weights proportional to batch sample sizes. Entries are mean target EV differences, \AnchorPCAinfty{} minus \poolPCA{}, in percentage points (values above zero imply that \AnchorPCAinfty{} explains more variance than \poolPCA{}).
    }
    \label{tab:gas_weighting}
    \begin{tabular}{rrr}
    \toprule
    $k$ & Uniform-domain fitting & Sample-size-weighted fitting \\
    \midrule
    10 & 5.29 & 5.07 \\
    20 & 6.03 & 11.19 \\
    30 & 5.78 & 4.50 \\
    \bottomrule
    \end{tabular}
\end{table}

\subsubsection{Gas-type classification.}
\label{app:gas_classification}
As a downstream task, we classify the six gas types using the split in \cref{fig:gas_sensor_b1_b6_k20}: we fit on batches B1--B6 (sources) and evaluate on B7--B10 (targets). Source labels are used for training and model selection, and target labels for evaluation.

\textit{Methods.} We fit L2-regularized logistic regression on the $20$-dimensional representations from \poolPCA{}, \AnchorPCA{\widehat{\lambda}}, and \AnchorPCAinfty{}, and on all $128$ unreduced features. This gives all four representations the same classifier family. Before fitting the representations, we rescale each feature using statistics from the pooled source observations, compute each batch's covariance around its own mean, and weight all source batches equally. We then standardize the classifier inputs using source observations only. As further baselines on the unreduced features, we fit L1-regularized logistic regression, linear GroupDRO \citep{sagawa2020distributionally}, which minimizes the largest source-batch loss, and linear V-REx \citep{krueger2021out}, which minimizes the mean source-batch loss plus a penalty on its variance across batches. So that large classes do not dominate the loss, logistic regression weights observations by inverse class frequency, and GroupDRO and V-REx use class-balanced cross-entropy within each batch.
L2 logistic regression fits the classes jointly (multinomial) and L1 logistic regression fits each class against the rest.

\textit{Hyperparameter selection.} We use leave-one-source-batch-out cross-validation. For each candidate, we refit the feature transformations, the representation (where applicable), and the classifier on five source batches, and score the held-out sixth. We repeat this for all six sources and select the candidate with the highest mean held-out balanced accuracy, weighting the six folds equally. Balanced accuracy averages the accuracy within each class present, so every class counts equally. Every logistic classifier selects the inverse penalty strength $C\in\{10^{-3},10^{-2},10^{-1},1,10,10^2,10^3\}$, where a smaller $C$ means stronger regularization. \AnchorPCA{\widehat{\lambda}} selects $C$ jointly with $2E\lambda$ from the grid in \cref{app:gas_penalty_selection}. GroupDRO selects its L2 coefficient from $\{10^{-4},10^{-3},10^{-2}\}$, where a larger value means stronger regularization. V-REx selects its L2 coefficient from the same grid jointly with the weight $\beta\in\{0,0.1,1,10,100,1000\}$ on its loss-variance penalty, where $\beta=0$ removes the penalty. We then refit each selected pipeline on all six sources.

\textit{Results.} \Cref{tab:gas_classification_summary} averages accuracy and balanced accuracy over the four target batches, and \cref{tab:gas_classification_batches} reports accuracy per batch. \AnchorPCA{\widehat{\lambda}} is more accurate than \poolPCA{} on all four target batches. \AnchorPCAinfty{} has the highest mean accuracy and balanced accuracy of the seven pipelines.
With only four target batches, we do not test these differences for significance.
Further, unpenalized logistic regression performed worse on all four representations, so we omit it.

\begin{table}[tbp]
    \centering
    \caption{
    Gas classification: mean accuracy and balanced accuracy on the target batches, with each batch weighted equally. L1 and L2 denote the penalty of the logistic-regression classifier; the last two rows use linear GroupDRO and V-REx classifiers, both
with L2 regularization.}
    \label{tab:gas_classification_summary}
    \renewcommand{\arraystretch}{1.15}\begin{tabular}{lclrr}
    \toprule
    Pipeline & Dim. & Selected values & \makecell{Accuracy\\(\%)} & \makecell{Balanced\\accuracy (\%)} \\
    \midrule
    \poolPCA{} + L2 & 20 & $C=10$ & 68.58 & 67.29 \\
    \addlinespace[4pt]
    \makecell[l]{\AnchorPCA{\widehat{\lambda}} + L2} & 20 & \makecell[l]{$2E\lambda=10$\\$C=0.1$} & 74.61 & 75.64 \\
    \addlinespace[4pt]
    \AnchorPCAinfty{} + L2 & 20 & $C=0.01$ & \textbf{75.56} & \textbf{76.79} \\
    \midrule
    Unreduced + L2 & 128 & $C=0.1$ & 75.05 & 75.08 \\
    \addlinespace[4pt]
    Unreduced + L1 & 128 & $C=0.1$ & 74.44 & 76.16 \\
    \addlinespace[4pt]
    Unreduced + GroupDRO & 128 & L2: $0.001$ & 74.98 & 75.34 \\
    \addlinespace[4pt]
    Unreduced + V-REx & 128 & \makecell[l]{$\beta=1$\\L2: $0.001$} & 74.31 & 74.08 \\
    \bottomrule
    \end{tabular}
\end{table}

\begin{table}[tbp]
    \centering
    \caption{Gas classification accuracy (\%) by target batch, using the methods in \cref{tab:gas_classification_summary}.}
    \label{tab:gas_classification_batches}
    \begin{tabular}{lrrrr}
    \toprule
    Pipeline & B7 & B8 & B9 & B10 \\
    \midrule
    \poolPCA{} + L2 & 80.51 & 79.59 & 65.32 & 48.89 \\
    \AnchorPCA{\widehat{\lambda}} + L2 & 80.68 & 88.10 & 65.96 & 63.69 \\
    \AnchorPCAinfty{} + L2 & 79.19 & 80.95 & 82.98 & 59.14 \\
    Unreduced + L2 & 86.13 & 85.37 & 64.47 & 64.22 \\
    Unreduced + L1 & 89.21 & 70.41 & 71.28 & 66.89 \\
    Unreduced + GroupDRO & 88.49 & 90.48 & 65.32 & 55.64 \\
    Unreduced + V-REx & 86.35 & 89.12 & 65.11 & 56.67 \\
    \bottomrule
    \end{tabular}
\end{table}

\subsection{\texorpdfstring{Fama--French industry portfolios}{Fama--French industry portfolios}}
\label{app:fama_french}

\paragraph{\texorpdfstring{Summary of the experiment}{Summary of the experiment}}
\label{subsec:exp_fama_french}
We evaluate Anchor PCA on a second real-world temporal study using
financial-market data from the Fama--French $49$ Industry Portfolios dataset
\citep{french49industry}. Each observation represents one trading day with
$p=49$ features corresponding to the value-weighted daily returns of different
industry portfolios. We use non-overlapping five-year periods as domains:
1970--1994 gives five source domains and 1995--2024 gives six target domains.

As the first two principal components capture roughly two-thirds of the variance in most source domains
(\cref{fig:fama_source_scree}), we first present results for $k=2$. At $k=2$, \AnchorPCA{\widehat{\lambda}} chooses $\lambda=\infty$, i.e. \AnchorPCAinfty{}. Its mean target EV is $55.67\%$,
compared with $51.80\%$ for \poolPCA{} and $52.79\%$ for
\texttt{norm-maxRegret} \citep{fries2026worstcaselowrankapproximations}.
Relative to \poolPCA{}, the mean gain is $3.87$ percentage points, with gains
of $1.94$--$5.55$ points across all six future periods; it also improves over
\texttt{norm-maxRegret} in all six periods. 
It also raises the worst-period EV from $40.18\%$ (\poolPCA{}) and
$41.44\%$ (\texttt{norm-maxRegret}) to $44.44\%$.
Across the larger ranks
$k\in\{3,\ldots,10\}$, \AnchorPCA{\widehat{\lambda}} differs from \poolPCA{}
by at most $0.25$ percentage points in absolute mean target EV.

\paragraph{Protocol.}
Following Kenneth French's Data Library, we treat entries of $-99.99$ and $-999$ as missing and drop days with any missing return. This removes one day in 1970--1974 and none elsewhere.
We center returns within each period and compute their sample covariance
without feature standardization. Source periods receive equal fitting weights.
For each fixed $k\in\{2,\ldots,10\}$, we select $2E\lambda$ by
\cref{subsec:finite_sample_implementation} with the grid of
\cref{app:gas_penalty_selection}, and then refit on all five sources.
At $k=2$, we also fit \texttt{norm-maxRegret}.

\paragraph{Rank choice and results.}
All five source scree plots show two dominant components followed by a flatter tail (\cref{fig:fama_source_scree}); the second eigenvalue is $1.73$--$3.77$ times the third. This supports examining $k=2$. 
\Cref{tab:fama_period_ev} reports target EV for each period at $k=2$, and \cref{tab:fama_rank_sensitivity} gives results for $k=3,\ldots,10$. For $k=3$ to $6$, cross-validation selects $2E\lambda=0$, which recovers \poolPCA{}.

\begin{figure}[tbp]
    \centering
    \includegraphics[width=\linewidth]{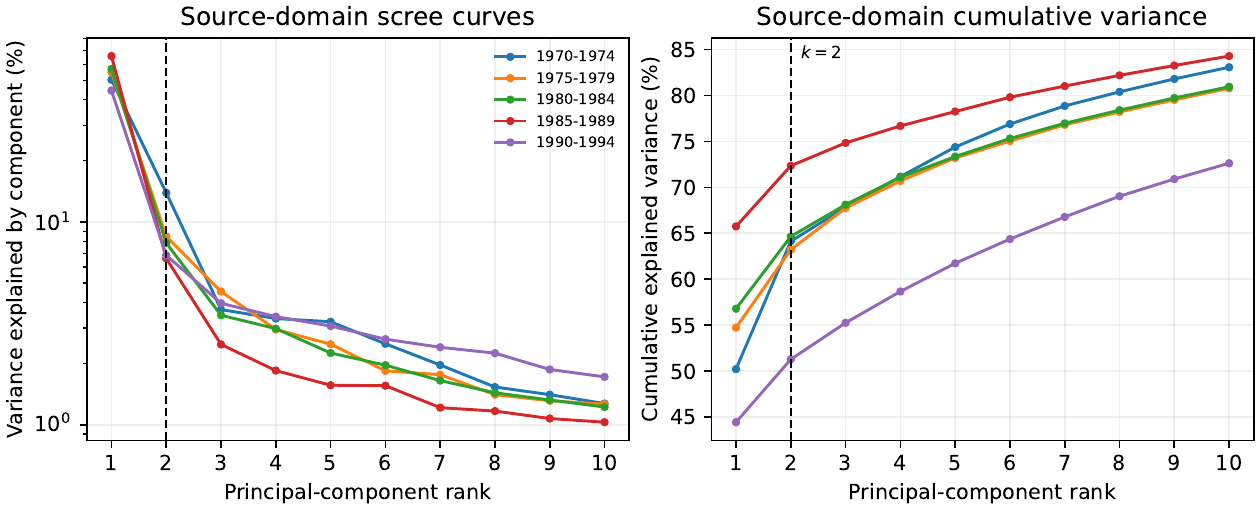}
    \caption{Fama--French source scree plots. 
    Colors show the five-year source periods shown in the legend. 
    Left: per-component EV (log scale); right: cumulative EV. Only source-period covariances enter these rank diagnostics.}
    \label{fig:fama_source_scree}
\end{figure}

\begin{table}[H]
    \centering
    \caption{Fama--French target EV (\%) at $k=2$. Source-only cross-validation selects \AnchorPCAinfty{}. The mean gives each target period equal weight; the minimum is across periods.}
    \label{tab:fama_period_ev}
    \begin{tabular}{lrrrr}
    \toprule
    Target period & Days & \poolPCA{} & \texttt{norm-maxRegret} & \AnchorPCAinfty{} \\
    \midrule
    1995--1999 & 1263 & 40.18 & 41.44 & 45.73 \\
    2000--2004 & 1256 & 41.63 & 42.15 & 44.44 \\
    2005--2009 & 1259 & 61.31 & 62.33 & 65.15 \\
    2010--2014 & 1258 & 64.72 & 65.91 & 68.75 \\
    2015--2019 & 1258 & 48.07 & 49.42 & 53.14 \\
    2020--2024 & 1258 & 54.88 & 55.50 & 56.82 \\
    \midrule
    Mean & -- & 51.80 & 52.79 & 55.67 \\
    Minimum & -- & 40.18 & 41.44 & 44.44 \\
    \bottomrule
    \end{tabular}
\end{table}

\begin{table}[H]
    \centering
    \caption{Fama--French rank sensitivity. EV is averaged equally over the six target periods. The final column gives \AnchorPCA{\widehat{\lambda}} minus \poolPCA{} in percentage points.}
    \label{tab:fama_rank_sensitivity}
    \begin{tabular}{rcrrr}
    \toprule
    $k$ & Selected $2E\lambda$ & \makecell{\poolPCA{}\\EV (\%)} & \makecell{\AnchorPCA{\widehat{\lambda}}\\ EV (\%)} & \makecell{Difference\\(points)} \\
    \midrule
    2 & $\infty$ & 51.80 & 55.67 & $+3.873$ \\
    3 & $0$ & 57.12 & 57.12 & $0.000$ \\
    4 & $0$ & 62.20 & 62.20 & $0.000$ \\
    5 & $0$ & 64.69 & 64.69 & $0.000$ \\
    6 & $0$ & 66.04 & 66.04 & $0.000$ \\
    7 & $1$ & 67.94 & 68.19 & $+0.248$ \\
    8 & $1$ & 69.79 & 69.89 & $+0.101$ \\
    9 & $1$ & 70.97 & 71.06 & $+0.095$ \\
    10 & $0.01$ & 72.07 & 72.07 & $-0.003$ \\
    \bottomrule
    \end{tabular}
\end{table}

\section{Nonlinear extensions} 
\label{app:nonlinear_extensions}
In this appendix, we discuss nonlinear extensions of our approach to finding rank-$k$ robust or invariant subspaces via Anchor PCA translated to autoencoder-based approaches.

Autoencoders (AEs) are a popular approach for nonlinear dimension reduction~\citep{rumelhart1986learning,hinton2006reducing}.
As mentioned in~\cref{sec:setting_and_notation}, the projection matrix $\Pi_W=WW^\top$ can be interpreted as a linear autoencoder which first projects a given $x$ to a $k$-dimensional embedding via $W^\top\in\RR^{k\times p}$ and then reconstructs $x$ by mapping back to the original $p$-dimensional space via $W\in\RR^{p\times k}$.
An autoencoder functions analogously, but with possibly nonlinear encoder and decoder, typically parameterized as (deep) neural networks and optimized via stochastic gradient descent~\citep{robbins1951stochastic}.

\subsection{Vanilla autoencoder (AE)}
An autoencoder $(f,g)$ consists of an encoder $f:\RR^p\to\RR^k$, which produces $k$-dimensional embeddings $z=f(x)$, and a decoder $g:\RR^k\to \RR^p$, which produces reconstructions $\hat x=g(z)=g(f(x))$.
A standard autoencoder is trained to minimize the mean squared reconstruction error 
\begin{equation*}
    \Lcal_\textsc{ae}\left(f,g; \PP \right)=
    \EE_{x\sim\PP}\left[
    \norm{x-g(f(x))}^2_2
    \right],
\end{equation*}
where, in practice, $\PP$ is replaced by its empirical distribution.

Given data from multiple domains, one approach to learning a shared nonlinear $k$-dimensional embedding is to train an autoencoder on the pooled data, similar to \poolPCA{} in the linear setting. 
We refer to this approach as \texttt{poolAE}.
However, like \poolPCA{}, this approach is susceptible to focusing on spurious directions that exhibit high variation in only a few domains. 

\subsection{Anchor autoencoder (\texttt{AnchorAE})}
We now discuss a method we call
\texttt{AnchorAE} as a more robust alternative. 
Let $\{(f_e,g_e)\}_{e=1}^E$ be domain-specific AEs, i.e.,
\begin{equation*}
    (f_e,g_e)\in\argmin_{f,g} \Lcal_\textsc{ae}\left(f,g; \PP_e\right).
\end{equation*}
In analogy to~\cref{eq:anchorPCA}, we define \texttt{AnchorAE}$_\lambda$ as the following unconstrained optimization problem
\begin{equation}
    \label{eq:anchorAE}
    \min_{f,g} 
    \EE_{x\sim\bar\PP}
    \left[
    \norm{x-g(f(x))}^2_2\right]
    +\lambda \sum_{e=1}^E
    \Rcal\left(g\circ f, g_e\circ f_e\right)
\end{equation}
where $\bar\PP=\frac{1}{E}\sum_{e=1}^E\PP_e$ denotes the (empirical) pooled distribution and $\Rcal$ is a suitable invariance penalty between the shared and domain-specific AEs (discussed in more detail below).

The first term in~\cref{eq:anchorAE} is the pooled reconstruction error, which can be viewed as the AE-equivalent or proxy of explained variance in PCA. 
The second term is meant to capture disagreement among the nonlinear projections produced by the shared $g\circ f$ and domain-specific  $g_e\circ f_e$ AEs, similar to the disagreement term \[\fnorm{\Pi_W-\Pi^{(e)}_k}^2\]
between projection matrices in \AnchorPCA{} in~\cref{eq:anchorPCA}.

\paragraph{On the choice of invariance penalty.}
Whereas in the linear case, it is possible to compare the projection matrices directly, it is less clear how to assess agreement between nonlinear functions, as these are infinite-dimensional objects whose global behavior is not determined by their local behavior. In principle, we can consider measuring invariance w.r.t.\ domain-specific nonlinear projections through suitable norms $\norm{(g\circ f)-(g_e\circ f_e)}$ in function space. 
However, each AE is unconstrained outside the support of the (empirical) distribution used to train it, i.e., $\PP_e$ for $(f_e,g_e)$ and $\bar\PP$ for $(f,g)$. Hence, we should evaluate disagreement only on the respective shared supports.
We therefore propose the following invariance penalty
\begin{equation*}
     \Rcal\left(g\circ f, g_e\circ f_e\right) = \EE_{x\sim\PP_e}
    \left[\norm{g(f(x))-g_{e}(f_{e}(x))}^2_2\right],
\end{equation*}
which can be viewed as the regret incurred by using the shared instead of the domain-specific AE for reconstructing data from domain $e$~\citep{freni2025maximum,agarwal2022minimax}.

\textbf{Limitations.} \texttt{AnchorAE} requires training $E+1$ separate AEs, since to the best of our knowledge --- and
unlike in the linear case---the problem cannot be reduced to a modified single learning objective for $(f,g)$. 
This approach may thus be computationally expensive, especially when the number of domains is large.
In the following subsections, we  discuss an approach for training a single, distributional AE from multi-domain data.

\subsection{Distributional principal autoencoder~(DPA)}
The distributional principal autoencoder~(DPA)~\citep{shen2024distributional} is a nonlinear dimension reduction approach, which can be interpreted as a nonlinear variant of PCA that, in contrast to standard autoencoders, emphasizes distributional rather than mean reconstruction.
We first review DPA in the single-domain setting and then discuss a possible extension for our multi-domain setting.

Given a number $k\leq p$ of components, DPA consists of an encoder $f:\RR^p\to\RR^k$, which produces embeddings $z=f(x)$, and a stochastic decoder or generator $g:\RR^k\times \RR^{p-k} \to \RR^p$, which, in addition to $z$, takes $(p-k)$-dimensional noise $\epsilon\sim\Ncal(0,I)$ as input and allows for sampling reconstructions $\hat x=g(z,\epsilon)$.
The DPA objective is constructed such that, for a fixed encoder~$f$, the optimal DPA decoder $g^*$ maps a given embedding~$z$ to the distribution of $X$ given $f(X)=z$,
\begin{equation}
    \label{eq:optimal_decoder_DPA}
    g^*(z,\epsilon; f)\overset{d}{=}
    \left(X|f(X)=z\right),
\end{equation}
where $\overset{d}{=}$ denotes equality in distribution. In other words, the optimal decoder evaluated at $z$ matches the distribution of realizations of $X$ that are mapped by the encoder to $z$.
At the same time, the DPA encoder minimizes the variability in the distributions in~\eqref{eq:optimal_decoder_DPA} by encoding the first $k$ 
`principal' 
components. 
Both of these goals are achieved by minimizing the following DPA objective
\begin{align*}
 &\Lcal_\textsc{dpa}^\beta(f,g;\PP)\\
 &\qquad =
 \EE_{X\sim\PP,\epsilon\sim\Ncal(0, I)
 }\norm{X-g(f(X),\epsilon)}^\beta_2
 -\frac{1}{2} \EE_{X\sim\PP,\epsilon,\epsilon'\iidsim\Ncal(0, I)
 }\norm{g(f(X),\epsilon)-g(f(X),\epsilon')}^\beta_2,
\end{align*}
where $\beta\in(0,2)$.
This objective corresponds to the negative expected energy score~\citep{gneiting2007strictly} between $X$ and the corresponding stochastic decoder output, conditional on the encoding $f(X)$ of $X$.
Due to this conditioning, the DPA objective differs from an energy distance by a normalization constant which depends on the encoder and encourages capturing principal (i.e., variation-minimizing) components. %

\subsection{Multi-domain DPA}
One could
extend the DPA approach to our task of learning invariant subspaces from multi-domain data as follows.
To map data from different distributions $\{\PP_e\}_{e\in\Ecal}$ to an invariant subspace, we consider a single shared encoder $f:\RR^p\to\RR^k$, whereas to allow for accurate distributional reconstruction across domains, we consider domain-specific decoders $\left\{g_e:\RR^k\times \RR^{p-k} \to \RR^p\right\}_{e\in \Ecal}$.

We could now train $E$ domain-specific DPAs. However, in contrast to \AnchorPCA{} and \texttt{AnchorAE}, it is unclear how to then compare or penalize disagreement between a shared DPA and domain-specific ones since, by construction, each DPA can fully reconstruct any distribution if given sufficient capacity in the form of additional $p-k$-dimensional noise as input to the stochastic decoder.
Moreover, as discussed above, training a separate model for each domain is computationally demanding. 

This motivates considering
 the following objective for training a single multi-domain DPA:
\begin{equation}
    \label{eq:multi_domain_DPA_objective}
    \Lcal_\textsc{md-dpa}^\beta\left(f, \left\{g_e\right\}_{e\in \Ecal}; \left\{\PP_e\right\}_{e\in \Ecal}\right):=\sum_{e\in\Ecal} 
    \Lcal_\textsc{dpa}^\beta(f,g_e;\PP_e).
\end{equation}
Since domain-specific decoders are used, this is different from a  distributional, nonlinear version of \poolPCA{}.
Intuitively, the additional reconstruction capacity from domain-specific decoders and noise 
may enable the encoder to focus on shared directions, rather than explaining pooled variance.

\end{document}